%% file: main.tex
\documentclass[11pt]{article}

\input{preamble}
\title{\LARGE\bf
Deep Neural Networks}

\author{Randall Balestriero and Richard G.\ Baraniuk \\[2mm]
Rice University}

\begin{document}
\maketitle

\begin{abstract}
\input{abstract}
\end{abstract}
\clearpage 
\tableofcontents
\clearpage



\input{intro}
\newpage
\section*{Symbols}
\begin{tabular}{|l|l|}
$x$&''Dummy'' variable representing an input/observation\\
$\hat{y}(x)$&''Dummy'' variable representing an output/prediction associated  to input $x$\\
$X_n$&Observation $n$ of shape $(K,I,J)$.\\ 
$Y_n$&Target variable associated to $X_n$, for classification $Y_n \in \{1,\dots,C\},\;C>1$,\\
&for regression $Y_n \in \mathbb{R}^C,C\geq 1$.\\
$\mathcal{D}$ (resp. $\mathcal{D}_s$)&Labeled training set with $N$ (resp. $N_s$) samples $\mathcal{D}=\{(X_n,Y_n)_{n=1}^N\}$.\\
$\mathcal{D}_u$&Unlabeled training set with $N_u$ samples $\mathcal{D}_u=\{(X_n)_{n=1}^{N_u}\}$.\\
$f^{(\ell)}_{\theta^{(\ell)}}$&Layer at level $\ell$ with internal parameters $\theta^{(\ell)},\ell=1,\dots,L$.\\
$\Theta$&Collection of all parameters $\Theta=\{\theta^{(\ell)},\ell=1,\dots,L\}$.\\
$f_{\Theta}$&Deep Neural Network mapping with $f_\theta:\mathbb{R}^D\rightarrow \mathbb{R}^C$\\
$(C^{(\ell)},I^{(\ell)},J^{(\ell)})$&Shape of the representation at layer $\ell$ with $(C^{(0)},I^{(0)},J^{(0)})=(K,I,J)$ and\\
&$(C^{(L)},I^{(L)},J^{(L)})=(C,1,1)$.\\
$D^{(\ell)}$&Dimension of the flattened representation at layer $\ell$ with $D^{(\ell)}=C^{(\ell)}I^{(\ell)}J^{(\ell)}$,$D^{(0)}=D$ and $D^{(L)}=C$.\\
$z^{(\ell)}(x)$&Representation of $x$ at layer $\ell$ in an unflattened format of shape $(C^{(\ell)},I^{(\ell)},J^{(\ell)})$,\\
&with $z^{(0)}(x)=x$\\
$z^{(\ell)}_{c,i,j}(x)$&Value at channel $c$ and spatial position $(i,j)$.\\
$\bz^{(\ell)}(x)$&Representation of $x$ at layer $\ell$ in a flattened format of dimension $D^{(\ell)}$\\
$\bz^{(\ell)}_d(x)$&Value at dimension $d$\\
\end{tabular}
\newpage

\input{background}

\newpage
\input{affine}
\newpage
\input{semisup}

\newpage
\input{unbiased}



\newpage





\section{Conclusion}
We presented a natural reformulation of deep neural network as composition of adaptive partitioning splines and in general linear spline operators. By doing so we have been able to explicitly determine the optimal network weights, their impact for adversarial example, generalization, memorization. From this we built an intuitive and generic method to invert arbitrary networks giving rise to semi-supervised and unsupervised application. We obtained state-of-the-art performances on MNIST with CNNs and Resnets and provided supplemental experiments highlighting the ability of the introduced reconstruction error to regularize and improve generalization. We also proposed a simple criterion to judge the quality of a network and its initialization based on template analysis allowing fast topology search. Finally, by bridging many fields such as template matching, adaptive partitioning splines, and depe neural networks, we hope to allow further and deeper analysis of all the presented results and insights.
\bibliography{ref}
\bibliographystyle{apalike}
\clearpage

\appendix

\input{theory}

\section{Dataset and Model Description}\label{description}

\subsection{CIFAR10}
\begin{table}[H]
\centering
\caption{Mapping from integer to class label.}
\begin{tabular}{|l | l|}\hline
0&airplane \\ \hline
1&automobile\\ \hline 
2&bird \\ \hline
3&cat \\ \hline
4&deer \\ \hline
5&dog \\ \hline
6&frog \\ \hline
7&horse \\ \hline
8&ship \\ \hline
9&truck\\ \hline
\end{tabular}
\end{table}

\subsection{Networks Description and Training Details}
We remind that for all the networks topologies, the mean/max describe the mode of the pooling layer whereas the presence of $W$ transcribes to wavelets a.k.a the convex version of the network where the convolutional layers and nonlinearity layers are replaced with the one proposed in this work, also, $F$ and $NF$ represent the Fixed and NonFixed version of the newly introduced convolutional layer as described in its section.
We present in the Table below the two topologies, Small and Large.

\begin{table}[H]
\centering
\caption{Deep CNN Architectures}
\begin{tabular}{|l | l|l|}\hline
SmallCNN (131512) & LargeCNN ()\\ \hline
Input &Input\\
Conv (32,5,5) full &Conv (96,3,3) same\\
Pool (2,2)         &Conv (96,3,3) full\\
Conv (64,3,3) valid&Conv (96,3,3) full\\
Conv (64,3,3) full &Pool (2,2)\\
Pool (2,2)         &Conv (192,3,3) valid\\
Conv (128,3,3) valid&Conv (192,3,3) full\\
Conv (10,1,1)     &Conv (192,3,3) valid\\
MeanPool (6,6)    &Pool (2,2)\\
SoftMax           &Conv (192,3,3) valid\\
                &Conv (192,1,1) valid\\
                &Conv (10,1,1) valid\\
                &MeanPool(6,6)\\
                &SoftMax \\ \hline
\end{tabular}
\end{table}
Note that for the standard cases (no $W$) the nonlinerities are always leaky rectifiers, for the convex cases ($W$) the first convolutional layer is left unconstrained and with a leaky rectifier in order to fulfill the template matching theorem.

Given this residual block, we define a simple Resnet (as opposed ot wide Resnets) which topology depends on a parameter $n>0$ defining the number of blocks per stages. The full Resnet has $3$ stages of ''widening'' given by the following table.
\begin{table}[h]
\centering
\caption{Residual Architectures}
\begin{tabular}{|l |}\hline
Resnet3 () \\ \hline
Input \\
Conv (16,3,3) valid \\\hline
Block (0)         \\
Block (0)         \\
Block (0)         \\ \hline
Block (1)         \\
Block (0)         \\
Block (0)         \\\hline
Block (1)         \\
Block (0)         \\
Block (0)         \\ \hline
GlobalMeanPool    \\
SoftMax \\\hline
\end{tabular}
\quad
\begin{tabular}{|l |}\hline
Resnet5 () \\ \hline
Input \\
Conv (16,3,3) valid \\\hline
Block (0)         \\
Block (0)         \\
Block (0)         \\
Block (0)         \\
Block (0)         \\ \hline
Block (1)         \\
Block (0)         \\
Block (0)         \\
Block (0)         \\
Block (0)         \\\hline
Block (1)         \\
Block (0)         \\
Block (0)         \\
Block (0)         \\
Block (0)         \\ \hline
GlobalMeanPool    \\
SoftMax\\\hline
\end{tabular}
\end{table}

For all topologies, Adam optimizer \cite{kingma2014adam} is used with a learning rate decay starting at $0.005$ or $0.0005$ for the Large topology.
No regularization is applied nor batch normalization in order to better highlight the impact of the proposed changes. No shuffling is performed between epochs and the only normalization done is to have each observation with $0$ mean and infinite norm equals to $1$. For CIFAR10 this normalization is not done per channel (RGB) but across them.
All the weights initialization as kept as default given the lasagne classes.

\end{document}

%% file: preamble.tex
\usepackage{times}
\usepackage{graphicx,latexsym,cite,amsmath,amssymb,amsthm,eucal,graphicx,color}
\usepackage{url}
\usepackage{comment}
\usepackage{float}
\usepackage{tcolorbox}
\usepackage{booktabs}
\usepackage{blindtext,graphicx}
\usepackage{hyperref}
\hypersetup{
    colorlinks,
    linkcolor={red!50!black},
    citecolor={blue!50!black},
    urlcolor={blue!80!black}
}
\usepackage{bm}
\usepackage[margin=1in]{geometry}
\usepackage{setspace}
\usepackage{listings}
\usepackage[version=4]{mhchem}
\usepackage{empheq}
\usepackage{mdframed}
\usepackage{nomencl,etoolbox,ragged2e,siunitx}
\usepackage{tikz}
\usetikzlibrary{shapes,arrows}
\usetikzlibrary{er,positioning}

\tikzstyle{block} = [rectangle, draw, fill=blue!20, 
    text width=12.8em, text centered, rounded corners, minimum height=4em]
\tikzstyle{line} = [draw, -latex']
\tikzstyle{cloud} = [draw, ellipse,fill=red!20, node distance=3cm,
    minimum height=2em]

\newtheorem{theorem}{Theorem}
\newtheorem{prop}{Proposition}
\newtheorem{cor}{Corollary}

\newtheorem{defn}{Definition}

\usepackage{tikz}
\usetikzlibrary{calc}

\DeclareMathAlphabet\mathbfcal{OMS}{cmsy}{b}{n}

\def \bz{\boldsymbol{z}}

\def \bC{\boldsymbol{C}}

\def \bW{{\boldsymbol W}}

\DeclareMathOperator*{\argmin}{arg\,min}

\DeclareMathOperator*{\argmax}{argmax}

\newcommand {\richb}[1]{{\color{blue}\sf{[richb: #1]}}}

\usepackage{xcolor}

%% file: abstract.tex
Deep Neural Networks (DNNs) are universal function approximators providing state-of-
the-art solutions on wide range of applications. Common perceptual tasks such as speech recognition,
image classification, and object tracking are now commonly tackled via DNNs. Some fundamental
problems remain: (1) the lack of a mathematical framework providing an explicit and interpretable
input-output formula for any topology, (2) quantification of DNNs stability regarding adversarial
examples (i.e. modified inputs fooling DNN predictions whilst undetectable to humans), (3) absence of
generalization guarantees and controllable behaviors for ambiguous patterns, (4) leverage unlabeled
data to apply DNNs to domains where expert labeling is scarce as in the medical field. Answering those
points would provide theoretical perspectives for further developments based on a common ground.
Furthermore, DNNs are now deployed in tremendous societal applications, pushing the need to fill this
theoretical gap to ensure control, reliability, and interpretability.

%% file: intro.tex
\section{Introduction}
\label{sec:intro}
Deep Neural Networks (DNNs) are universal function approximators providing state-of-
the-art solutions on wide range of applications. Common perceptual tasks such as speech recognition,
image classification, and object tracking are now commonly tackled via DNNs. Some fundamental
problems remain: (1) the lack of a mathematical framework providing an explicit and interpretable
input-output formula for any topology, (2) quantification of DNNs stability regarding adversarial
examples (i.e. modified inputs fooling DNN predictions whilst undetectable to humans), (3) absence of
generalization guarantees and controllable behaviors for ambiguous patterns, (4) leverage unlabeled
data to apply DNNs to domains where expert labeling is scarce as in the medical field. Answering those
points would provide theoretical perspectives for further developments based on a common ground.
Furthermore, DNNs are now deployed in tremendous societal applications, pushing the need to fill this
theoretical gap to ensure control, reliability, and interpretability.

DNNs are models involving compositions of nonlinear and linear transforms. (1) We will
provide a straightforward methodology to express the nonlinearities as affine spline functions. The
linear part being a degenerated case of spline function, we can rewrite any given DNN topology as
succession of such functionals making the network itself a piecewise linear spline. This formulation
provides a universal piecewise linear expression of the input-output mapping of DNNs, clarifying the
role of its internal components. (2) In functional analysis, the regularity of a mapping is defined via its
Lipschitz constant. Our formulation eases the analytical derivation of this stability variable measuring
the adversarial examples sensitivity. For any given architecture, we provide a measure of risk to
adversarial attacks. (3) Recently, the deep learning community has focused on the reminiscent theory of
flat and sharp minima to provide generalization guarantees. Flat minima are regions in the parameter
space associated with great generalization capacities. We will first, prove the equivalence between flat
minima and spline smoothness. After bridging those theories, we will motivate a novel regularization
technique pushing the learning of DNNs towards flat minima, maximizing generalization
performances. (4) From (1) we will reinterpret DNNs as template matching algorithms. When coupled
with insights derived from (2), we will integrate unlabeled data information into the network during
learning. To do so, we will propose to guide DNNs templates towards their input via a scheme
assimilated as a reconstruction formula for DNNs. This inversion can be computed efficiently by back-
propagation leading to no computational overhead. From this, any semi-supervised technique can be
used out-of-the-box with current DNNs where we provide state-of-the-art results. Unsupervised tasks would also become reachable to DNNs, a
task considered as the keystone of learning for the neuro-science community.
To date, those problematics have been studied independently leading to over-specialized solutions
generally topology specific and cumbersome to incorporate into a pre-existing pipeline. On the other
hand, all the proposed solutions necessitate negligible software updates, suited for efficient large-scale
deployment.

\begin{center}
\begin{tabular}{|p{0.99\linewidth}|}\hline 
\vspace{-0.5cm}
\begin{center}
\textbf{\underline{Non-exhaustive list of the main contributions}}
\end{center}
1) We first develop spline operators (SOs) \ref{splineoperator}, a natural generalization of multivariate spline functions as well as their linear case (LSOs). LSOs are shown to ''span'' DNNs layers, being restricted cases of LSOs \ref{subsec:nn}. From this, composition of those operators lead to the explicit analytical input-output formula of DNNs, for any architecture \ref{subsub:networks}. We then dive into some analysis:
\begin{itemize}
  \setlength\itemsep{0.1em}
      \item Interpret DNNs as template matching machines, provide ways to visualize and analyze the inner representation a DNN has of its input w.r.t each classes and understand the prediction \ref{sub:ada}.
    \item Understand the impact of design choices such as skip-connections and provide conditions for ''good'' weight initialization schemes \ref{subsub:networks}.
    \item Derive a simple methodology to compute the Lipschitz constant of any DNN, quantifying their stability  and derive strategies for adversarial example robustness \ref{subsub:lip}.
    \item Study the impact of depth and width for generalization and class separation, orbit learning \ref{orbits}.
\end{itemize}

2)Secondly, we prove the following implications for any DNN with the only assumption that all inputs have same energy, as $||X_n||^2=K>0,\forall n$.

\begin{tikzpicture}[node distance = 5.3cm, auto]
    \node [block] (regularization) {\textbf{Regularization:}\\Tikhonov, L1, dropout or any introduced noise in the network};
    \node [block, below =0.6cm of regularization] (memo) {\textbf{Dataset Memorization:}\\colinearity of the templates towards dataset memorization, introduction of ''good'' and ''bad'' memorization};
    \node [block, below =0.6cm of memo] (recon) {\textbf{Reconstruction:}\\provide ways to detect ''ambiguous'' inputs,anomaly detection, allow for semi-sup with state-of-the-art performances and unsupervised training, clustering};
    \node [block, right of=memo,below=0.1cm of regularization] (flat) {\textbf{Flat-Minima, Generalization:}\\define the concept of generalization for DNNs, links with regularization.};
    \node [block,below=0.8cm of flat] (design) {\textbf{Systematic DNN design:}\\new quantitative measure of dataset specific generalization capacities of untrained DNN for fast architecture search prior learning.};
    \node [block, left of =memo] (lip) {\textbf{Lipschitz constant:}\\measure, control DNNs stability, contractive DNNs and parsimonious templates};
    \node [block, below =1cm of lip] (adv) {\textbf{Robustness to adv. examples:}\\prevent attacks in a systematic way, bounds on DNNs attacks sensitivity};
    
\draw[latex'-latex',double] (regularization) --  (lip.north) node [midway, fill=white,text opacity=1,fill opacity=0.] {\ref{subsub:lip}}; ;
\draw[-latex',double] (regularization) --  (flat.north) node [midway, fill=white,text opacity=1,fill opacity=0.] {\ref{generalization}};
\draw[latex'-latex',double] (regularization) --  (memo) node [midway, fill=white,text opacity=1,fill opacity=0.] {\ref{subsub:memo}};
\draw[latex'-latex',double] (lip) --  (adv) node [midway, fill=white,text opacity=1,fill opacity=0.] {\ref{subsub:lip}};
\draw[latex'-latex',double] (memo) --  (adv) node [midway,left=-9pt, fill=white,text opacity=1,fill opacity=0.] {\ref{subsub:lip}};
\draw[-latex',double] (memo) --  (recon) node [midway, fill=white,text opacity=1,fill opacity=0.] {\ref{subsub:all_recon}};
\draw[-latex',double] (flat) --  (design) node [midway, fill=white,text opacity=1,fill opacity=0.] {};
\draw[-latex',double] (recon.south east) --  (design.south) node [midway,left=9pt,above=-9pt, fill=white,text opacity=1,fill opacity=0.] {};
\end{tikzpicture}

\\ \hline
\end{tabular}
\end{center}

%% file: background.tex
\section{Background: Deep Neural Networks for Function Approximation}\label{background}

Most of applied mathematics interests take the form of function approximation. Two main cases arise, one where the target function $f$ to approximate is known and one where only a set of samples $(X_n,f(X_n))_{n=1}^N$ are observed, providing limited information on the domain-codomain structure of $f$. The latter case is the one of supervised learning. Given the {\em training set} $\mathcal{D}:=\{(X_n,Y_n)_{n=1}^N\}$ with $Y_n:=f(X_n)$, the unknown functional $f$ is estimated through the approximator $\hat{f}$. Finding an approximant $\hat{f}$ with correct behaviors on $\mathcal{D}$ is usually an ill-posed problem with many possible solutions. Yet, each one might behave differently for new observations, leading to different generalization performances. Generalization is the ability to replicate the behavior of $f$ on new inputs not present in $\mathcal{D}$ thus not exploited to obtain $\hat{f}$. Hence, one seeks for an approximator $\hat{f}$ having the best generalization performance. In some applications, the unobserved $f$ is known to fulfill some properties such as boundary and regularity conditions for PDE approximation. In machine learning however, the lack of physic based principles does not provide any property constraining the search for a good approximator $\hat{f}$ except the performance measure based on the training set $\mathcal{D}$ and an estimate of generalization performance based on a {\em test set}. To tackle this search, one commonly resorts to a parametric functional $\hat{f}_\Theta$ 
where $\Theta$ contains all the free parameters controlling the behavior of $\hat{f}_\Theta$. The task thus ''reduces'' to finding the optimal set of parameters $\Theta^*$ minimizing the empirical error on the training set and maximizing empirical generalization performance on the test set. 
We now refer to this estimation problem as a regression problem if $Y_n$ is continuous and a classification problem if $Y_n$ is categorical or discrete. We also restrict ourselves to $\hat{f}_\Theta$ being a Deep Neural Network (DNN) and denote $f_\Theta := \hat{f}_\Theta$. Also, $x$ is used for a generic input as opposed to the $n^{th}$ given sample $X_n$.

DNNs are a powerful and increasingly applied machine learning framework for complex prediction tasks like object and speech recognition. In fact, they are proven to be universal function approximators\cite{cybenko1989approximation,hornik1989multilayer}, fitting perfectly the context of function approximation of supervised learning described above.
There are many flavors of DNNs, including convolutional, residual, recurrent, probabilistic, and beyond.  
Regardless of the actual network topology, we represent the mapping from the input signal $x\in\mathbb{R}^D$ to the output prediction $\widehat{y} \in \mathbb{R}^C$ as $f_\Theta: \mathbb{R}^D \rightarrow \mathbb{R}^C$. By its parametric nature, the behavior of $f_\Theta$ is governed by its underlying parameters $\Theta$.
All current deep neural networks boil down to a composition of $L$ ''layer mappings'' denoted by
\begin{equation}
f_\Theta(x)=(f^{(L)}_{\theta^{(L)}} \circ \dots \circ f^{(1)}_{\theta^{(1)}})(x),\;\;\;\Theta=\{\theta^{(1)},\dots,\theta^{(L)}\}.
\end{equation}
In all the following cases, a neural network layer at level $(\ell)$ is an operator $f^{(\ell)}_{\theta^{(l)}}$ that takes as input a vector-valued signal $\bz^{(\ell-1)}(x)\in \mathbb{R}^{D^{(\ell-1)}}$ which at $\ell=0$ is the input signal $\bz^{(0)}(x):=x$  and produces a vector-valued output $\bz^{(\ell)}(x)\in \mathbb{R}^{D^{(\ell)}}$.
This succession of mappings is in general non-commutative, making the analysis of the complete sequence of generated signals crucial, denoted by
\begin{equation}
\bz^{(\ell)}(x)=(f^{(\ell)}_{\theta^{(\ell)}} \circ \dots \circ f^{(1)}_{\theta^{(1)}})(x),\ell \in \{1,\dots,L\}.
\end{equation}
For concreteness, we will focus on processing $K$-channel inputs $x$, such as RGB images, stereo signals, as well as multi-channel representations $z^{(\ell)},\ell=1,\dots,L$ which we refer to as a ``signal''. 
This signal is indexed $z^{(\ell)}_{c,i,j},c=1,\dots,C^{(\ell)},i=1,\dots,I^{(\ell)}$, $j=1,\dots,J^{(\ell)},\ell=1,\dots,L$, where $i,j$ are usually spatial coordinates, and $c$ is the channel. Any signal with greater index-dimensions fall under the following analysis by adaptation of the notations and operators. Hence, the volume $z^{(\ell)}$ is of shape $(C^{(\ell)},I^{(\ell)},J^{(\ell)})$ with $(C^{(0)},I^{(0)},J^{(0)})=(K,I,J)$ and $(C^{(L)},I^{(L)},J^{(L)})=(C,1,1)$. For consistency with the introduced layer mappings, we will use $\bz^{(\ell)}$, the flattened version of $z^{(\ell)}$ as depicted in Fig. \ref{fig:reshape}. The dimension of $\bz^{(\ell)}$ is thus $D^{(\ell)}=C^{(\ell)}I^{(\ell)}J^{(\ell)}$.  
In this section, we introduce the basic concepts and notations of the main used layers enabling to create state-of-the-art DNNs as well as standard training techniques to update the parameters $\Theta$.

\subsection{Layers Description}
\begin{figure}[t!]
    \centering
    \includegraphics[width=3in]{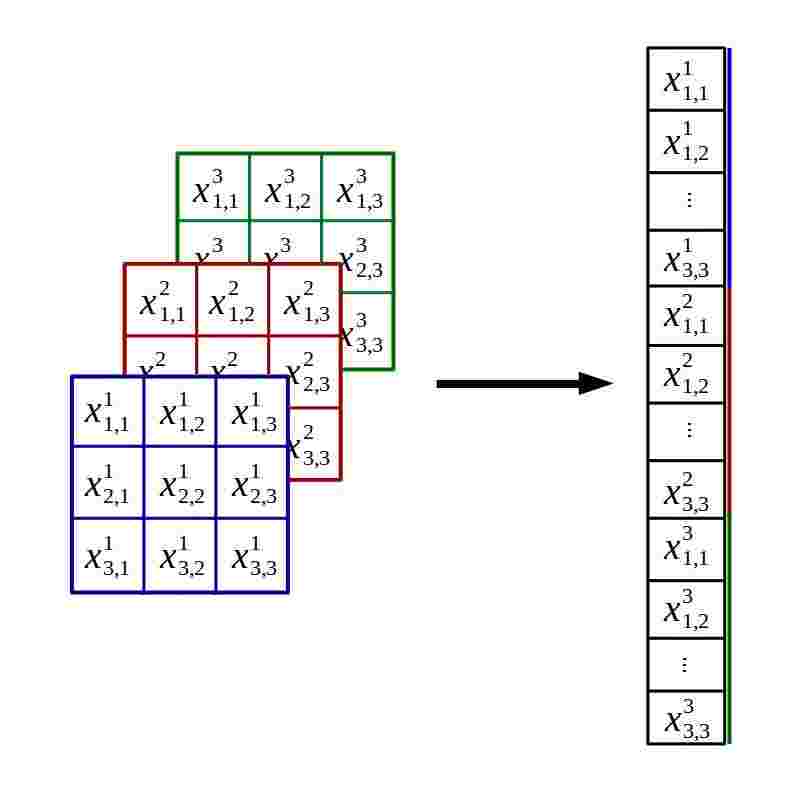}
    \caption{Reshaping of the multi-channel signal $z^{(\ell)}$ of shape $(3,3,3)$ to form the vector $\bz^{(\ell)}$ of dimension $27$.}
    \label{fig:reshape}
\end{figure}
In this section we describe the common layers one can use to create the mapping $f_\Theta$. The notations we introduce will be used throughout the report. We now describe the following: Fully-connected;Convolutional; Nonlinearity;Sub-Sampling;Skip-Connection;Recurrent layers.

\paragraph{Fully-Connected Layer}
A Fully-Connected (FC) layer is at the origin of DNNs known as Multi-Layer Perceptrons (MLPs) \cite{pal1992multilayer} composed exclusively of FC-layers and nonlinearities. This layer performs a linear transformation of its input as
\begin{align}\boxed{
    f^{(\ell)}_W(\bz^{(\ell-1)}(x))=W^{(\ell)}\bz^{(\ell-1)}(x)+b^{(\ell)}.}
\end{align}
The internal parameters $\theta^{(\ell)}=\{W^{(\ell)},b^{(\ell)}\}$ are defined as $W^{(\ell)} \in \mathbb{R}^{D^{(\ell)}\times D^{(\ell-1)}}$ and $b \in \mathbb{R}^{D^{(\ell)}}$. This linear mapping produces an output vector $\bz^{(\ell)}$ of length $D^{(\ell)}$. In current topologies, FC layers are used at the end of the mapping, as layers $L$ and $L-1$, for their capacity to perform nonlinear dimensionality reduction in order to output $C$ output values. However, due to their high number of degrees of freedom $(D^{(\ell)}\times D^{(\ell-1)}+D^{(\ell)})$ and the unconstrained internal structure of $W^{(\ell)}$, MLPs inherit poor generalization performances for common perception tasks as demonstrated on computer vision tasks in \cite{zhang2016understanding}.

\paragraph{Convolutional Layer}
The greatest accuracy improvements in DNNs occurred after the introduction of the {\em convolutional layer}. Through convolutions, it leverages one of the most natural operation used for decades in signal processing and template matching.
In fact, as opposed to the FC-layer, the convolutional layer is the corestone of DNNs dealing with perceptual tasks thanks to their ability to perform local feature extractions from their input. It is defined as
\begin{align}\boxed{
f^{(\ell)}_C(\bz^{(\ell-1)}(x))=\bC^{(\ell)}\bz^{(\ell-1)}(x)+b^{(\ell)}}.
\end{align}
where a special structure is defined on $\bC^{(\ell)}$ so that it performs multi-channel convolutions on the vector $\bz^{(\ell-1)}$. To highlight this fact, we first remind the multi-channel convolution operation performed on the unflatenned input $z^{(\ell-1)}(x)$ of shape $(C^{(\ell-1)},I^{(\ell-1)},J^{(\ell-1)})$ given a filter bank $W^{(\ell)}$ composed of $C^{(\ell)}$ filters, each being a $3D$ tensor of shape $(C^{(\ell-1)},M^{(\ell)},N^{(\ell)})$ with $M^{(\ell)}\leq I^{(\ell-1)},N^{(\ell)}\leq J^{(\ell-1)}$. Hence $W^{(\ell)} \in \mathbb{R}^{C^{(\ell)}\times C^{(\ell-1)}\times M^{(\ell)}\times N^{(\ell)}}$ with $C^{(\ell-1)}$ representing the filters depth, equal to the number of channels of the input, and $(M^{(\ell)},N^{(\ell)})$ the spatial size of the filters.
The application of the linear filters $W^{(\ell)}$ on the signal form another multi-channel signal as
\begin{align}
(W^{(\ell)} \star z^{(\ell-1)}(x))_{c,i,j} ~=&~ \sum_{k=1}^{C^{(\ell-1)}} (W^{(\ell)}_{c,k}\star z^{(\ell-1)}_{k}(x))_{i,j}\nonumber \\
=&~\sum_{k=1}^{C^{(\ell-1)}} \sum_{m=1}^{M^{(\ell)}} \sum_{n=1}^{N^{(\ell)}} W^{(\ell)}_{c,k,m,n}  z^{(\ell-1)}_{k,i-m,j-n}(x),
\end{align}
where the output of this convolution contains $C^{(\ell)}$ channels, the number of filters in $W^{(\ell)}$. Then a bias term is added for each output channel, shared across spatial positions. We denote this bias term as $\xi \in \mathbb{R}^{C^{(\ell)}}$.
As a result, to create channel $c$ of the output, we perform a $2D$ convolution of each channel $k=1,\dots,C^{(\ell-1)}$ of the input with the impulse response $W^{(\ell)}_{c,k}$ and then sum those outputs element-wise over $k$ to finally add the bias leading to $z^{(\ell)}(x)$ as
\begin{align}
z^{(\ell)}_{c}(x) ~=~ \sum_{k=1}^{C^{(\ell-1)}} (W^{(\ell)}_{c,k}\star z^{(\ell-1)}_{k}(x))+\xi_c.
\label{eq:conv1}
\end{align}
In general, the input is first transformed in order to apply some boundary conditions such as zero-padding, symmetric or mirror. Those are standard padding techniques in signal processing \cite{mallat1999wavelet}.
We now describe how to obtain the matrix $\bC^{(\ell)}$ and vector $b^{(\ell)}$ corresponding to the operations of Eq. \ref{eq:conv1} but applied on the flattened input $\bz^{(\ell-1)}(x)$ and producing the output vector $\bz^{(\ell)}$.
The matrix $\bC^{(\ell)}$ is obtained by replicating the filter weights $W^{(\ell)}_{c,k}$ into the circulent-block-circulent matrices $\bW^{(\ell)}_{c,k},c=1,\dots,C^{(\ell)},k=1,\dots,C^{(\ell-1)}$ \cite{jayaraman2009digital}
and stacking them into the super-matrix $\bC^{(\ell)}$
\begin{align}
\bC^{(\ell)} ~=~ 
\begin{bmatrix}
    \bW^{(\ell)}_{1,1} & \bW^{(\ell)}_{1,2} & \dots & \bW^{(\ell)}_{1,C^{(\ell-1)}} \\
    \bW^{(\ell)}_{2,1} & \bW^{(\ell)}_{2,2} & \dots & \bW^{(\ell)}_{2,C^{(\ell-1)}} \\
    \vdots & \vdots & \ddots & \vdots \\
    \bW^{(\ell)}_{C^{(\ell)},1} & \bW^{(\ell)}_{C^{(\ell)},2} & \dots & \bW^{(\ell)}_{C^{(\ell)},C^{(\ell-1)}}
\end{bmatrix}.
\label{eq:bigC}
\end{align}
We provide an example in Fig. \ref{fig:w1} for $\bW^{(\ell)}_{c,k}$ and $\bC^{(\ell)}$.
\begin{figure}[!htb]
    \centering
    \begin{minipage}{.5\textwidth}
        \centering
        \includegraphics[width=0.9\linewidth]{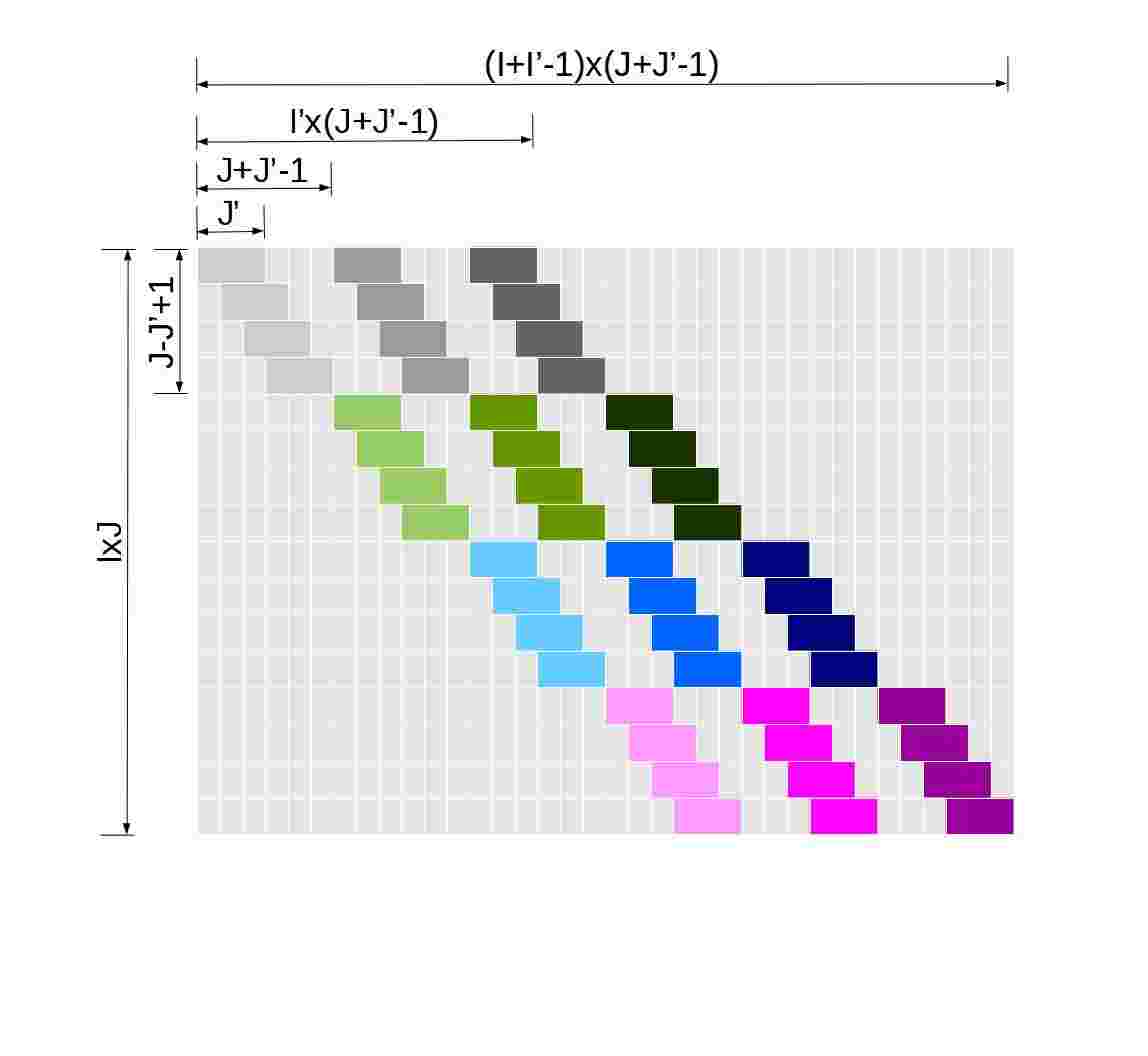}
    \end{minipage}%
    \begin{minipage}{0.5\textwidth}
        \centering
        \includegraphics[width=0.9\linewidth]{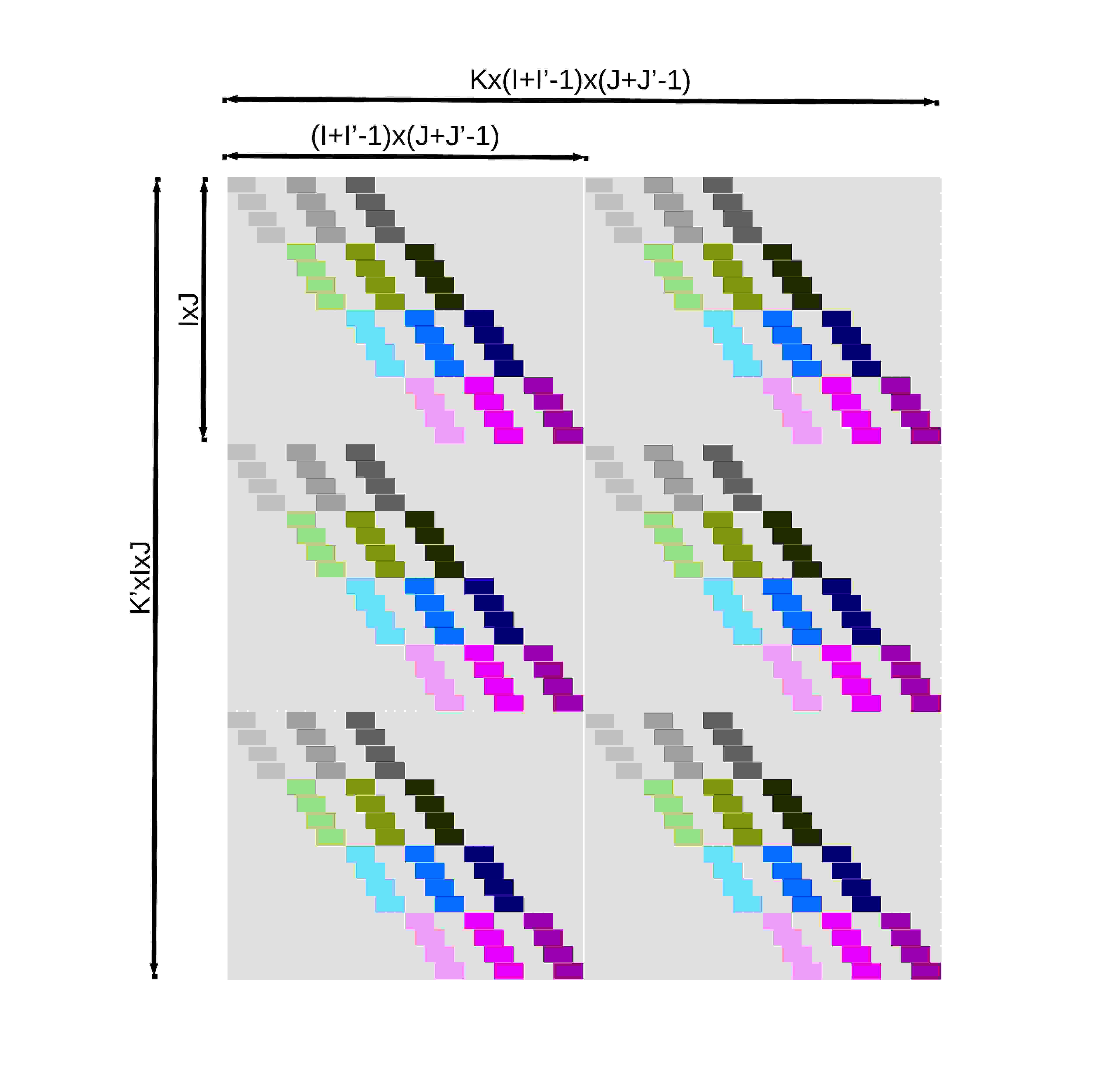}
        \label{fig:prob1_6_1}
    \end{minipage}
        \caption{Left: depiction of one convolution matrix $\bW_{k,l}$.Right: depiction of the super convolution matrix $\bC$.}        \label{fig:w1}
\end{figure}
By the sharing of the bias across spatial positions, the bias term $b^{(\ell)}$ inherits a specific structure. It is defined by replicating $\xi^{(\ell)}_c$ on all spatial position of each output channel $c$:
\begin{equation}
    b^{(\ell)}_{d}=\xi^{(\ell)}_c,d=I^{(\ell)}J^{(\ell)}(c-1)+1,\dots,I^{(\ell)}J^{(\ell)}(c-1),\;c=1,\dots,C^{(\ell)}.
\end{equation}
The internal parameters of a convolutional layer are $\theta^{(\ell)}=\{W^{(\ell)},\xi^{(\ell)}\}$. The number of degrees of freedom for this layer is much less than for a FC-layer, it is of $C^{(\ell)}C^{(\ell-1)}M^{(\ell)}N^{(\ell)}+C^{(\ell)}$.
If the convolution is circular, the spatial size of the output is preserved leading to $(I^{(\ell)},J^{(\ell)})=(I^{(\ell-1)},J^{(\ell-1)})$ and thus the output dimension only changes in the number of channels.
Taking into account the special topology of the input by constraining the $\bC^{(\ell)}$ matrix to perform convolutions coupled with the low number of degrees of freedom while allowing a high-dimensional output leads to very efficient training and generalization performances in many perceptual tasks which we will discuss in details in \ref{generalization}.
While there are still difficulties to understand what is encoded in the filters $W^{(\ell)}$, it has been empirically shown that for images, the first  filter-bank $W^{(1)}$ applied on the input images converges toward an over-complete Gabor filter-bank, considered as a natural basis for images \cite{meyer1993algorithms,olshausen1996emergence}. Hence, many signal processing tools and results await to be applied for analysis.

\paragraph{Element-wise Nonlinearity Layer}
A {\em scalar/element-wise nonlinearity} layer applies a nonlinearity $\sigma$ to each entry of a vector and thus preserve the input vector dimension $D^{(\ell)}=D^{(\ell-1)}$.
As a result, this layer produces its output via application of $\sigma:\mathbb{R}\rightarrow \mathbb{R}$ across all positions as
\begin{equation}\boxed{
    f^{(\ell)}_\sigma(\bz^{(\ell-1)}(x))_d=\sigma \left(\bz^{(\ell-1)}_d(x)\right),d=1,\dots,D^{(\ell)}.}
\end{equation}
The choice of nonlinearity greatly impacts the learning and performances of the DNN as for example sigmoids and tanh are known to have vanishing gradient problems for high amplitude inputs, while ReLU based activation lead to unbounded activation and dying neuron problems.
Typical choices include
\begin{itemize}
    \item Sigmoid: $\sigma_{sig}(u)=\frac{1}{1+e^{-u}}$,
    \item tanh: $\sigma_{tanh}(u):=2\sigma_{sig}(2u)-1$,
    \item ReLU: $\sigma_{relu}(u)=\max(u,0)$,
    \item Leaky ReLU: $\sigma_{lrelu}(u) = \max(u,0)+\min(\eta u,0),\;\eta>0$,
    \item Absolute Value: $\sigma_{abs}(u) = \max(u,0)+\max(-u,0)$.
\end{itemize}
The presence of nonlinearities in DNNs is crucial as otherwise the composition of linear layers would produce another linear layer, with factorized parameters. 
When applied after a FC-layer or a convolutional layer we will consider the linear transformation and the nonlinearity as part of one layer. Hence we will denote $f^{(\ell)}_{\theta^{(\ell)}}(\bz^{(\ell-1)})=(f^{(\ell)}_\sigma \circ f^{(\ell)}_W)(\bz^{(\ell-1)})$ for example.

\paragraph{Pooling Layer}
A {\em pooling} layer operates a sub-sampling operation on its input according to a sub-sampling policy $\rho$ and a collection of regions on which $\rho$ is applied. We denote each region to be sub-sampled by $R_d,d=1,\dots,D^{(\ell)}$ with $D^{(\ell)}$ being the total number of pooling regions. Each region contains the set of indices on which the pooling policy is applied leading to
\begin{equation}\boxed{
    f^{(\ell)}_\rho(\bz^{(\ell-1)}(x))_d=\rho (\bz^{(\ell-1)}_{R_d}(x)),d=1,\dots,D^{(\ell)}.}
\end{equation}
where $\rho$ is the pooling operator and  $\bz^{(\ell-1)}_{R_d}(x)=\{\bz^{(\ell-1)}_i(x),i\in R_d\}$. Usually one uses mean or max pooling defined as
\begin{itemize}
    \item Max-Pooling: $\rho_{max}(\bz^{(\ell-1)}_{R_d}(x))=\max_{i\in R_d} \bz^{(\ell-1)}_i(x)$,
    \item Mean-Pooling: $\rho_{mean}(\bz^{(\ell-1)}_{R_d}(x))=\frac{1}{Card(R_d)}\sum_{i\in R_d} \bz^{(\ell-1)}_i(x)$.
\end{itemize}
The regions $R_d$ can be of different cardinality $\exists d_1,d_2 | Card(R_{d_1})\not = Card(R_{d_2}) $ and can be overlapping $\exists d_1,d_2 | Card(R_{d_1})\cap  Card(R_{d_2})\not = \emptyset $. However, in order to treat all input dimension, it is natural to require that each input dimension belongs to at least one region: $\forall k \in \{1,\dots, D^{(\ell-1)}\}, \exists d \in \{1,\dots, D^{(\ell)}\}| k\in R_d$.
The benefits of a pooling layer are three-fold. Firstly, by reducing the output dimension it allows for faster computation and less memory requirement. Secondly, it allows to greatly reduce the redundancy of information present in the input $\bz^{(\ell-1)}$. In fact, sub-sampling, even though linear, is common in signal processing after filter convolutions. Finally, in case of max-pooling, it allows to only backpropagate gradients through the pooled coefficient enforcing specialization of the neurons. The latter is the corestone of the winner-take-all strategy stating that each neuron specializes into what is performs best. Similarly to the nonlinearity layer, we consider the pooling layer as part of its previous layer. 

\paragraph{Skip-Connection} A skip-connection layer can be considered as a bypass connection added between the input of a layer and its output. Hence, it allows for the input of a layer such as a convolutional layer or FC-layer to be linearly combined with its own output. The added connections lead to better training stability and overall performances as there always exists a direct linear link from the input to all inner layers. Simply written, given a layer $f^{(\ell)}_{\theta^{(\ell)}}$ and its input $z^{(\ell-1)}(x)$, the skip-connection layer is defined as
\begin{equation}\boxed{
    f^{(\ell)}_s(\bz^{(\ell-1)(x)};f^{(\ell)}_{\theta^{(\ell)}})=\bz^{(\ell-1)}(x)+f^{(\ell)}_{\theta^{(\ell)}}(\bz^{(\ell-1)}(x)).}
\end{equation}
In case of shape mis-match between $\bz^{(\ell-1)}(x)$ and $f^{(\ell)}_{\theta^{(\ell)}}(\bz^{(\ell-1)}(x))$, a ''reshape'' operator is applied to $\bz^{(\ell-1)}(x)$ before the element-wise addition. Usually this is done via a spatial down-sampling and/or through a convolutional layer with filters of spatial size $(1,1)$.

\paragraph{Recurrent} Finally, another type of layer is the recurrent layer which aims to act on time-series. It is defined as a recursive application along time $t=1,\dots,T$ by transforming the input as well as using its previous output. The most simple form of this layer is a fully recurrent layer defined as 
\begin{align}
    \bz^{(1,t)}(x)&=\sigma\left( W^{(in,h_1)}x^t+W^{(h_1,h_1)}\bz^{(1,t-1)}(x)+b^{(1)}  \right)\\
    \bz^{(\ell ,t)}(x)&=\sigma\left( W^{(in,h_\ell)}x^t+W^{(h_{\ell-1},h_\ell)}\bz^{(\ell-1),t}(x)+W^{(h_{\ell},h_\ell)}\bz^{(\ell,t-1)}(x)+b^{(\ell)}  \right)
\end{align}
while some applications use recurrent layers on images by considering the serie of ordered local patches as a time serie, the main application resides in sequence generation and analysis especially with more complex topologies such as LSTM\cite{graves2005framewise} and GRU\cite{chung2014empirical} networks.
We depict the topology example in Fig. \ref{fig_rnn}.
\begin{figure}
  \centering
  \includegraphics[width=6in]{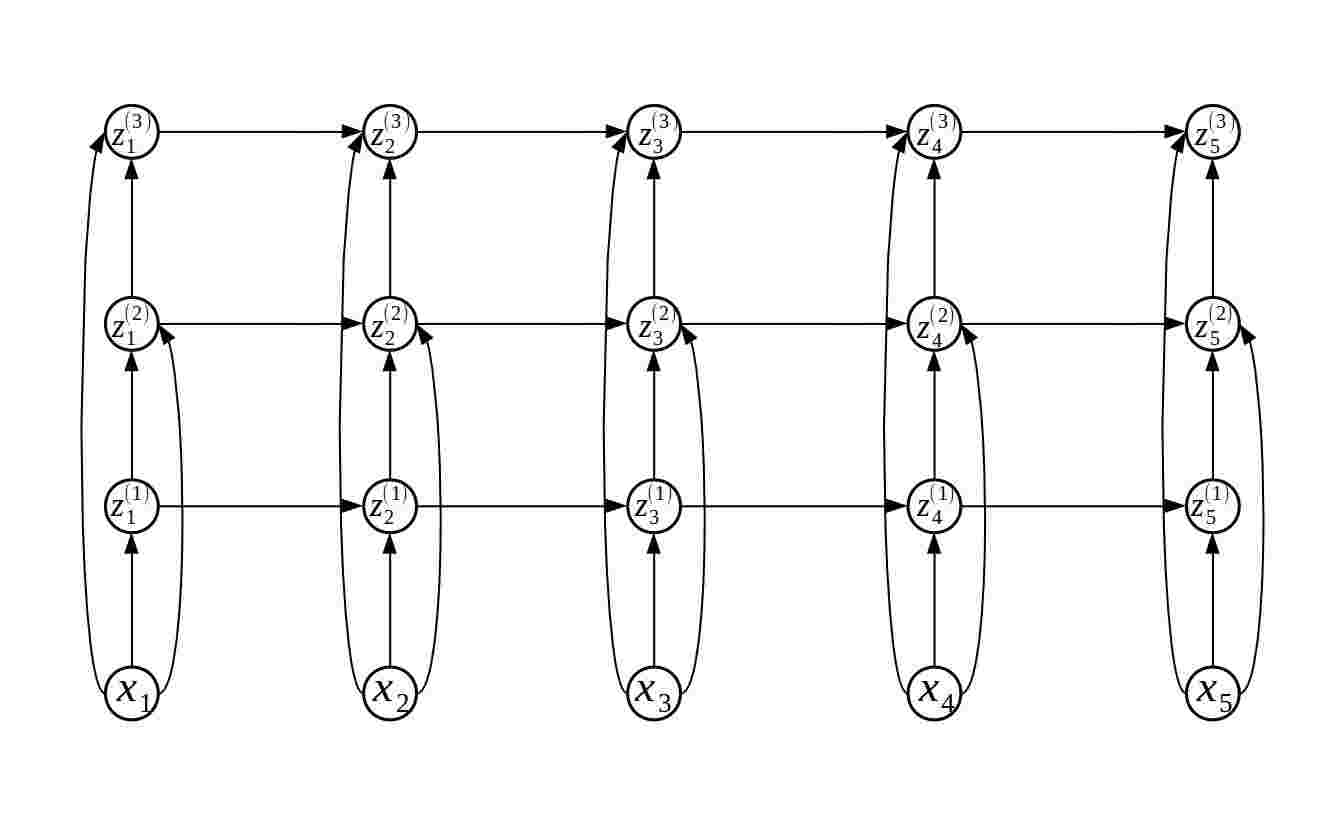}
  \caption{Depiction of a simple RNN with $3$ layers. Connections highlight input-output dependencies.}
\label{fig_rnn}
\end{figure}

\subsection{Deep Convolutional Network}\label{DCN}
The combination of the possible layers and their order matter greatly in final performances, and while many newly developed stochastic optimization techniques allow for faster learning, a sub-optimal layer chain is almost never recoverable.
We now describe a ''typical'' network topology, the deep convolutional network (DCN), to highlight the way the previously described layers can be combined to provide powerful predictors. Its main development goes back to \cite{lecun1995learning} for digit classification.
A DCN is defined as a succession of blocks made of $3$ layers : Convolution $\rightarrow$ Element-wise Nonlinearity $\rightarrow$ Pooling layer. 
In a DCN, several of such blocks are cascaded end-to-end to create a sequence of activation maps followed usually by one or two FC-layers. Using the above notations, a single block can be rewritten as $(f_\rho \circ f_\sigma \circ f_C)$. Hence a basic model with $2$ blocks and $2$ FC-layers is defined as
\begin{align}
    \bz^{(4)}(x)=&(\underbrace{f^{(4)}_W \circ f^{(3)}_\sigma \circ f^{(3)}_W}_{\text{MLP}} \circ \underbrace{f^{(2)}_\rho \circ f^{(2)}_\sigma \circ f^{(2)}_C }_{\text{ Block 2}}\circ \underbrace{f^{(1)}_\rho \circ f^{(1)}_\sigma \circ f^{(1)}_C}_{\text{Block 1}})(x)\\
    =&(f^{(4)}_{\theta^{(4)}}\circ f^{(3)}_{\theta^{(3)}} \circ f^{(2)}_{\theta^{(2)}} \circ f^{(1)}_{\theta^{(1)}})(x).
\end{align}
The astonishing results that a DCN can achieve come from the ability of the blocks to convolve the learned filter-banks with their input, ''separating'' the underlying features present relative to the task at hand. This is followed by a nonlinearity and a spatial sub-sampling to select, compress and reduce the redundant representation while highlighting task dependent features. Finally, the MLP part simply acts as a nonlinear classifier, the final key for prediction. The duality in the representation/high-dimensional mappings followed by dimensionality reduction/classification is a core concept in machine learning referred as: pre-processing-classification.

\subsection{Learning}

In order to optimize all the weights $\Theta$ leading to the predicted output $\hat{y}(x)$, one disposes of (1) a labeled dataset $\mathcal{D}=\{(X_n,Y_n),n=1,...,N\}$, (2) a loss function $\mathcal{L}:\mathbb{R}^C \times \mathbb{R}^C \rightarrow \mathbb{R}$, (3) a learning policy to update the parameters $\Theta$.
In the context of classification, the target variable $Y_n$ associated to an input $X_n$ is  categorical $Y_n \in \{1,\dots,C\}$. 
In order to predict such target, the output of the last layer of a network $\bz^{(L)}(X_n)$ is transformed via a softmax nonlinearity\cite{de2015exploration}. It is used to transform $\bz^{(L)}(X_n)$ into a probability distribution and is defined as
\begin{equation}
    \hat{y}_c(X_n)=\frac{e^{\bz^L_c(X_n)}}{\sum_c e^{\bz^L_c(X_n)}},\in (0,1), c=1,\dots,C
\end{equation}
thus leading to $\hat{y}_c(X_n)$ representing $\mathbb{P}(\text{class of $X_n$ is $c$}|X_n)$. 
The used loss function quantifying the distance between $\hat{y}(X_n)$ and $Y_n$ is the cross-entropy (CE) defined as
\begin{align}
    \mathcal{L}_{CE}(Y_n,\hat{y}(X_n)) &= -\log\left(\hat{y}_{Y_n}(X_n)\right)\\
    &= -\bz^L_{Y_n}(X_n)+\log\left(\sum_{c=1}^C e^{\bz^L_{c}(X_n)}\right).
\end{align}
For regression problems, the target $Y_n$ is continuous and thus the final DNN output is taken as the prediction $\hat{y}(X_n)=\bz^{(L)}(X_n)$. The loss function $\mathcal{L}$ is usually the ordinary squared error (SE) defined as
\begin{align}
    \mathcal{L}_{SE}(Y_n,\hat{y}(X_n))&= \sum_{c=1}^C \left( Y_{n,c}-\hat{y}(X_n)_c\right)^2.
\end{align}

Since all of the operations introduced above in standard DNNs are differentiable almost everywhere with respect to their parameters and inputs, given a training set and a loss function, one defines an update strategy for the weights $\Theta$. This takes the form of an iterative scheme based on a first order iterative optimization procedure. Updates for the weights are computed on each input and usually averaged over {\em mini-batches} containing $B$ exemplars with $B\ll N$. This produces an estimate of the ''correct'' update for $\Theta$ and is applied after each mini-batch. Once all the training instances of $\mathcal{D}$ have been seen, after $N/B$ mini-batches, this terminates an {\em epoch}. The dataset is then shuffled and this procedure is performed again. Usually a network needs hundreds of epochs to converge.
For any given iterative procedure, the updates are computed for all the network parameters by {\em backpropagation} \cite{hecht1988theory}, which follows from applying the chain rule of calculus. Common policies are Gradient Descent (GD) \cite{rumelhart1988learning} being the simplest application of backpropagation, Nesterov Momentum \cite{bengio2013advances} that uses the last performed updates in order to ''accelerate'' convergence and finally more complex adaptive methods with internal hyper-parameters updated based on the weights/updates statistics such as  Adam\cite{kingma2014adam}, Adadelta\cite{zeiler2012adadelta}, Adagrad\cite{duchi2011adaptive}, RMSprop\cite{tieleman2012lecture}, \dots
Finally, to measure the actual performance of a trained network in the context of classification, one uses the accuracy loss defined as
\begin{equation}
    \mathcal{L}_{AC}(Y_n,\hat{y}(X_n))=-1_{\{Y_n=\argmax_{c}\hat{y}_c(X_n)\}},
\end{equation}
defined such that smaller value is better,

%% file: affine.tex
\section{Understand Deep Neural Networks Internal Mechanisms}
\label{sec:layers}

In this section we develop spline operators, a generalization of spline function which are also generalized DNN layers. By doing so, we will open DNNs to explicit analysis and especially understand their behavior and potential through the spline region inference. DNNs will be shown to leverage template matching, a standard technique in signal processing to tackle perception tasks.

Let first motivate the need to adopt the theory of spline functions for deep learning and machine learning in general.
As described in Sec.\ \ref{background}, the task at hand is to use parametric functionals $f_\Theta$ to be able to understand\cite{cheney1980approximation}, predict, interpolate the world around us\cite{reinsch1967smoothing}. For example, partial differential equations allow to approximate real world physics \cite{bloor1990representing,smith1985numerical} based on grounded principles. For this case, one knows the underlying laws that must be fulfilled by $f_\Theta$. In machine learning however, one only disposes of $N$ observed inputs $X_n,n=1,\dots,N$ or input-output pairs $(X_n,Y_n)_{n=1}^N$. To tackle this approximation problem, splines offer great advantages.
From a computational regard, polynomials are very efficient to evaluate via for example the Horner scheme\cite{pena2000multivariate}. 
Yet, polynomials have ''chaotic'' behaviors especially as their degree grows, leading to the Runge's phenomenon\cite{boyd2009divergence}, synonym of extremely poor interpolation and extrapolation capacities. On the other hand, low degree polynomials are not flexible enough for modeling arbitrary functionals. Splines, however, are defined as a collection of polynomials each one acting on a specific region of the input space. The collection of possible regions $\Omega=\{\omega_1,\dots,\omega_R\}$ forms a partition of the input space. On each of those regions $\omega_r$, the associated and usually low degree polynomial $\phi_{r}$ is used to transform the input $x$.
Through the per region activation, splines allow to model highly nonlinear functional, yet, the low-degree polynomials avoid the Runge phenomenon. Hence, splines are the tools of choice for functional approximation if one seeks robust interpolation/extrapolation performances without sacrificing the modeling of potentially very irregular underlying mappings $f$. 

In fact, as we will now describe in details, current state-of-the-art DNNs are linear spline functions and we now proceed to develop the notations and formulation accordingly.
Let first remind briefly the case of multivariate linear splines.
\begin{defn}
Given a partition $\Omega=\{\omega_1,\dots,\omega_R\}$ of $\mathbb{R}^D$, we denote \textbf{multivariate spline}  with local mappings $\Phi=\{\phi_{1},\dots,\phi_{R}\}$, with $\phi_{\omega_r}:\mathbb{R}^D\rightarrow \mathbb{R}$ the mapping
\begin{align}
    s[\Phi,\Omega](x)=&\sum_{r=1}^R \phi_{r}(x)1_{\{x \in \omega_r\}}\\
    =&\phi[x](x),
\end{align}
where the input dependent selection is abbreviated via
\begin{equation}
    \phi[x]:=\phi_{r} \text{ s.t. } x \in \omega_r.
\end{equation}
If the local mappings $\phi_{r}$ are linear we have $\phi_{r}(x)=\langle a_{r},x\rangle+b_{r}$ with $a_{r} \in \mathbb{R}^D$ and $b_{r}\in \mathbb{R}$. We denote this functional as a \textbf{multivariate linear spline}:
\begin{align}
    s[\textbf{a},\textbf{b},\Omega](x)=&\sum_{r=1}^R \left(\langle a_r,x\rangle + b_r \right)1_{\{x \in \omega_r\}}\\
    =&a[x]^Tx+b[x].
\end{align}
where we explicit the polynomial parameters by $\textbf{a}=\{a_r,\dots,a_{R}\}$ and $\textbf{b}=\{b_{1},\dots,b_{R}\}$.
\end{defn}
In the next sections, we study the capacity of linear spline operators to span standard DNN layers. All the development of the spline operator as well as a detailed review of multivariate spline functions is contained in Appendix \ref{splineoperator}. Afterwards, the composition of the developed linear operators will lead to the explicit analytical input-output mapping of DNNs allowing to derive all the theoretical results in the remaining of the report.
In the following sections, we omit the cases of regularity constraints on the presented functional thus leading to the most general cases.

\subsection{Spline Operators[FINI]}

A natural extension of spline functions is the \textbf{{\em spline operator}} (SO) we denote $S:\mathbb{R}^D \rightarrow \mathbb{R}^K,K>1$. We present here a general definition and propose in the next section an intuitive way to construct spline operators via a collection of multivariate splines, the special case of current DNNs.
\begin{defn}
A spline operator is a mapping $S:\mathbb{R}^D \rightarrow \mathbb{R}^K$ defined by a collection of local mappings $\Phi^S=\{\phi^S_r:\mathbb{R}^D \rightarrow \mathbb{R}^K,r=1,\dots,R\}$ associated with a partition of $\mathbb{R}^D$ denoted as $\Omega^S=\{\omega^S_r,r=1,\dots,R\}$ s.t.
\begin{align*}
    S[\Phi^S,\Omega^S](x)=&\sum_{r=1}^R\phi^S_r(x)1_{\{x \in \omega^S_r\}}\\
    =&\phi^S[x](x),
\end{align*}
where we denoted the region specific mapping associated to the input $x$ by $\phi^S[x]$.

A special case occurs when the mappings $\phi^S_r$ are linear. We thus define in this case the \textbf{\textit{linear spline operator}} (LSO) which will play an important role for DNN analysis. In this case, $\phi^S_r(x)=A_rx+b_r$, with $A_r\in \mathbb{R}^{K \times D},b_r \in \mathbb{R}^D, \forall r$. As a result, a LSO can be rewritten as
\begin{align*}
    S[\textbf{A},\textbf{b},\Omega^S](x)=&\sum_{r=1}^R(A_rx+b_r)1_{\{x \in \omega^S_r\}}\\
    =&A[x]x+b[x]
\end{align*}
where we denoted the collection of intercept and biases as $\textbf{A}=\{A_r,r=1,\dots,R\}$, $\textbf{b}=\{b_r,r=1,\dots,R\}$ and finally the input specific activation as $A[x]$ and $b[x]$.
\end{defn}
Such operators can also be defined via a collection of multivariate polynomial (resp. linear) splines.
Given $K$  multivariate spline functions $s[\Phi_k,\Omega_k]:\mathbb{R}^D \rightarrow \mathbb{R},k=1,\dots,K$, their respective output is ''stacked'' to produce an output vector of dimension $K$. The internal parameters of each multivariate spline are $\Omega_k$, a partition of $\mathbb{R}^D$ with $Card(\Omega_k)=R_k$ and $\Phi_k=\{\phi_{1},\dots,\phi_{R_k}\}$. Stacking their respective output to form an output vector leads to the induced spline operator $S\left[\left(s[\Phi_k,\Omega_k]\right)_{k=1}^K\right]$.
\begin{defn}
The spline operator $S:\mathbb{R}^D \rightarrow \mathbb{R}^K$ defined with $K$ 
multivariate splines $\left(s[\Phi_k,\Omega_k]\right)_{k=1}^K$ with $s[\Phi_k,\Omega_k]:\mathbb{R}^D \rightarrow \mathbb{R}$ is defined as 
\begin{align}
    S\left[\left(s[\Phi_k,\Omega_k]\right)_{k=1}^K\right](x)=\left[ 
    \begin{matrix}
    s[\Phi_1,\Omega_1](x)\\
    \vdots \\
    s[\Phi_K,\Omega_K](x)
    \end{matrix}
    \right].
\end{align}
with $\Omega_k =\{\omega_{k,1},\dots,\omega_{k,R_k)}\}$, $\Phi_k=\{\phi_{k,r} \in \mathbb{R}^D,r=1,\dots, R_k\},k=1,\dots,K$. 
\end{defn}
The use of $K$ multivariate splines to construct a SO does not provide directly the explicit collection of mappings and regions $\Phi^S,\Omega^S$. Yet, it is clear that the SO is jointly governed by all the individual multivariate splines.
Let first present some intuitions on this fact.
The spline operator output is computed with each of the $K$ splines having ''activated'' a region specific functional depending on their own input space partitioning. In particular, each of the region $\omega^S_r$ of the input space leading to a specific joint configuration $\phi^S_r$ is the one of interest, leading to $\Omega^S$ and $\Phi^S$.
We can thus write explicitly the new regions of the spline operator based on the ensemble of partition of all the involved multivariate splines as
\begin{align}\label{omegaS}
    \Omega^S=\left(\bigcup_{(\omega_1,\dots,\omega_K) \in \Omega_1\times \dots \times \Omega_K}\left\{\bigcap_{k\in\{1,\dots,K\}}\omega_k\right\}\right)\setminus\{\emptyset\}.
\end{align}
We also denote the number of region associated to this SO as $R^S=Card(\Omega^S)$.
From this, the local mappings of the SO $\phi^S_r$ correspond to the joint mappings of the splines being activated on $\omega^S_r$ we denote 
\begin{align}
\phi^S[\omega^S_r](x)=\left[ 
    \begin{matrix}
     \phi_1[\omega^S_r](x)\\
    \vdots \\
     \phi_K[\omega^S_r](x)
    \end{matrix}
    \right],
\end{align}
with $\phi_k[\omega^S_r]=\phi_{k,q} \in \Phi_k$ s.t. $\omega^S_r \subset \omega_{k,q}$. In fact, for each region $\omega^S_r$ of the SO there is a unique region $\omega_{k,q_k}$ for each of the splines $k=1,\dots,K$, such that it is a subset as $\omega^S_r\subset \omega_{k,q_k}$ and it is disjoint to all others $\omega^S_r \cap \omega_{k,l}=\emptyset, \forall l \not = q_k,k=1,\dots,K$. In other word we have the following property:
\begin{align}
    \forall (r,k) \in \{1,\dots,R^S\}\times \{1,\dots,K\}, \exists ! q_k \in \{1,\dots,R_k\} \text{ s.t. } \omega^S_r \cap \omega_{k,l}=\left\{\begin{array}{l}
        \omega^S_r,\;\;l = q_k\\
        \emptyset,\;\;l\not = q_k
    \end{array} \right.,
\end{align}
as we remind $\omega^S_r \cap \omega_{k,l}=\omega^S_r\iff \omega^S_r \subset \omega_{k,q_k}$.

This leads to the following SO formulation
\begin{align}
    S\left[\left(s[\Phi_{k},\Omega_k]\right)_{k=1}^K\right](x)=&\sum_{r=1}^{R^S} \left[ 
    \begin{matrix}
    s[\Phi_{1},\Omega_1](x)\\
    \vdots \\
    s[\Phi_{K},\Omega_K](x)
    \end{matrix}
    \right]1_{\{x\in \omega^S_r\}}\nonumber\\
    =&\sum_{r=1}^{R^S} \left[ 
    \begin{matrix}
     \phi_1[\omega^S_r](x)\\
    \vdots \\
     \phi_K[\omega^S_r](x)
    \end{matrix}
    \right]1_{\{x\in \omega^S_r\}}\nonumber\\
    =&\sum_{r=1}^{R^S} \phi^S[\omega^S_r](x)1_{\{x\in\omega^S_r\}}\nonumber\\
    =&\phi^S[x](x),
\end{align}
\begin{figure}
    \centering
    \includegraphics[width=4in]{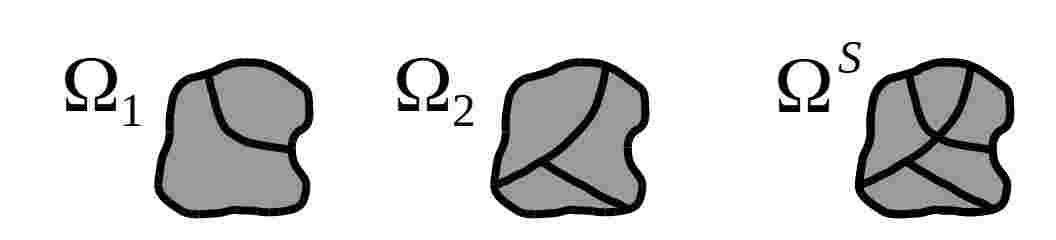}
    \caption{Illustrative examples of the new partition $\Omega^\mathcal{S}$ given two partitions $\Omega_1,\Omega_2$ with $I_1 =2,I_2=3$.}
    \label{fig2}
\end{figure}
We can now study the case of linear splines leading to LSOs.
If a SO is constructed via aggregation of linear multivariate splines,. The linear property allows notation simplifications. It is defined as 
\begin{align}
    S\left[\left(s[\textbf{a}_{k},\textbf{b}_{k},\Omega_k]\right)_{k=1}^K\right](x)=&\sum_{r=1}^{R^S} \left[ 
    \begin{matrix}
    s[\textbf{a}_{1},\textbf{b}_{1},\Omega_1](x)\\
    \vdots \\
    s[\textbf{a}_{K},\textbf{b}_{K},\Omega_K](x)
    \end{matrix}
    \right]1_{\{x\in \omega^S_r\}}\nonumber\\
    =&\sum_{r=1}^{R^S} \left[ 
    \begin{matrix}
    \langle a_{1}[\omega^S_r], x\rangle+b_{1}[\omega^S_r]\\
    \vdots \\
    \langle a_{K}[\omega^S_r], x\rangle+b_{K}[\omega^S_r]
    \end{matrix}
    \right]1_{\{x\in \omega^S_r\}}\nonumber\\
    =&\sum_{r=1}^{R^S} \left(\left[ 
    \begin{matrix}
    a_{1}[\omega^S_r]^T\\
    \vdots \\
    a_{K}[\omega^S_r]^T
    \end{matrix}
    \right]x+
    \left[ 
    \begin{matrix}
    b_{1}[\omega^S_r]\\
    \vdots \\
    b_{K}[\omega^S_r]
    \end{matrix}
    \right]
    \right)1_{\{x\in \omega^S_r\}}\nonumber\\
    =&\sum_{r=1}^{R^S} (A[\omega^S_r] x+b[\omega^S_r])1_{\{x\in\omega^S_r\}}\nonumber\\
    =&A[x]x+b[x],
\end{align}
with $\Omega_k =\{\omega_{k,1},\dots,\omega_{k,R_k)}\}$, $\textbf{a}_k=\{a_{k,r} \in \mathbb{R}^D,r=1,\dots, R_k\}$, $\textbf{b}_k=\{b_{k,r} \in \mathbb{R},r=1,\dots, R_k\}$. 

As it is clear, the collection of matrices and biases and the partitions completely define a LSO. Hence, we denote the set of all possible matrices and biases as    $\textbf{A}=\{A[\omega] ,\omega \in \Omega^S\}$, $\textbf{b}=\{b[\omega] ,\omega \in \Omega^S\}$. Any LSO is thus written as $S\left[\textbf{A},\textbf{b},\Omega^S\right]$.

\subsection{Linear Spline Operator: Generalized Neural Network Layers}\label{subsec:nn}
In this section we demonstrate how current DNNs are expressed as composition of LSOs. 
We first proceed to describe layer specific notations and analytical formula to finally perform composition of LSOs providing analytical DNN mappings in the next section.
\subsubsection{Nonlinearity layers}

We first analyze the elementwise nonlinearity layer. Our analysis deals with any given nonlinearity. If this nonlinearity is by definition a spline s.a. with ReLU, leaky-ReLU, absolute value, they fall directly into this analysis. If not, arbitrary functions such as tanh, sigmoid  are approximated via linear splines.
We remind that a nonlinearity layer $f^{(\ell)}_\sigma$ is defined by applying a nonlinearity $\sigma$ on each input dimension of its input $\bz^{(\ell-1)} \in \mathbb{R}^{D^{(\ell-1)}}$ and produces a new output vector $\bz^{(\ell)} \in \mathbb{R}^{D^{(\ell)}}$. While in general the used nonlinearity is the same applied on each dimension we present here a more general case where one has a specific $\sigma_d$ per dimension. In addition, we present the case where the nonlinearity might not act on only one input dimension but any part of it. We thus define by $\sigma^{(\ell)}[\textbf{a}_{d},\textbf{b}_{d},\Omega_d]:\mathbb{R}^{D^{(\ell-1)}}\rightarrow \mathbb{R}$ the nonlinearity acting on the $d^{th}$ input dimension, $\bz^{(\ell-1)}_d$.
We provide illustration of famous nonlinearities in Table \ref{table_relu} being cases where the output dimension at position $d$ only depends on the input dimension of $d$.

\begin{table}
\centering
\begin{tabular}{|c|c|c|}\hline
ReLU\cite{glorot2011deep} & LReLU\cite{xu2015empirical} & Abs.Value\\ \hline
$\begin{aligned}[t]
    &\Omega_d =\{\omega_{d,1},\omega_{d,2}\},\nonumber\\
    &\omega_{d,1} = \{x\in\mathbb{R}^d : x_d>0\},\nonumber \\
    &\omega_{d,2} = \{x\in\mathbb{R}^d : x_d\leq 0\}\nonumber\\
    &a_{d,1}=\textbf{e}_d,a_{d,2}=0,\nonumber \\
    &b_{k,1}=b_{k,2}=0,
\end{aligned}$     &  
$\begin{aligned}[t]
    &\Omega_d =\{\omega_{d,1},\omega_{d,2}\},\nonumber\\
    &\omega_{d,1} = \{x\in\mathbb{R}^d : x_d>0\},\nonumber \\
    &\omega_{d,2} = \{x\in\mathbb{R}^d : x_d\leq 0\}\nonumber\\
    &a_{d,1}=\textbf{e}_d,a_{d,2}=\eta \textbf{e}_d,\eta>0,\nonumber \\
    &b_{k,1}=b_{k,2}=0,
\end{aligned}$     &  
$\begin{aligned}[t]
    &\Omega_d =\{\omega_{d,1},\omega_{d,2}\},\nonumber\\
    &\omega_{d,1} = \{x\in\mathbb{R}^d : x_d>0\},\nonumber \\
    &\omega_{d,2} = \{x\in\mathbb{R}^d : x_d\leq 0\}\nonumber\\
    &a_{d,1}=\textbf{e}_d,a_{d,2}=-\textbf{e}_d,\nonumber \\
    &b_{k,1}=b_{k,2}=0,
\end{aligned}$    
\\ \hline
\end{tabular}
\caption{Example of multivariate splines associated with standard DNN nonlinearities, note that they are exact by definition as opposed to a linear spline approximating a sigmoid function for example, for which approximation can be made arbitrarily close.}\label{table_relu}
\end{table}
Given a collection of $\mathcal{D}^{(\ell)}$ such nonlinearities $(\sigma[(\textbf{a}_{d},\textbf{b}_{d},\Omega_d])_{d=1}^{D^{(\ell)}}$ of such linear splines, we define the spline operator which defined an actual nonlinear layer as $S^{(\ell)}_\sigma[\textbf{A}_\sigma,\textbf{b}_\sigma,\Omega^S_\sigma]$ with the induced matrix and vector parameters as defined in the previous section. Hence we have 
\begin{align}
    f^{(\ell)}_\sigma(\bz^{(\ell-1)})=&S^{(\ell)}_\sigma[\textbf{A}_\sigma,\textbf{b}_\sigma,\Omega^S_\sigma](\bz^{(\ell-1)})\nonumber \\
    =&A^{(\ell)}_\sigma[\bz^{(\ell-1)}]+b[\bz^{(\ell-1)}].
\end{align}
We also have for the provided examples
\begin{align*}
    \mathbf{b}_{relu}&=\mathbf{b}_{lrelu}=\mathbf{b}_{abs}=\{\textbf{0}\},\\
    \mathbf{A}_{relu}&=\{diag(v)|v \in \{0,1\}^K\},\\
    \mathbf{A}_{lrelu}&=\{diag(v)|v \in \{\eta,1\}^K\},\\
    \mathbf{A}_{abs}&=\{diag(v)|v \in \{-1,1\}^K\}.
\end{align*}
In order to better demonstrate the underlying spline operator mappings induced by typical DNN nonlinearities we provide a detailed example for the ReLU case below with $D^{(\ell-1)}=3,D^{(\ell)}=3$, we now omit the layer notation for the example. We have according to the presented definition that the slopes of the splines are
\begin{center}
\begin{align*}
    \textbf{A}_{relu}= &\left\{A_{(2,2,2)}=
    \left[\begin{matrix}0&0&0\\0&0&0\\0&0&0\end{matrix}\right],
    A_{(2,2,1)}=\left[\begin{matrix}0&0&0\\0&0&0\\0&0&1\end{matrix}\right],
    A_{(2,1,2)}=\left[\begin{matrix}0&0&0\\0&1&0\\0&0&0\end{matrix}\right],
    A_{(2,1,1)}=\left[\begin{matrix}0&0&0\\0&1&0\\0&0&1\end{matrix}\right],
    \right.
    \\&
    \left.
    A_{(1,2,2)}=\left[\begin{matrix}1&0&0\\0&0&0\\0&0&0\end{matrix}\right], A_{(1,2,1)}=\left[\begin{matrix}1&0&0\\0&0&0\\0&0&1\end{matrix}\right],
    A_{(1,1,2)}=\left[\begin{matrix}1&0&0\\0&1&0\\0&0&0\end{matrix}\right],
    A_{(1,1,1)}=\left[\begin{matrix}1&0&0\\0&1&0\\0&0&1\end{matrix}\right]
    \right\},
    \end{align*}
    \begin{align*}
    \textbf{b}_{relu}=&\{\textbf{0}\},
    \end{align*}
    \begin{align*}
    \Omega^S_{relu}=&\{\omega^S_{(2,2,2)},\omega^S_{(2,2,1)},\omega^S_{(2,1,2)},\omega^S_{(2,1,1)},\omega^S_{(1,2,2)},\omega^S_{(1,2,1)},\omega^S_{(1,1,2)},\omega^S_{(1,1,1)}\},\\
    &\omega^S_{(2,2,2)} = \{x \in \mathbb{R}^3:x_1< 0,x_2<0,x_3<0\},\\
    &\omega^S_{(2,2,1)} = \{x \in \mathbb{R}^3:x_1< 0,x_2<0,x_3\geq 0\},\\
    &\omega^S_{(2,1,2)} = \{x \in \mathbb{R}^3:x_1< 0,x_2\geq 0,x_3<0\},\\
    &\omega^S_{(2,1,1)} = \{x \in \mathbb{R}^3:x_1< 0,x_2\geq 0,x_3\geq 0\},\\
    &\omega^S_{(1,2,2)} = \{x \in \mathbb{R}^3:x_1\geq 0,x_2<0,x_3<0\},\\
    &\omega^S_{(1,2,1)} = \{x \in \mathbb{R}^3:x_1\geq 0,x_2<0,x_3\geq 0\},\\
    &\omega^S_{(1,1,2)} = \{x \in \mathbb{R}^3:x_1\geq 0,x_2\geq 0,x_3<0\},\\
    &\omega^S_{(1,1,1)} = \{x \in \mathbb{R}^3:x_1\geq 0,x_2\geq 0,x_3\geq 0\},
\end{align*}
\end{center}
and we thus have when given some inputs that
\begin{align*}
    S_{relu}[\textbf{A}_{relu},\textbf{b}_{relu},\Omega^S_{relu}]([1,2,3]^T)&=A_{(1,1,1)}[1,2,3]^T+b_{(1,1,1)}\\
                                            &=[1,2,3]^T,\\
    S_{relu}[\textbf{A}_{relu},\textbf{b}_{relu},\Omega^S_{relu}]([1,-2,3]^T)&=A_{(1,2,1)}[1,-2,3]^T+b_{(1,2,1)}\\
                                            &=[1,0,3]^T,\\
    S_{relu}[\textbf{A}_{relu},\textbf{b}_{relu},\Omega^S_{relu}]([-1,-2,-3]^T)&=A_{(2,2,2)}[-1,-2,-3]^T+b_{(2,2,2)}\\
                                            &=[0,0,0]^T.\\
\end{align*}

\subsubsection{Sub-Sampling layers}

We now study the case of sub-sampling layers. We first remind briefly that a pooling layer $f^{(\ell)}_\rho$ is defined by a pooling policy $\rho$ and a collection of regions $R_d,d=1,\dots,d^{(\ell)}$. Each of these regions contain indices on which $\rho$ will be applied to form the output vector $\bz^{(\ell)}\in \mathbb{R}^{D^{(\ell)}}$. As for the nonlinear layer, it is common to use the same pooling policy across all regions, yet we now formula the spline functional for the more general case of a pooling per region $\rho_d$ by $\rho^{(\ell)}[\textbf{a}_{d},\textbf{b}_{d},\Omega_d]:\mathbb{R}^{D^{(\ell-1)}}\rightarrow \mathbb{R}$.
We provide illustration of standard cases in Table \ref{table_pooling}.
\begin{table}
\centering
\begin{tabular}{|c|c|}\hline
Max-Pooling & Mean-Pooling\\ \hline
$\begin{aligned}[t]
    &\Omega_d =\{\omega_{d,1},\dots,\omega_{d,Card(R_d)}\},\\
    &\omega_{d,r} = \{x\in\mathbb{R}^d | r=\argmax x_{R_d}\}\nonumber\\
    &[a_{d,r}]_l=\textbf{e}_{R_d(r)},b_{d,r}=0,
\end{aligned}$     &  
$\begin{aligned}[t]
    &\Omega_k = \{\mathbb{R}^d\},\nonumber\\
    &a_{d,1}=\frac{1}{Card(R_d)}(\sum_{i \in R_d} \textbf{e}_i),b_{k,1}=0,\forall k
\end{aligned}$    
\\ \hline
\end{tabular}
\caption{Example of multivariate splines associated with standard DNN pooling.}\label{table_pooling}
\end{table}
Similarly to the nonlinearity case we can now define the spline operator.
Again, given a collection $\Big(\rho[\textbf{a}_d,\textbf{b}_d,\Omega_d]\Big)_{d=1}^{D^{(\ell)}}$ of such affine splines we can create the spline operator denoted by $S_\rho[\textbf{A}_\rho,\textbf{b}_\rho,\Omega^S_\rho]$. For the max-pooling policy we have
\begin{align}
&\textbf{A}_{max} = \left\{\left[ 
    \begin{matrix}
    e_1^T\\
    \vdots \\
    e_{D^{(\ell)}}^T
    \end{matrix}
    \right],e_d \in \{ \mathbf{e}_{d},d \in R_d\}, d=1,\dots, D^{(\ell)}\right\},\textbf{b}_{max}=\{\textbf{0}\},
\end{align}

We now present an illustrative example of the max-pooling LSO with $D^{(\ell-1)}=4,D^{(\ell)}=2$ and $R_1=\{1,2\},R_2=\{3,4\}$ which corresponds to pooling over non-overlapping regions of size $2$. We thus have
\begin{center}
\begin{align*}
    \textbf{A}_{max}=&\left\{A_{(1,1)}=
    \left[\begin{matrix}1&0&0&0\\0&0&1&0\end{matrix}\right],
    A_{(1,2)}=\left[\begin{matrix}1&0&0&0\\0&0&0&1\end{matrix}\right],
    \right. \\ 
    &\left. A_{(2,1)}=\left[\begin{matrix}0&1&0&0\\0&0&1&0\end{matrix}\right],
    A_{(2,2)}=\left[\begin{matrix}0&1&0&0\\0&0&0&1\end{matrix}\right]\right\},
    \end{align*}
    \begin{align*}
    \textbf{b}_{max}=&\{\textbf{0}\},
    \end{align*}
    \begin{align*}
    \bm{\Omega}_{max}=&\{\omega^S_{(1,1)},\omega^S_{(1,2)},\omega^S_{(2,1)},\omega^S_{(2,2)}\},\\
    &\omega^S_{(1,1)} = \{x \in \mathbb{R}^4:1=\argmax_{d\in\{1,2\}}x_d,3=\argmax_{d\in\{3,4\}}x_d\},\\
    &\omega^S_{(1,2)} = \{x \in \mathbb{R}^4:1=\argmax_{d\in\{1,2\}}x_d,4=\argmax_{d\in\{3,4\}}x_d\},\\
    &\omega^S_{(2,1)} = \{x \in \mathbb{R}^4:2=\argmax_{d\in\{1,2\}}x_d,3=\argmax_{d\in\{3,4\}}x_d\},\\
    &\omega^S_{(2,2)} = \{x \in \mathbb{R}^4:2=\argmax_{d\in\{1,2\}}x_d,4=\argmax_{d\in\{3,4\}}x_d\},
    \end{align*}
    \end{center}
and we thus have when given some inputs that
\begin{align*}
    S_{max}[\textbf{A}_{max},\textbf{b}_{max},\bm{\Omega}_{max}]([1,2,3,4]^T)&=A_{(2,2)}[1,2,3,4]^T+b_{(2,2)}\\
                                            &=[2,4]^T,\\
    S_{max}[\textbf{A}_{max},\textbf{b}_{max},\bm{\Omega}_{max}]([1,-2,3,1]^T)&=A_{(1,1)}[1,-2,3,1]^T+b_{(1,1)}\\
                                            &=[1,3]^T,\\
    S_{max}[\textbf{A}_{max},\textbf{b}_{max},\bm{\Omega}_{max}]([-1,-2,-3,0]^T]&=A_{(1,2)}[-1,-2,-3,0]^T+b_{(1,2)}\\
                                            &=[-1,0]^T.\\
\end{align*}

\subsubsection{Linear layers:FC and convolutional}
Finally, in order to provide a complete spline interpretation of DNN layers we present the case of linear layers as convolutional and FC layers. By definition of being linear mappings, they are equivalent to a spline operator with one region corresponding to the input space. We thus define this operator as
\begin{align}
    S^{(\ell)}_{W}[\{W^{(\ell)}\},\{b^{(\ell)}\},\{\mathbb{R}^{D^{(\ell-1)}}\}](\bz^{(\ell-1)})&=W^{(\ell)}\bz^{(\ell-1)}+b^{(\ell)},\;\;\text{FC-layer}\\
    S^{(\ell)}_{\bC}[\{\bC^{(\ell)}\},\{b^{(\ell)}\},\{\mathbb{R}^{D^{(\ell-1)}}\}](\bz^{(\ell-1)})&=\bC^{(\ell)}\bz^{(\ell-1)}+b^{(\ell)},\;\;\text{convolutional layer}
    \end{align}
where we shall omit the trivial parameters and denote $S^{(\ell)}_{W}:=S^{(\ell)}_{W}[\{W^{(\ell)}\},\{b^{(\ell)}\},\{\mathbb{R}^{D^{(\ell-1)}}\}]$ and $S^{(\ell)}_{\bC}:=S^{(\ell)}_{\bC}[\{\bC^{(\ell)}\},\{b^{(\ell)}\},\{\mathbb{R}^{D^{(\ell-1)}}\}]$ .
We derived all the notations and gave examples on how to define standard DNNs layers via linear spline operators. We can now move to the composition of such operators defining the complete DNN mappings.

\subsection{Deriving Analytical DNNs Mappings to Explicit their Faculty to Perform Template Matching}\label{subsub:networks}
To perform perceptual tasks such as object recognition, a standard technique is template matching. It aims as detecting the presence in the input of a class specific template even if the template in the input has suffered some perturbation.
Template matching is well studied when the template perturbation belongs to the standard groups of natural deformations s.a. translation, rotation for example and this process is usually referred as elastic matching. There are also been extension to perform a hierarchical elastic matching in \cite{zhang1997face,bajcsy1989multiresolution,burr1981elastic} by marginalizing out layer after layer all the possible local perturbation.
Many extensions have also been studied to model more complex diffeomorphisms as in \cite{korman2013fast,kim2007grayscale}. A detailed review of elastic matching is proposed in \cite{uchida2005survey}. This task can also be formulated as a problem of best basis selection where the optimal atom is the correct template with the input adapted perturbation. Concept of input dependent basis has been well studied for example in \cite{coifman1992entropy,tropp2004greed,mallat2008wavelet,berger1994removing}. 
Yet, the need for exact mathematical modeling of the template transformations limit the ability to produce algorithms flexible enough to learn classes of diffeomorphisms in a complete data driven, parametric learning approach. As we will see, this is performed by state-of-the-art DNNs.

\subsubsection{Composition of Splines for Explicit DNN Template Matching}
As demonstrated in the previous sections, neural network layers are special cases of LSOs.
From the derived notation and LSOs of the previous section,  we can now proceed to rewrite the complete DNN mapping as composition of such operators. Firstly, we define $S^{(\ell)}[\textbf{A}^{(\ell)},\textbf{b}^{(\ell)},\Omega^{S(\ell)}]:=S^{(\ell)}_{\theta^{(\ell)}}$. We thus have for any DNN 
\begin{equation}
f_\Theta(x)\approx (S^{(L)}_{\theta^{(L)}} \circ \dots \circ S^{(1)}_{\theta^{(1)}})(x),\;\;\;\Theta=\{\theta^{(1)},\dots,\theta^{(L)}\},
\end{equation}
where the approximation becomes an equality if the used layers are splines s.a. with ReLU, max-pooling. If not,  arbitrary close approximation schemes can be found. 
This provides a very intuitive result from this composition of linear mappings.
\begin{theorem}\label{th1}
Any deep network $f_\Theta$ made of LSOs s.a. max-pooling, ReLU, leaky ReLU,\dots is itself a LSO of the form 
\begin{equation}
    f_\Theta(x)= A[x]x+b[x],\forall x.
\end{equation}
For the case where the layers are not natural LSOs s.a. with tanh, sigmoid nonlinearities, then, it can always be approximated arbitrarily closely by a affine spline operator and thus
\begin{equation}
    f_\Theta(x)\approx A[x]x+b[x],\forall x.
\end{equation}
\end{theorem}
In fact, it has been shown that linear splines can approximate any functions arbitrarily closely \cite{nishikawa1998accurate}. In addition, using linear spline approximations is computationally efficient. For example, using ultra fast sigmoid (a linear spline version) instead of the standard sigmoid results in almost $2\times$ speedup  for 100M float64 elements on a Core2 Duo @ 3.16 GHz \cite{bastien2012theano}.

As opposed to previous work studying DNNs as composition of linear mappings in the spcial case of ReLU coupled with mean o max pooling \cite{rister2017piecewise}, we extend the results to arbitrary DNNs by allowing linear spline approximation of non spline functionals as well as general piecewise linear splines.

We thus propose  to bridge the concept of input adaptive representations with DNNs. To do so we leverage the fact shown above that any DNN can be rewritten as a linear mapping as $f_\Theta(x)=A[x]x+b[x]$, with input dependent intercept $A[x]$ and biases $b[x]$. We thus propose the following definition.
\begin{defn}
As any DNN can be rewritten $f_\Theta(x)=A[x]x+b[x]$, we denote $A[x]$ as the template of the DNN mapping, and specifically $A[x]_{c,.}$ the template associated with class $c$. By nature of the underlying LSOs, the input adaptive template matching is induced by the per region coefficients making DNNs effective hierarchical template matching algorithms.
\end{defn}

We now proceed to write the analytical mapping defined by this composition of LSOs. For clarity, we only set one MLP layer at the end of the DNNs, extensions to any number of MLP layers is straightforward by adding a simple product term over those layers . Finally, while providing formula for $3$ standard topologies, we also aim at presenting the methodology in order for one to generalize the presented results to any used topology, as only replacement of some operators will lead to any possible DNN topology.
With this layer notation we can naturally derive the formula for the output $\bz^{(L)}(x)$ of any DNN with $L$ being the number of layers in the mapping. In fact, as $S^{(\ell)}_{\theta^{(\ell)}}$ is a LSO at each layer $\ell$ we have the output expression given by
\begin{align}\label{all_eq}
    \bz^{(L)}(x)& =  W^{(L)}\left(\underbrace{\left(\prod_{\ell=L-1}^1A^{(\ell)}_{\theta^{(\ell)}} \right)x+\sum_{\ell=1}^{L-1} \left(\prod_{j=L-1}^{\ell+1}A_{\theta^{(\ell)}}^{(j)}\right)b^{(\ell)}_{\theta^{(\ell)}}}_{\text{Convolutional Layers}}\right)+b^{(L)}\\
    &= \underbrace{W^{(L)}\left(\prod_{\ell=L-1}^1A^{(\ell)}_{\theta^{(\ell)}} \right)}_{\text{Template Matching}}x+\underbrace{W^{(L)}\sum_{\ell=1}^{L-1} \left(\prod_{j=L-1}^{\ell+1}A_{\theta^{(\ell)}}^{(j)}\right)b^{(\ell)}_{\theta^{(\ell)}}+b^{(L)}}_{\text{Bias}}\\
    &=A^{(L\rightarrow 1)}[x]x+b^{(L\rightarrow 1)}[x].
\end{align}
We denoted by $A^{(L\rightarrow 1)}[x]$ and $b^{(L\rightarrow 1)}[x]$ the induced templates after unrolling over all the layers. More generally if this is done till layer $\ell$ it is denoted as $A^{(\ell\rightarrow 1)}[x]$ and $b^{(\ell\rightarrow 1)}[x]$.
Given this interpretation we now proceed to derive explicitly what are the templates and biases for some standard topologies below as well as emphasizing the methodology for one to adapt the result to specific cases.
\subsubsection{Deep Convolutional Networks}

We first study the case of standard DCNs as described in \ref{DCN}. 
A DCN is composed of $B$ blocks of $3$ layers defined as
\begin{align}
    \bz^{(\ell)}(x)=&S^{(\ell)}_{\theta^{(\ell)}}(\bz^{(\ell-1)}(x))\nonumber\\
    =&(S^{(\ell)}_\rho \circ S^{(\ell)}_\sigma \circ S^{(\ell)}_{\bC})(\bz^{(\ell-1)}(x))\nonumber\\
    =&A^{(\ell)}_\rho[\bz^{(\ell-1)}(x)]\left(A^{(\ell)}_\sigma[\bz^{(\ell)}(x)]\left(\bC^{(\ell)}\bz^{(\ell-1)}(x)+b^{(\ell)}\right)+b^{(\ell)}_\sigma[\bz^{(\ell)(x)}]\right)+b^{(\ell)}_\rho[\bz^{(\ell)}(x)]\nonumber\\
    =&\underbrace{A^{(\ell)}_\rho A^{(\ell)}_\sigma\bC^{(\ell)}}_{\text{Template Matching}}\bz^{(\ell-1)}(x)+\underbrace{A^{(\ell)}_\rho A^{(\ell)}_\sigma b^{(\ell)}+A^{(\ell)}_\rho b^{(\ell)}_\sigma +b^{(\ell)}_\rho}_{\text{Bias}}\nonumber\\
    :=&A_{\theta^{(\ell)}}^{(\ell)}\bz^{(\ell-1)}(x)+b_{\theta^{(\ell)}}^{(\ell)},
\end{align}
hence for a convolutional block, we have
\begin{align}
    A_{\theta^{(\ell)}}^{(\ell)}&=A^{(\ell)}_\rho A^{(\ell)}_\sigma\bC^{(\ell)},\\
    b^{(\ell)}_{\theta^{(\ell)}}&=A^{(\ell)}_\rho A^{(\ell)}_\sigma b^{(\ell)}+A^{(\ell)}_\rho b^{(\ell)}_\sigma +b^{(\ell)}_\rho.
\end{align}
 The topology implies the input conditioning of the spline tp depend on the previous layer output hence $A^{(\ell)}[\bz^{(\ell-1)}]:=A^{(\ell)}$.
 Using Eq. \ref{all_eq}, we can write the overall DCN mapping as
\begin{align*}
        \bz_{CNN}^{(L)}(x) =& \underbrace{W^{(L)} \prod_{\ell=L-1}^1A_{\rho}^{(\ell)}A^{(\ell)}_{\sigma} \bC^{(\ell)}}_{\text{Template Matching}}x\\
        &+\underbrace{W^{(L)}\sum_{\ell=1}^{L-1}\left(\prod_{j=L-1}^{\ell+1}A_{\rho}^{(j)}A^{(j)}_{\sigma} \bC^{(j)}\right)\left(A_{\rho}^{(\ell)}A^{(\ell)}_{\sigma} b^{(\ell)}+A^{(\ell)}_\rho b_\sigma^{(\ell)}+b_\rho^{(\ell)}\right)+b^{(L)}}_{\text{Bias}}.
\end{align*}
For cases of unbiased nonlinearities and pooling s.a. ReLU and max-pooling, this formula simplifies to
\begin{empheq}[box=\fbox]{align}
    \bz_{CNN}^{(L)}(x) =& \underbrace{W^{(L)} \prod_{\ell=L-1}^1A_{\rho}^{(\ell)}A^{(\ell)}_{\sigma} \bC^{(\ell)}}_{\text{Template Matching}}x+\underbrace{W^{(L)}\sum_{\ell=1}^{L-1}\left(\prod_{j=L-1}^{\ell+1}A_{\rho}^{(j)}A^{(j)}_{\sigma} \bC^{(j)}\right)\left(A_{\rho}^{(\ell)}A^{(\ell)}_{\sigma} b^{(\ell)}\right)+b^{(L)}}_{\text{Bias}}.
\end{empheq}
Hence the per layer templates and biased are defined as 
\begin{align}
    A^{(L\rightarrow 1)}[x]&= \prod_{k=\ell}^1A_{\rho}^{(k)}A^{(k)}_{\sigma} \bC^{(k)},\ell=1,\dots,L-1\\
    b^{(L\rightarrow 1)}[x]&=\sum_{k=1}^{\ell}\left(\prod_{j=\ell}^{k+1}A_{\rho}^{(j)}A^{(j)}_{\sigma} \bC^{(j)}\right)\left(A_{\rho}^{(k)}A^{(k)}_{\sigma} b^{(k)}\right),\ell=1,\dots,L-1&=
\end{align}
\subsubsection{Deep Residual Networks}
We now present the case of Residual Networks. A generic residual layer\cite{he2016deep} is defined as
\begin{align}
    z^{(\ell)}(x)=&S^{(\ell)}_{\theta^{(\ell)}}(\bz^{(\ell-1)}(x))\nonumber \\
    =&A_{\rho}^{(\ell)}\left(A_{\sigma,in}^{(\ell)}\left(\bC_{in}^{(\ell)}z^{(\ell-1)}(x)+b_{in}^{(\ell)}\right)+\bC^{(\ell)}_{out}\bz^{(\ell-1)}(x)+b_{out}^{(\ell)}+b_{\sigma,in}^{(\ell)}\right)+b^{(\ell)}_\rho\nonumber\\
    =&\underbrace{A_{\rho}^{(\ell)}\left(A_{\sigma,in}^{(\ell)}\bC_{in}^{(\ell)}+\bC^{(\ell)}_{out}\right)}_{\text{Template Matching}}\bz^{(\ell-1)}(x)+\underbrace{A_{\rho,in}^{(\ell)}A_{\sigma,in}^{(\ell)}b_{in}^{(\ell)}+A_{\rho,in}^{(\ell)}b_{out}^{(\ell)}+A_{\rho,in}^{(\ell)}b_{\sigma,in}^{(\ell)}+b^{(\ell)}_\rho}_{\text{Bias}}\nonumber\\
    :=&A_{\theta^{(\ell)}}^{(\ell)}\bz^{(\ell-1)}(x)+b_{\theta^{(\ell)}}^{(\ell)},
\end{align}
with the nonlinearity conditioned on $\bC_{in}^{(\ell)}z^{(\ell-1)}(x)+b_{in}^{(\ell)}$ and the pooling on $A_{\sigma,in}^{(\ell)}\left(\bC_{in}^{(\ell)}z^{(\ell-1)}(x)+b_{in}^{(\ell)}\right)+\bC^{(\ell)}_{out}\bz^{(\ell-1)}(x)+b_{out}^{(\ell)}+b_{\sigma,in}^{(\ell)}$.
Hence for a residual block, we have
\begin{align}
    A_{\theta^{(\ell)}}^{(\ell)}&=A_{\rho}^{(\ell)}A_{\sigma,in}^{(\ell)}\bC_{in}^{(\ell)}+\bC^{(\ell)}_{out},\\
    b^{(\ell)}_{\theta^{(\ell)}}&=A_{\rho}^{(\ell)}A_{\sigma,in}^{(\ell)}b_{in}^{(\ell)}+A_{\rho}^{(\ell)}b_{out}^{(\ell)}+A_{\rho}^{(\ell)}b_{\sigma,in}^{(\ell)}+b_\rho^{(\ell)}.
\end{align}
Using Eq. \ref{all_eq}, we can write the overall Resnet mapping as
\begin{align*}
\bz_{RES}^{(L)}(x) =& \underbrace{W^{(L)}\left[ \prod_{\ell=L-1}^1A_{\rho}^{(\ell)}\left(A_{\sigma,in}^{(\ell)}\bC_{in}^{(\ell)}+ \bC^{(\ell)}_{out} \right) \right]}_{\text{Template Matching}}x\\
&+\underbrace{\sum_{\ell=L-1}^1\left(\prod_{j=L-1}^{\ell+1}A_{\rho}^{(j)}( A_{\sigma,in}^{(j)}\bC_{in}^{(j)}+ \bC^{(j)}_{out})\right)\left( A_{\rho}^{(\ell)}A_{\sigma,in}^{(\ell)}b_{in}^{(\ell)}+A_{\rho}^{(\ell)}b_{out}^{(\ell)} +A_{\rho}^{(\ell)}b^{(\ell)}_{\sigma,in}+b^{(\ell)}_\rho\right)+b^{(L)}}_{\text{Bias}}.
\end{align*}
It is common in Resnet to not have a pooling operation $A_{\rho}^{(\ell)}=I,b_{\rho}^{(\ell)}=0$ but instead to apply a linear sub-sampling via the stride parameter of the convolution. Standard convolutions have a stride of $(1,1)$ corresponding to no sub-sampling. Stride of $(k,k)$ naturally correspond to $(k,k)$ linear sub-sampling. Also, if no bias nonlinearity is used, then the Resnet recursion simplifies to
\begin{equation}
    \boxed{\bz_{RES}^{(L)}(x) = \underbrace{W^{(L)}\left[ \prod_{\ell=L-1}^1\left(A_{\sigma,in}^{(\ell)}\bC_{in}^{(\ell)}+ \bC^{(\ell)}_{out} \right) \right]}_{\text{Template Matching}}x+\underbrace{\sum_{\ell=L-1}^1\left(\prod_{i=L-1}^{\ell+1}( A_{\sigma,in}^{(\ell)}\bC_{in}^{(\ell)}+ \bC^{(\ell)}_{out})\right)\left( A_{\sigma,in}^{(\ell)}b_{in}^{(\ell)}+b_{out}^{(\ell)} \right)+b^{(L)}}_{\text{Bias}}}.
\end{equation}
Interestingly one can rewrite the Resnet formulation in the case $\bC_{out}^{(\ell)}=I$ as 
\begin{align*}
    \bz_{RES}^{(L)}(x) =& x +\sum_{\ell=1}^{L-1}\left( \prod_{k=\ell}^1 A_{\sigma,in}^{(k)}W_{in}^{(k)} \right) x+\sum_{\ell=L-1}^1\left(\prod_{i=L-1}^{\ell+1}( \bC^{(i)}_{out}+A_{\sigma,in}^{(i)}W_{in}^{(i)} )\right)\left( b_{out}^{(\ell)}+A_{\sigma,in}^{(\ell)}b_{in}^{(\ell)} \right)\\
    =&\sum_{\ell=L-1}^1\left(\prod_{i=L-1}^{\ell+1}( \bC^{(i)}_{out}+A_{\sigma,in}^{(i)}W_{in}^{(i)} )\right)\left( b_{out}^{(\ell)}+A_{\sigma,in}^{(\ell)}b_{in}^{(\ell)} \right)\\
    &
    \left.\begin{array}{l}
    +x\\
    +A_{\sigma,in}^{(1)}W_{in}^{(1)}x\\
    +A_{\sigma,in}^{(2)}W_{in}^{(2)}A_{\sigma,in}^{(1)}W_{in}^{(1)}x\\
    +A_{\sigma,in}^{(3)}W_{in}^{(3)}A_{\sigma,in}^{(2)}W_{in}^{(2)}A_{\sigma,in}^{(1)}W_{in}^{(1)}x\\
    \vdots\\
    +\left(\prod_{k=L-1}^1 A_{\sigma,in}^{(k)}W_{in}^{(k)} \right) x
    \end{array}\right\}\text{Ensemble of Models}
\end{align*}
In fact, is has been shown in \cite{veit2016residual} that deep residual networks behave like ensemble of relatively shallow models.

\subsubsection{Deep Recurrent Networks}

Similarly, we can derive the one step of a standard fully recurrent neural network \cite{graves2013generating} as
\begin{align*}
    \bz_{RNN}^{(1,t)} &=A^{(1,t)}_{\sigma}[W^{(1)}_{in}x^t+W^{(1)}_{rec}\bz^{(1,t-1)}_{RNN}+b^{(1)}](W^{(1)}_{in}x^t+W^{(1)}_{rec}\bz^{(1,t-1)}_{RNN}+b^{(1)})+b_{\sigma}^{(1,t)},\\
    \bz_{RNN}^{(\ell,t)} &=A^{(\ell,t)}_{\sigma}[W^{(\ell)}_{in}x^t+W^{(\ell)}_{rec}\bz^{(\ell,t-1)}_{RNN}+W^{(\ell)}_{up}\bz^{(\ell-1,t)}_{RNN}+b^{(\ell)}](W^{(\ell)}_{in}x^t+W^{(\ell)}_{rec}\bz^{(\ell,t-1)}_{RNN}\\
    &+W^{(\ell)}_{up}\bz^{(\ell-1,t)}_{RNN}+b^{(\ell)})+b_{\sigma}^{(\ell,t)},\text{ $\ell>1$}.
\end{align*}
By the double recursion of the formula (in time and in depth) we first proceed by writing the time unrolled RNN mapping as
\begin{align*}
    \bz^{(1,T)}_{RNN}(x) &= \sum_{t=T}^1\Big(\prod_{k=T}^{t+1}A^{(1,k)}_{\sigma}W^{(1)}_{rec}\Big)A^{(1,t)}_{\sigma}(W^{(1)}_{in}x^t+ b_{\sigma}^{(1,t)}+A^{(1,t)}_{\sigma} b^{(1)}) \\
    &= \sum_{t=T}^1\Big(\prod_{k=T}^{t+1}A^{(1,k)}_{\sigma}W^{(1)}_{rec}\Big)A^{(1,t)}_{\sigma}W^{(1)}_{in}x_t+\sum_{t=T}^1\Big( \prod_{k=T}^{t+1}A^{(1,k)}_{\sigma} W_{rec}^{(1)} \Big)\Big[ b_{\sigma}^{(1,t)}+A^{(1,t)}_{\sigma} b^{(1)} \Big]\\
    \bz^{(\ell,T)}_{RNN}(x) 
    &=  \sum_{t=T}^1\Big(\prod_{k=T}^{t+1}A^{(\ell,k)}_{\sigma}W^{(\ell)}_{rec}\Big)A^{(\ell,t)}_{\sigma}W^{(\ell)}_{in}x^t\nonumber\\
    &+\sum_{t=T}^1\Big( \prod_{k=T}^{t+1}A^{(\ell,k)}_{\sigma} W_{rec}^{(\ell)} \Big)\Big[ b_{\sigma}^{(\ell,t)}+A^{(\ell,t)}_{\sigma} b^{(\ell)}+ A^{(\ell,t)}_{\sigma}W^{(\ell)}_{up}\bz^{(\ell-1,t)}_{RNN}(x)\Big],\text{ $\ell >1$.}
\end{align*}
The presented formula unrolled in time are still recursive in depth. While the exact unrolled version would be cumbersome for any layer $\ell$ we propose a simple way to find the analytical formula based on the possible paths an input can take till the final time representation of layer $\ell$. To do so, one can look in Fig. \ref{fig:blah1},\ref{fig:blah2}.
\begin{figure}
\begin{minipage}[b]{0.45\linewidth}
  \centering
  \includegraphics[width=3.in]{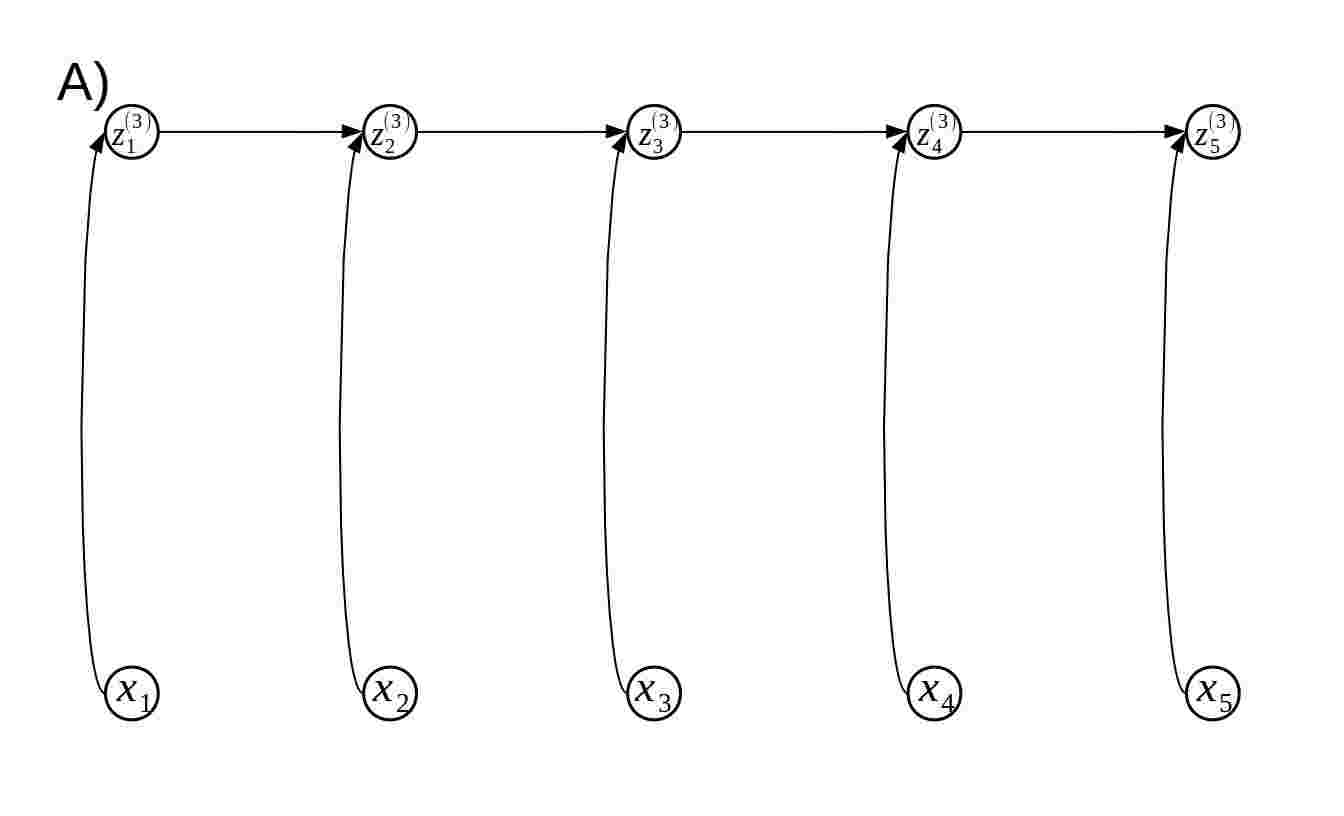}
  \caption{At any given layer, there always exists a direct input to representation path and recursion.}
  \label{fig:blah1}
\end{minipage}
\hfill
\begin{minipage}[b]{0.45\linewidth}
  \centering
  \includegraphics[width=3.in]{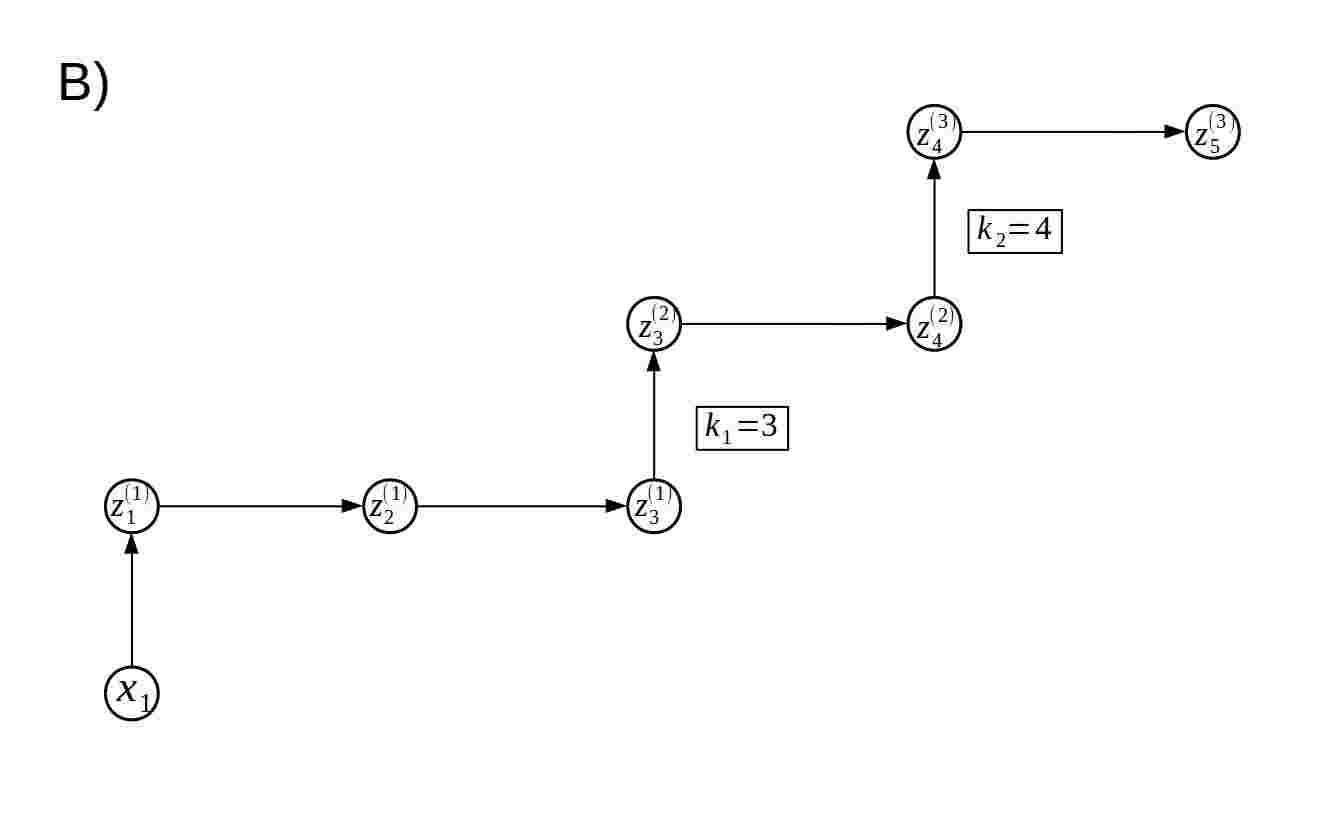}
  \caption{In addition, all the possible input path going through the hidden layers. Combinatorial number of paths yet fully determined by the succession of forward in time or upward in layer successions.}
  \label{fig:blah2}
\end{minipage}
\end{figure}
Hence we can thus decompose all those paths by blocks of forward interleave with upward paths.
With this, we can see that the possible paths are all the path from the input to the final nodes, they can not got back in time nor down in layers. Hence they are all the possible combinations for forward in time or upward in depth. We can thus find the exact output formula below for RNN as
\begin{empheq}[box=\fbox]{align}
    \bz^{(2,T)}_{RNN}(x) 
    =&  \sum_{t=1}^{T}\sum_{k_1=t}^{T}\Big(\prod_{q=T}^{k_1+1}A^{(2,q)}_{\sigma}W^{(2)}_{rec}\Big)A^{(2,k_1)}_{\sigma}W^{(2)}_{up}\Big( \prod_{q=t+1}^{k_1-1}A^{(1,q)}_{\sigma} W_{rec}^{(1)} \Big)A^{(1,t)}_{\sigma}W^{(1)}_{in}x^t\\
    \bz^{(3,T)}_{RNN}(x) 
    =&  \sum_{t=1}^{T}\sum_{k_1=t}^{T}\sum_{k_2\geq k_1}^{T}\Big(\prod_{q=T}^{k_2+1}A^{(3,q)}_{\sigma}W^{(3)}_{rec}\Big)A^{(3,k_2)}_{\sigma}W^{(3)}_{up}\Big(\prod_{q=k_1}^{k_2-1}A^{(2,q)}_{\sigma}W^{(2)}_{rec}\Big)A^{(2,k_1)}_{\sigma}W^{(2)}_{up}\nonumber\\
    &\Big( \prod_{q=t+1}^{k_1-1}A^{(1,q)}_{\sigma} W_{rec}^{(1)} \Big)A^{(1,t)}_{\sigma}W^{(1)}_{in}x^t \\
    \bz^{(4,T)}_{RNN}(x) 
    =&  \sum_{t=1}^{T}\sum_{k_1=t}^{T}\sum_{k_2\geq k_1}^{T}\sum_{k_3\geq k_2}^{T}\Big(\prod_{q=T}^{k_3+1}A^{(4,q)}_{\sigma}W^{(4)}_{rec}\Big)A^{(4,k_3)}_{\sigma}W^{(4)}_{up}\Big(\prod_{q=T}^{k_2+1}A^{(3,q)}_{\sigma}W^{(3)}_{rec}\Big)A^{(3,k_2)}_{\sigma}W^{(3)}_{up}\nonumber \\
    &\Big(\prod_{q=k_1}^{k_2-1}A^{(2,q)}_{\sigma}W^{(2)}_{rec}\Big)A^{(2,k_1)}_{\sigma}W^{(2)}_{up}\Big( \prod_{q=t+1}^{k_1-1}A^{(1,q)}_{\sigma} W_{rec}^{(1)} \Big)A^{(1,t)}_{\sigma}W^{(1)}_{in}x^t\\
    \dots&\nonumber
\end{empheq}

\subsection{Template Matching with DNN: How and Why}\label{sub:ada}
We have seen in the last section the template matching formulation of DNNs via LSOs simply as being the slope of the linear transform. By definition of template matching, there exists an internal ''matching'' procedure performed by the DNN. We propose to study this inference problem in this section in standard DNN and why can we label DNNs as template matching machines. As we will see, a greedy, per layer, maximization problem is governing the spline selection and thus template inference. We then study the impact of choosing different LSOs, such as ReLU or absolute value and their impact in the inference problem each one performs.
Deriving such results will allow two main applications. 
Firstly, with the convexity property, one can derive arbitrary splines with regions that can be implicitly changed and learned ''online'', as the selection will become intrinsically partition agnostic, known as adaptive partitioning \cite{hannah2013multivariate}. Secondly, the inference problem will be of great interest when dealing with deep neural networks analysis in further sections.
We first briefly describe some theoretical results to link inference-LSOs-template matching.

\subsubsection{Template Matching in the Context of Splines}
We study in this context under what condition spline functions can be considered to perform template matching. Let first define what do we refer to as template inference.
\begin{defn}
For a spline functional (univariate;multivariate;SO), given a partition of the input space denoted by $\Omega=\{\omega_1,\dots,\omega_R\}$ and local mappings $\phi_1,\dots,\phi_R$, the inference problem refers to, given an input $x$, finding the region in which it belongs:
\begin{center}
    Given $x$: Find $\omega \in \Omega$ s.t. $x \in \omega$.
\end{center}
This region is then used to perform the actual mapping via $\phi[x]$.
\end{defn}
As we now describe, this problem can represent very interesting behaviors linked with template selection in some cases, especially when the functional is convex.
We study in this section the convexity criteria for spline operators and the associated spline inference problem.
Note that we focus now on linear functionals, provided results can easily be extended.
\begin{theorem}\label{theorem1}
Given a linear multivariate spline $s[\textbf{a},\textbf{b},\Omega]$ we have 
\begin{align}
    s[\textbf{a},\textbf{b},\Omega](x)=&a[x]^Tx+b[x]\\
    =&\max_{r=1,\dots,R}a_r^Tx+b_r,\forall x,
\end{align}
if and only if $s[\textbf{a},\textbf{b},\Omega]$ is a convex function\cite{hannah2013multivariate}.
\end{theorem}
This theorem states that given a convex spline function, finding the region to which an input $x$ belongs to is equivalent to finding the region in which the mapping leads to the highest output. 
We provide an illustrative example in Fig. \ref{figexp}.
This result provides ways to create adaptive partitioning convex splines simply by learning the collection of hyperplanes with the mappings defined as the maximum of the hyperplane projections.
\begin{figure}
    \centering
    \includegraphics[width=5in]{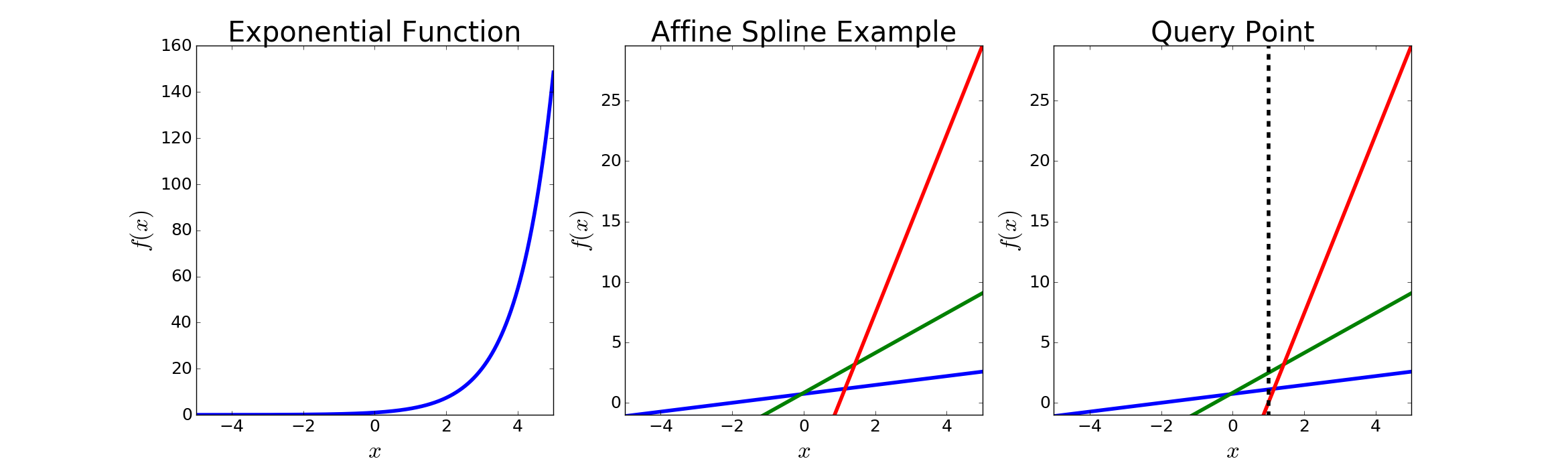}
    \caption{Illustrative examples of the theorem where one can see that the sub-region/sub-function associated with the query point is the one returning the maximum value among all the possible sub-functions.}
    \label{figexp}
\end{figure}
We can now extend the result to LSOs made of a collection of $K$ multivariate splines.
\begin{theorem}\label{theorem2}
Given an LSO defined $S\left[\left(s[\textbf{a}_{k},\textbf{b}_{k},\Omega_k]\right)_{k=1}^K\right]$, with all internal linear multivariate splines $s[\textbf{a}_{k},\textbf{b}_{k},\Omega_k]$ being convex, we have
\begin{align}
S[\textbf{A},\textbf{b},\Omega^S](x)=\Phi^*(x),
\end{align}
with $\phi^{S*}=\argmax_{\phi^S \in \Phi^S} \langle \phi^S(x),1\rangle$, with $\Phi^S=\{\phi^S_r,r=1,\dots,R\}$ and with $\phi^S_r(x)=A_rx+b_r$.
\end{theorem}
\begin{proof}
\begin{align}
    \phi^{S*}=\argmax_{\phi^S \in \Phi^S} \langle \phi^S(x),1 \rangle =&\argmax_{[\phi_1,\dots,\phi_K]\in \Phi_1\times \dots\times \Phi_K} \langle \Phi(x),1 \rangle\nonumber\\
    =&\left[
    \begin{matrix}
    \argmax_{\phi_1 \in \Phi_1}  \sum_{k=1}^K\phi_k(x)\\
    \vdots \\
    \argmax_{\phi_K\in \Phi_K}  \sum_{k=1}^K\phi_k(x)\\
    \end{matrix}
    \right]\nonumber\\
    =&\left[
    \begin{matrix}
    \argmax_{\phi_1\in \Phi_1}  \phi_1(x)\\
    \vdots \\
    \argmax_{\phi_K\in \Phi_K}  \phi_K(x)\\
    \end{matrix}
    \right]\nonumber\\
    =&\left[
    \begin{matrix}
    \phi_1[x]\\
    \vdots \\
    \phi_K[x]\\
    \end{matrix}
    \right],\;\; s_k \text{ convex } \forall k\nonumber\\
    =&\Phi[x]
\end{align}
\end{proof}
This last theorem leverages the independence between the multivariate splines making up the spline operator. As a result, the per multivariate spline region selection solved via the $\max$ operator in case the are convex can be done for all multivariate spline simultaneously via the $\langle.,1 \rangle $ operator, leading to the sum of the output dimensions.

\subsubsection{DNNs Are Composition of Adaptive Partitioning Splines}
As demonstrated in the last section, convex splines defined through a max over hyperplanes projections is defined as adaptive partitioning as changing the hyperplanes induces changes in the input space partitioning. Hence for regression problems for example, optimal partitions can be found in this manner simply by tweaking the hyperplanes parameters \cite{hannah2013multivariate,magnani2009convex} and has been shown to be very performant in the context of Nonlinear Least Square Regression. In fact, hyperplanes combination to solve function approximation problems go back to \cite{breiman1993hinging} reinforcing the fact that we can now see current state-of-the-art (sota) DNN as efficient composition of such approaches. We now describe this last statement in details.
Current sota DNNs leverage the Relu or LReLU nonlinearity, both convex, as well as max and/or mean-pooling, being also convex mappings. Based on the results drawn from the last section we can thus see that all the succession of linear layers such as FC-layer of convolutional layer followed by nonlinearities and possibly sub-sampling correspond to adaptive partitioning multivariate spline function.
In fact, one has in those cases

\begin{tabular}{|p{0.92\linewidth}|}\hline 
\begin{theorem}
DNNs with convex activation functions s.a. Relu or LReLU, and/or convex sub-sampling s.a. mean or max pooling applied on linear layers are composition of partition adaptive splines\cite{hannah2013multivariate,magnani2009convex},
\begin{align}
S_{\sigma}(Wx+b)=A_{\sigma,r^*}x+b_{\sigma,r^*},\;\text{ with } r^*=\argmax_{r=1,\dots,R^S_\sigma} \langle A_r(Wx+b)+b_r,1\rangle.
\end{align}
\end{theorem}
\\ \hline
\end{tabular}\\ \\
Given the last theorem, one might wonder if the local per layer partition optimization can be extended to a global adaptive partitioning. This question is answered in Appendix \ref{global} where we provide sufficient condition to obtain a globally convex DNN, hence making the last theorem not only applicable on a per layer basis but overall the mapping. 
In fact, composition of such layers are in general not globally convex with unconstrained weights. 
We now have linked DNN to known powerful frameworks for function approximation and can now provide ways to visualize the final inferred templates with standard DNNs.
We propose to do so in the net section in order to highlight the extrem adaptivity DNNs have with this regard.

\subsubsection{Input Encoding and Template Visualization}
In this section we present experiments on MNIST and CIFAR10 to provide visualization of the adapted templates given few samples. We also provide a simple methodology to compute the templates.
Since the final DNN can be expressed as $f_\Theta(x)=A[x]x+b[x]$, it is clear that we can obtain the adapted template for class $c$ as $a[x]_{c,.}=\frac{d f_\Theta(x)_c}{dx}$. 
We present below computed templates for one model, the LargeCNNmean. We remind that specific model descriptions are provided in Appendix \ref{description}. All the other templates related to other topologies are provided in Appendix.
\begin{figure}[!htb]
    \centering
    \begin{minipage}{.5\textwidth}
        \centering
        \includegraphics[width=3in]{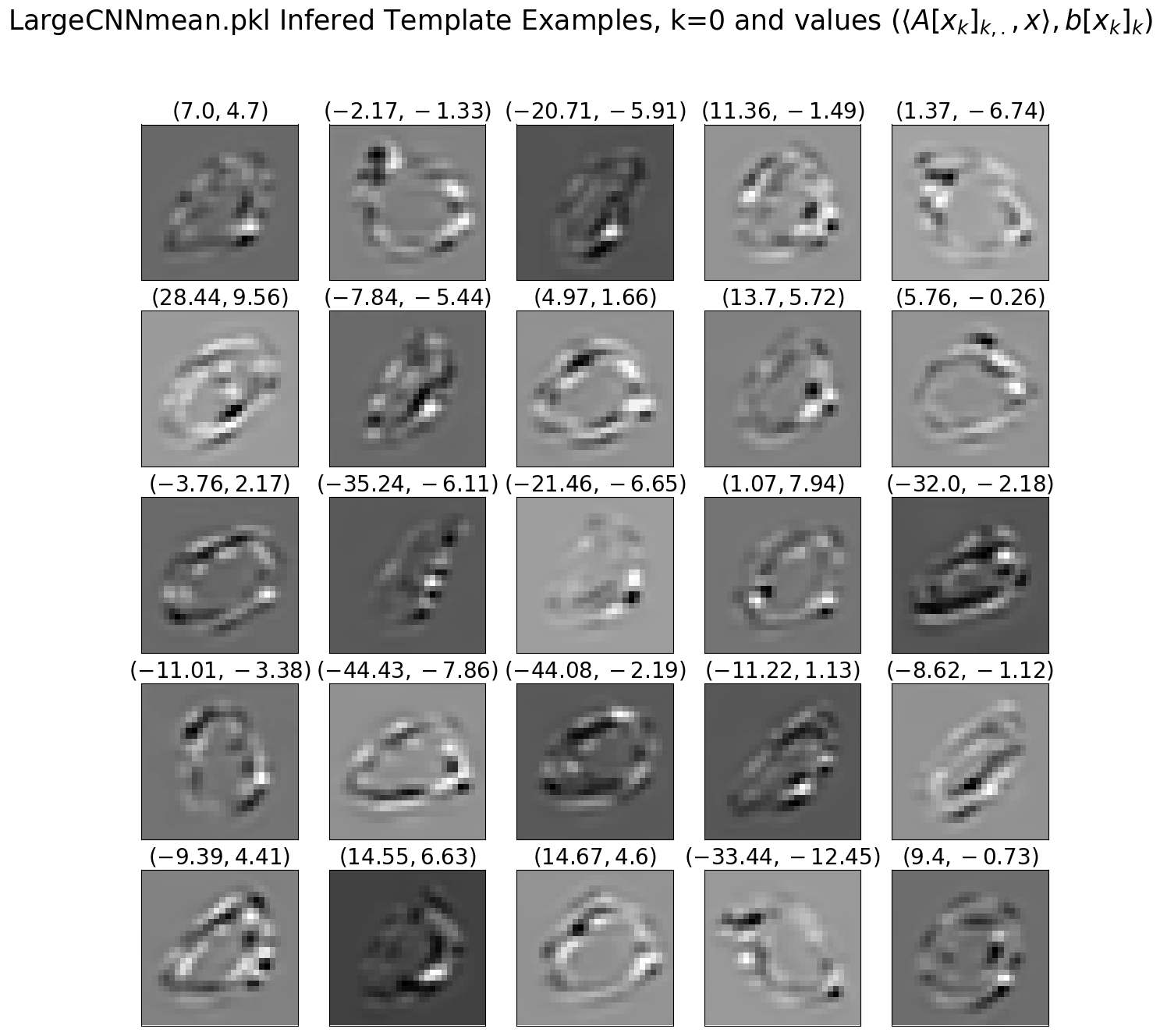}
    \end{minipage}%
    \begin{minipage}{0.5\textwidth}
        \centering
        \includegraphics[width=3in]{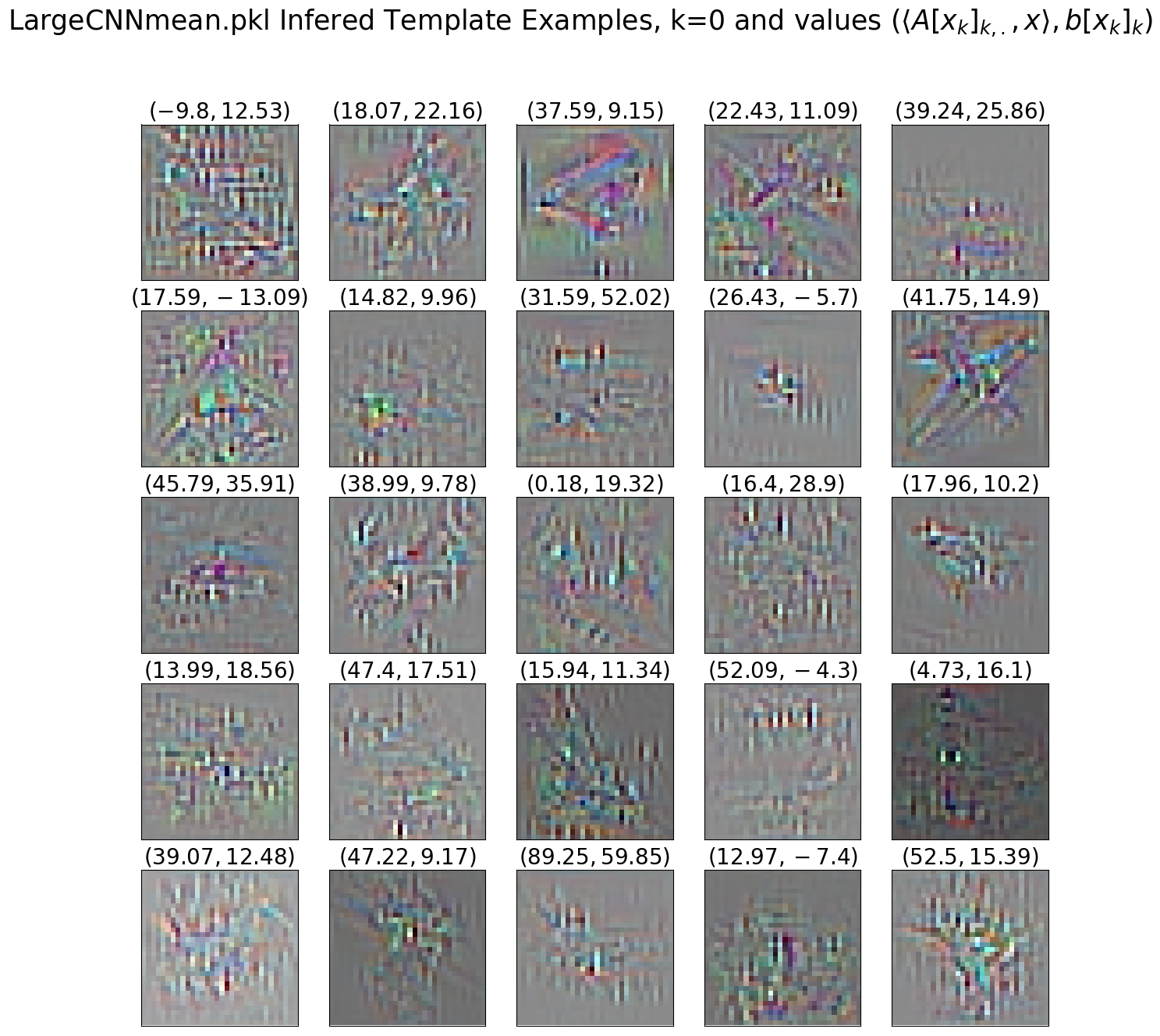}
        \label{fig:prob1_6_1}
    \end{minipage}
        \caption{Left: Examples of Templates for class $0$ given inputs $X_n$ belonging to class $0$ ($Y_n=0$). As one can see, the class $0$ templates fully adapt to their input for LargeCNNmean. Right: Templates on CIFAR10 associated to class plane. Interestingly the background is not captured and only the class shape is present.}\label{fig:oneclass_templates}
\end{figure}
We present in Fig. \ref{fig:oneclass_templates} different templates induced for one class with the input belonging to this same class, hence this template matching matches the class of interest. In the title of each subplot is provided the template matching output $\langle A[X_n]_{Y_n},X_n\rangle$ as well as the bias $b[X_n]_{Y_n}$. One can see the adaptivity of the templates but also for the CIFAR10 case (on the right) the ''denoising'' of the background. Only the object of interest remains.
\begin{figure}
    \centering
    \includegraphics[width=6in]{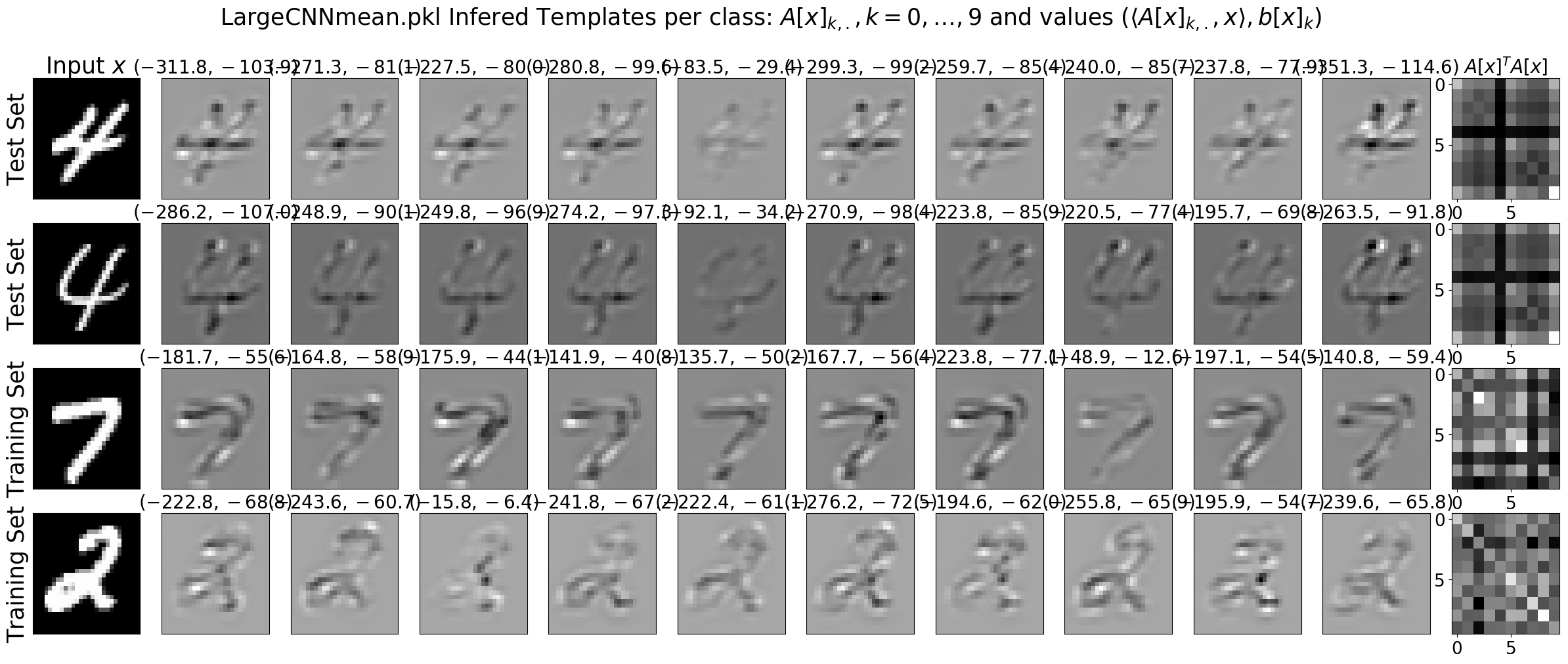}
    \caption{Depiction on MNIST for $4$ inputs belonging to class $4,4,7,2$ of all the classes templates with classes from $0$ (left) to $9$ (right). Right column is the Gram matrix of the templates representing the correlation between templates of different classes for each input. In the subplot titles are the values of the template matching and the bias.}
    \label{fig:mnist_templates}
\end{figure}
We also present in Fig. \ref{fig:mnist_templates} more specific templates for $4$ different inputs. In particular we provide the templates of all the classes for a given input as well as the gram-matrix of those templates, representing their correlation. It is interesting to denote that the wrong class templates are not $0$ neither white noise like but on the contrary inversely correlated w.r.t. the input and the correct template class. This phenomenon will be explained in details in the later section discussing the optimal template and their convergence.
\begin{figure}
    \centering
    \includegraphics[width=6in]{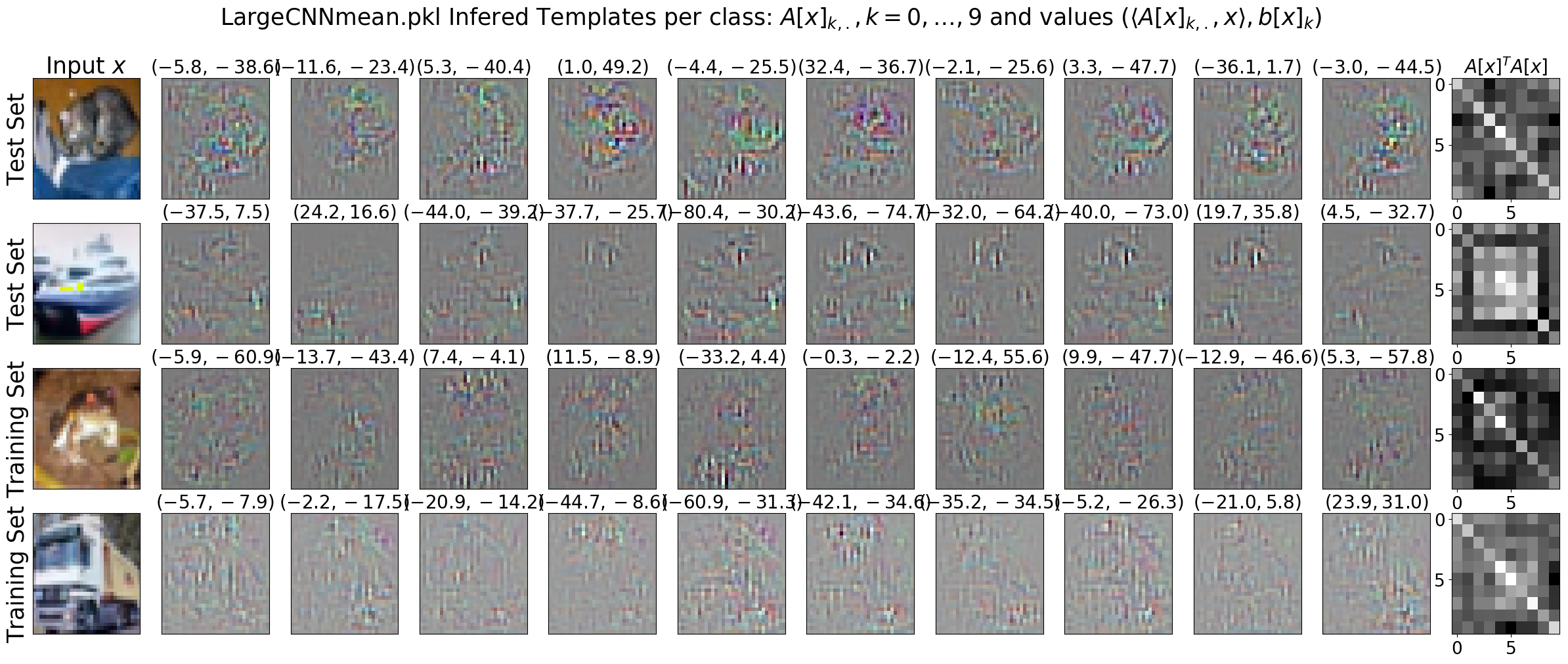}
    \caption{Depiction on CIFAR10 for $4$ inputs belonging to class cat,boat,frog,truck, of all the classes templates with classes from $0$ (left) to $9$ (right). Right column is the Gram matrix of the templates representing the correlation between templates of different classes for each input. In the subplot titles are the values of the template matching and the bias.}
    \label{fig:cifar_templates}
\end{figure}
Finally, in Fig. \ref{fig:cifar_templates} we provide the same type of analysis but for CIFAR10. Clearly the shape and class are less distinguishable in the templates, yet they are fitted to the input so that the template matching of the right class is indeed the maximum of all the classes.

Concerning the encoding of a given input $x$, a DNN find very close interpretation with standard signal processing tools. In fact, it encodes $x$ via localization information and amplitudes as would be the case of Fourier transform with phase and amplitude. In fact, we denote by amplitude the result of $A[x]x$ and by phase the collection of regions at each layer in which the input belongs to. If we denote by $\Omega[x]^{(\ell)}$ the region at level $\ell$ in which $\bz^{(\ell)}$ belongs to, then a DNN representation of an input $x$ is defined as follows.
\begin{prop}
A DNN encoding of an input $x$ is define as the complementary couple of amplitude and phase defined as $(A[x]x,(\Omega[x]^{(\ell)})_{\ell=1}^L)$. This information is indeed enough to fully reconstruct an input as we will demonstrate later with deep neural network inversion and semi supervised applications.
\end{prop}

%% file: semisup.tex
\section{Theoretical Results: Generalization, Optimal Learning, Memorization}
In this section, we will leverage the derived spline framework to provide analytical understanding of current DNNs. To do so, we first demonstrate the impacts of regularization into the quest of generalization performances for DNNs. Through regularization we can obtain analytical optimal templates and thus provide a clear methodology to guide DNNs towards this optimum.
Through this analysis, results on adversarial examples, DNN inversion and optimization schemes will be studied. Let first review the generalization problem for DNNs and the mathematical tools that can be leveraged.

\subsection{What is Generalization for DNN and why Regularization is Key}\label{generalization}
As detailed in the introduction Section \ref{sec:intro}, generalization is the ability for the approximant $f_\Theta$ to reproduce the behavior of the true unknown functional $f$ on new points not present in the training set $\mathcal{D}$. This being very general we now have to distinguish two important cases. First, when $f$ represents the ''law of nature'' in the sense that it follows some fundamental unbreakable laws. Hence, given an input $x$, only one possible outcome $f(x)$ exists, fully determined by fundamental laws such as thermodynamics. A typical example would be to predict the internal energy of a system. The second case concerns machine learning. It is the one where $f$ is ''human'', corresponding to human perceptions, a qualitative, normed interpretation of an input $x$. Typical example would be for $x$ to be a pixel representation of a scene and the associated $f(x)$ the human associated label of the \textit{main object of interest}. However, from individuals to others, the core definition of $f$ might change, raising questions about the term generalization in computer vision. Due to this unclear definition, the only quantitative measure one has of generalization is the use of a test set on which $f_\Theta$ has not be trained. This test set is used to compare to predictions based on $f_\Theta$ with the known  correct outputs of $f$. Since we focus on accuracy performance, we denote this loss by $\mathcal{L}_{AC}$ standing for accuracy loss. As a result, one uses $\mathcal{L}_{CE}$ the cross-entropy loss and $\mathcal{D}$ to update the parameters $\Theta$ and $\mathcal{L}_{AC}$ with $\mathcal{D}_{test}$ for a quantitative measure of generalization of the trained approximant. 
Due to the discussed context, we have to consider generalization differently than simply maximizing the test set accuracy. In fact, we propose here to define generalization as a measure of performance consistency from the training set to the test set. Hence, even a less accurate model is favored if its ability do not vary from samples used for its training to new samples. 
\begin{defn}
We define the generalization measure $G$ of a trained network $f_\Theta$ by the average empirical difference of performance between training set and test set as
\begin{equation}
    \mathcal{L}_G(\mathcal{D},\mathcal{D}_{test},f_\Theta)=d\left(\mathcal{L}_{AC}(\mathcal{D},f_\Theta),\mathcal{L}_{AC}(\mathcal{D}_{test},f_\Theta)\right),
\end{equation}
with $d$ a distance metric, $\mathcal{L}_{AC}$ a generic performance metric linked with quality of the prediction, accuracy loss in our case
\begin{align*}
&\mathcal{L}_{AC}(\mathcal{D},f_\Theta)=\frac{1}{Card(\mathcal{D})}\sum_{(X_n,Y_n) \in \mathcal{D}}\mathcal{L}_{AC}(Y_n,\hat{y}(X_n))\\ 
&\mathcal{L}_{AC}(\mathcal{D}_{test},f_\Theta)=\frac{1}{Card(\mathcal{D}_{test})}\sum_{(X_n,Y_n) \in \mathcal{D}_{test}}\mathcal{L}_{AC}(Y_n,\hat{y}(X_n)).
\end{align*}
\end{defn}
As a result, a network performing similarly on train and test set is considered as optimal in term of generalization of its underlying learned representation.
This results in a systematic way to measure and learn topologies as we now define the overall objective.
\begin{center}
\begin{tabular}{|p{0.97\linewidth}|}\hline
\begin{defn}
The optimal network given a finite training set $\mathcal{D}$ and test set $\mathcal{D}_{test}$ is defined as
\begin{align}
    f^*_{\Theta^*}=\argmin_{f_\Theta}\underbrace{\mathcal{L}_{AC}(\mathcal{D},f_\Theta)}_{\text{Empirical Risk Minimization: $\Theta^*$}}+\underbrace{\mathcal{L}_G(\mathcal{D},\mathcal{D}_{test},f_\Theta)}_{\text{Structural Risk Minimization: $f^*$}}
\end{align}
\end{defn}
\\ \hline
\end{tabular}
\end{center}
This search of the optimal approximant can thus be done in a two-step process by first fixing a topology and then minimizing $\mathcal{L}_{CE}$ synonym of minimization of $\mathcal{L}_{AC}$ on the training set. Doing this over multiple topologies and then selecting the optimal network by search of the one with minimum generalization loss $\mathcal{L}_G$.
Since this results in learning of tremendous possible models, one usually tries to find a way to translate $\mathcal{L}_G$ into a differentiable loss that can be used on the training set.
This usually takes the form of standard regularization such as Tikhonov, dropout and so on. Yet, those approaches can only impact the final parameters $\Theta$, and thus have only a limited impact on the true generalization loss as opposed to topological changes in $f$. Nevertheless, we now develop in the following section precise analysis and results to link regularization with generalization, overfitting. This will also to understand the dataset memorization problem and what generalization actually means for DNNs. The next section will however build on those results and attempt to tackle this problem from a broader point of view via a systematic way to estimate $\mathcal{L}_G$ prior learning hence allowing easier design search.

Given an approximant $f_\Theta$, the search for best generalization performances is commonly interpreted as finding parameters $\Theta$ in a flat-minima region. A flat-minima region is a part of the parameter space associated with great generalization performance of the approximant. The term flatness is easily interpreted as follows. One seeks $\Theta+\epsilon$ to also belong to this region. Hence moving around $\Theta$ still produce great generalization leading to a flat generalization performance as $G(f_\Theta))\approx G(f_{\Theta+\epsilon})$. This is opposed to sharp-minima where $||G(f_\Theta))- G(f_{\Theta+\epsilon})|| \gg ||\epsilon||$.
This analysis started long before current DNN outbreaks. Generalization, in addition of being associated to flat minima \cite{wolpert1994bayesian} is also mapped to  complexity of networks
\cite{hochreiter1995simplifying} which is linked with Kolmogorov complexity and Minimum Description Length. This comes from the fact that flat minima are associated to simpler networks which then leads to high generalization \cite{schmidhuber1994discovering,hochreiter1995simplifying}. However, standard analysis if hardly applied as the measure and definition of a DNN complexity is still not clear. Practical approaches aiming at guiding $\Theta$ towards flat-minima then took different forms.
From one side, reduction of the number of degrees of freedom via weight sharing led to promising results \cite{nowlan1992simplifying,rumelhart1986parallel,
lang1990time,yann1987modeles,lecun1989generalization} while being very general and model agnostic. Another approach uses early-stopping as in \cite{morgan1990generalization,weigend1990predicting,
vapnik1992principles,moody1994architecture,guyon1992structural} motivated by the famous  point of inflexion of the testing error, first reducing till a breaking point where it increases. This point is the optimal to stop training as generalization error is minimized. 
Both methods require inside information and expertise. Thus the search for a more principled method possibly adaptive to any case led to regularization studies.
To do so, penalization of complex networks was applied and  led to great advances in the flat-minima search.This complexity based approach takes the form of Occam's razor principle \cite{blumer1987occam} and was in practice applied via weight penalization
\cite{mackay1996bayesian,hinton1986learning,
hinton1987learning,plaut1986experiments,williams1995bayesian}. For example, with norm based regularization. In fact, in the case of Tikhonov regularization, a very intuitive interpretation of the wights appear: \textbf{the weight amplitudes is equal or proportional to their error derivative, a.k.a their importance}. By making weights amplitude correlated to their role in the loss minimization, only the necessary one will reaming nonzero hence simplifying the network through sparsity of the model/connections.
Finally, input and/or weight noise applied during training has also found equivalences with regularization. It consists of perturbing the input or the current set of parameters with additive or multiplicative noise throughout the learning phase. This has recently took the form of dropout, a multiplicative noise with Bernoulli variables \cite{srivastava2014dropout,gal2016dropout,srivastava2014dropout}  randomly turning neurons or connections to $0$ in DNNs. This concept goes back to synaptic noise \cite{murray1993synaptic},
and \cite{matsuoka1992noise} where generalization performances of neural networks trained with backpropagation is studied via introduction of noise to the input. It was then shown in 1995 \cite{bishop2008training} that introduction of noise during learning is equivalent to a generalized Tikhonov regularization technique. More precisely, it has been shown that while additive noise provides an induced penalty term on the norm of the weights a la Tikhonov, multiplicative noise provides a weighting of this regularization based one the Fisher information of the weights \cite{li2016whiteout}.
A probabilistic interpretation of dropout indeed demonstrates the push of the weights towards sparse solutions \cite{nalisnick2015scale}. 
Going further, an explicit regularization term is found from dropout and extended in \cite{wager2013dropout}.
Based on those approaches, we can now have the following intuitive explanation of why current topologies work so well:
\begin{itemize}
    \item Multi-Objective Regularization: Introduction of noise during learning coupled with explicit norm based penalties on the weights
    \item Weight-Sharing: Convolutional topologies allow extremely  efficient and smart weight sharing for perceptual tasks reducing the number of degrees of freedom while providing very high-dimensional mappings
    \item Cross-Validation and Early-Stopping: huge resources now allow fine search of topologies and hyper-parameters
\end{itemize}
We now provide in the next section theoretical results in the context of regularization. As we will see norm constraints on the free parameters, hence the templates for DNNs allow to obtain closed form optimal theoretical templates. From this, different results will be derived from adversarial example existence to dataset memorization and network inversion.


\subsection{Learning Optimal Templates}
 
In this section we study the learning of the DNN templates for two cases. First in the case of a loss function without regularization terms. Secondly when sparsity constraints is imposed. For both cases we study what are the optimal templates, their convergence and demonstrate the need for regularization. Regularization will be shown to make the problem of optimal template learning well-defined as well as being robust to poor weight initialization. Based on this, we then provide a methodology to quantify the quality of a given DNN topology and weight initialization schemes simply based on the induced templates and their potentials in the next Section \ref{subsub:sys}.

\subsubsection{Unregularized Learning Solution: Unstable Training}\label{subsub:colinear}

We present here the general analysis for any DNN using the cross-entropy loss function coupled with softmax activations. 
We denote by $A_c$ the $c^{th}$ template. There is no input conditioning ($A[x]$) as we aim at finding the explicit optimal form this template should have given the input $x$. Hence for now $A_c$ is an generic template.
The global loss function to be minimized is the negative cross-entropy between the true label $Y_n$ and the estimation $\hat{y}(X_n)$. We remind that it is defined as
\begin{equation}
    \mathcal{L}_{CE}(Y_n,\hat{y}(X_n))=-(\langle A_{Y_n},x\rangle+b_{Y_n})+\log\Big(\sum_{c=1}^C e^{\langle A_c,x\rangle+b_c} \Big).
\end{equation}
From this we apply standard iterative gradient based minimization procedures to seek the optimal templates $A_c$ for a given input $X_n$ and for each classes.
We have
\begin{align}
    \frac{\partial \mathcal{L}_{CE}(Y_n,\hat{y}(X_n))}{\partial A_{Y_n}}&=-X_n +X_n \frac{ e^{\langle A_{Y_n},X_n\rangle+b_{Y_n}}}{\sum_{c=1}^C e^{\langle A_c,X_n\rangle+b_c}}\nonumber \\
    &= X_n(\hat{y}(X_n)_{Y_n}-1),\\
    \frac{\partial \mathcal{L}_{CE}(Y_n,\hat{y}(X_n))}{\partial A_c}&=X_n\frac{ e^{\langle A_c,X_n\rangle+b_c}}{\sum_{k=1}^C e^{\langle A_k,X_n\rangle+b_k}} \nonumber \\
    &=X_n\hat{y}(X_n)_{c},\forall c \not = Y_n,
\end{align}
leading to the following gradient update rule with learning rate $\lambda$
\begin{align}
    A_{Y_n}^{(t+1)}=&A_{Y_n}^{(t)}-\lambda X_n(\hat{y}_{Y_n}(X_n)-1)\nonumber \\
    =&A_{Y_n}^{(t)}+\lambda X_n(1-\hat{y}_{Y_n}(X_n))\\
    A_{c}^{(t+1)} =& A_c^{(t)}- \lambda X_n \hat{y}(X_n)_{c}, \forall c \not = Y_n.
\end{align}
We do not analyze the behavior for the biases $b$ since they do not interfere with the optimal templates.
It is clear that $\hat{y}(X_n)_{c}>0$ as well as $(1-\hat{y}_{Y_n}(X_n))>0$. This implies that the update rule adds the re-scaled input $X_n$ for the correct template $A_{Y_n}$ whereas for the other classes, $X_n$ is substracted. This way, by adding or substracting the input to the templates, the template matching mapping can either increase or decrease. 
It becomes clear that the sensitivity to initialization is extreme as given a starting random template, it can only moves along the input $X_n$ direction. Hence there are infinitely many optimal templates, one for each starting point. Moreover, the similarity between the templates and the input will also depend on this initialization. We depict this phenomenon in Fig. \ref{fig:memorization} on the left subplot.
For the case of a structured templates as in practice it is defined as a composition of affine mappings with different internal parameters, it is clear that the update will push the template as close as possible to this optimal based on the ability of the mappings composition to produce it. With increasing number of free parameters and network complexity it is fair to assume that the induced update will be close to this optimum.
We now dive into the regularized case.

\subsubsection{Regularized Learning: Global Optimum, Robust, implies Dataset Memorization}\label{subsub:memo}
By adding a regularization penalty to the loss function such as sparsity constraint with norm based loss, we can obtain analytical optimal templates as the optimization problem becomes well-defined. As we will see, dataset memorization is the global optimum.

While memorization is often associated to overfitting and bad performances, we will see that this general statement is more ambiguous. By memorization, we denote the ability for DNNs template  to become collinear to their input.
In fact, the term memorization itself should be seen as a good ability of a DNN if it holds for arbitrary inputs from the manifold of interest. In this case, the DNN would be able to span all inputs of interests, effectively making it a basis of the training set, testing set and so on. 
We now study the impact of norm constraints on the optimal templates of DNNs with only assumption that all inputs $X_n$ have same energy, in particular $||X_n||^2=1$.

\begin{theorem}
In the case where all inputs have identity norm $||X_n||=1,\forall n$ and assuming all templates denoted by $A[X_n]_c,c=1,\dots,C$ have a norm constraints as $\sum_{c=1}^C ||A[X_n]_c||^2\leq K,\forall X_n$ then the unique globally optimal templates are
 \begin{equation}
     A^*[X_n]_c=\left\{ \begin{array}{l}
          \sqrt{\frac{C-1}{C}K}X_n,\iff c=Y_n\\
          -\sqrt{\frac{K}{C(C-1)}}X_n,\text{ else}
     \end{array}\right.
 \end{equation}
 \end{theorem}
\begin{figure}[t!]
    \centering
    \includegraphics[width=6in]{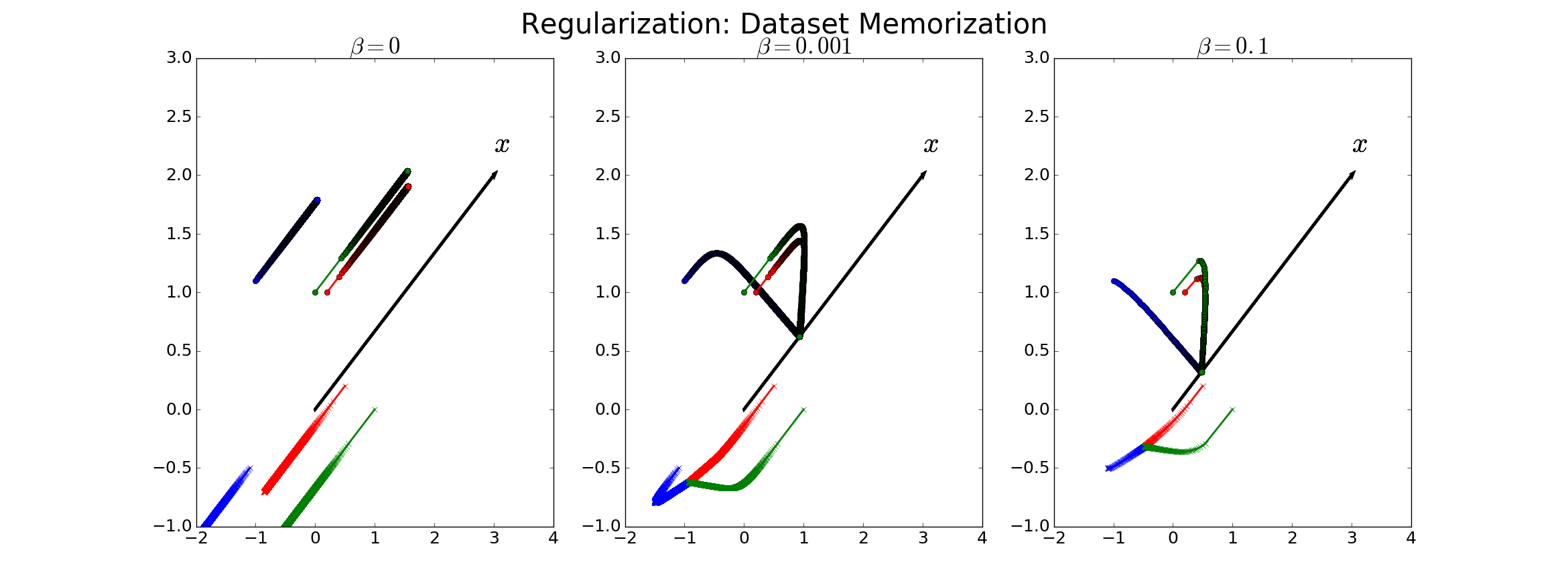}
    \caption{Tikhonov based regularization implies a push of the templates towards dataset memorization.}
    \label{fig:memorization}
\end{figure}
In order to highlight and provide intuition on the two derived results, we provide in Fig. \ref{fig:memorization} a simple example in $2$ dimensional case. We initialize the templates randomly and set the templates of opposite classes to be equal a initialization to better see the update differences based on a same starting point. On the left part no Tikhonov regularization has been applied whereas the middle and right plots contains standard and strong regularization. We can see that with regularization, templates converge toward a rescaled $X_n$, with scale depending on the regularization parameter. This convergence speed is also dependent on this parameter. On the other hand for the no penalty case, a simple move of the templates can be seen as ill-posed and leads to very pool templates. Hence, dataset memorization can not be said to always be synonym to overfitting and thus poor performance as they might correspond to optimal templates.
From the derived results, we now propose to analyze adversarial examples though the lenses of templates and optimal templates. As we will see, adversarial examples are natural and should not be fought, high sensitivity however should be controlled.

\subsection{Adversarial Examples Are Natural, being Fooled is Not}

In this section we propose a two-fold analysis of adversarial examples. 
Note that we focus here on model based adversarial examples which are optimized based on a trained DNN. 
Firstly we demonstrate that existence of adversarial noise is due to optimal templates.
In particular we demonstrate that training with current loss function and condition can only lead to adversarial example existence. 
Secondly we study the sensitivity of DNN mappings in general including adversarial noise sensitivity. To do so we propose a simple methodology to compute the Lipschitz constant of any DNN layer via the corresponding LSO. By allowing explicit computation of the overall mapping sensitivity we further justify the need for sparsity constraints on the parameters.
We first review briefly what are adversarial examples.

We remind that we denote a differentiable mapping $f_\Theta: \mathbb{R}^D \rightarrow \mathbb{R}^C$ that we call a pre-processing system producing a feature map e.g. SIFT, wavelets, DCN. In order to produce a class distribution $\hat{y}$ a softmax nonlinearity is applied as presented in Section \ref{sec:intro}.
We further denote as $X_n^{Y_n}$ an input image $X_n$ of class $Y_n$. Adding to it a small optimized noise in order to fool the system to predict a wrong class $k,k\not = Y_n$ lead to the new input denoted $X_n^{Y_n \rightarrow k}$.
In order to generate an adversarial example we use
\begin{equation}\label{adveq}
    X_n^{Y_n \rightarrow k}=X_n+\alpha \frac{d \hat{y}(X_n^{Y_n})_{k}}{d X_n^{Y_n}}.
\end{equation}
It thus intuitively corresponds to pushing by an amplitude of $\alpha \in \mathbb{R}^+$ the natural input $X_n^{Y_n}$ with the wrong class template $A[X_n^{Y_n}]_k$. 
Clearly, the demonstration that a very small $\alpha$ leads to a complete change in the prediction $\hat{y}(X_n)$ implies that the joint system is not Lipschitz contractive since
\begin{align}
     \| \hat{y}(X_n^{Y_n})-\hat{y}(X_n^{Y_n \rightarrow k})\| \gg &\| X_{n}-X_n^{Y_n \rightarrow k}\|\\
     \gg &\alpha \|  \frac{d \hat{y}(X_n)_{k}}{d X_n}     \|.
\end{align}

We first review briefly previous work on adversarial examples.
In order to become more robust to adversarial examples \cite{gu2014towards,lyu2015unified} developed gradient regularization techniques by imposing a sparsity penalty term on the partial derivatives of the output the the deep nets layers w.r.t. their inputs. This is as we showed above on way to reduce the amplitude of the adversarial examples by imposing that $||\frac{d f_\Theta(X_n)}{dX_n}|| < \epsilon$. On the other hand, using adversarial examples during training has also been done showing better performances on the test set as in \cite{shaham2015understanding}. However none of these methods provide a guarantee on the generalization of the technique for new unseen examples and adversarial examples. On the opposite in \cite{fawzi2015analysis}, has shown that for linear classifier, invariance can not be achieved in general. Similarly, in \cite{gu2014towards} the idea of being invariant is rejected stating that one can always engineer adversarial noise, however this happens to not be the case following the sufficient condition we propose.
Another approach is proposed by defensive distillation \cite{papernot2016distillation} which consists of artificially increasing the output of $f_\Theta(X_n)$ or the input to the softmax as stated which thus forces the prediction $\hat{y}(X_n)$ to become $\epsilon$-close to a one hot representation and thus making the norm of the derivative w.r.t. input close to $0$-norm, this is interesting has it exploits the vanishing gradient problem which has the same time pushed to $0$ the margin to robustify the network against adversarial example and has been questioned in \cite{carlini2016defensive}.
While competitive robustness has been achieved in \cite{gu2014towards} via the jacobian penalty term it was already hint that a deeper problem resides in the training procedures and the loss function as it is confirmed.
In \cite{szegedy2013intriguing} additional weight decay on the parameters is applied to the loss function thus reducing $||\frac{df_\Theta(X_n)}{dX_n}||$ leading to smaller adversarial energy without yet removing its presence.

\subsubsection{Unregularized Optimal Templates Imply Adv. Noise}
We study the impact of the templates concerning the presence of adversarial examples.
When regularization is used during training, the optimal templates for any given input $X_n^{Y_n}$ are of the form $A[X_n^{Y_n}]_{Y_n} \propto X_n^{Y_n}$ and $A[X_n^{Y_n}]_{c} \propto -X_n^{Y_n},c \not = Y_n$. Hence it is clear that adversarial examples based on model optimization can only scale down the input $X_n$ hence $X_n^{Y_n \rightarrow c}=\kappa X_n^{Y_n},\kappa <1$.
For the case of unregularized templates however, the input additive transformation will introduce the initial template. As we saw, the final template after learning results from the initialized temples plus a succession of updates pushing it by a rescaled version of $X_n$. Hence we denote $A[X_n^{Y_n}]_c=A[X_n^{Y_n}]_c^0+\beta X_n^{Y_n}$.
The noisy input thus corresponds to $X_n^{Y_n \rightarrow c}=X_n^{Y_n}+\kappa (A[X_n^{Y_n}]_c^0-\beta X_n^{Y_n})$. As a result, it does not just scale down the input but add a random noise to the input. This random noise corresponding to part of the wrong class template, implies much higher unstabilities and can look like standard noise to us, as it is in practice.

We now present two properties on the gradient based update in the case of no regularization that will be used to demonstrate the rise of adversarial examples as part of the weight optimum convergence. This first one stats that as the number of classes increase, the wrong class templates will be less and less updated in the opposite direction of the input making the initialization noise more present in the adversarial noise possible leading in the extreme case
$X_n^{Y_n \rightarrow c}=X_n^{Y_n}+\kappa (A[X_n^{Y_n}]_c^0)$
\begin{prop}
The sum of the updates for the wrong classes add up the the opposite of the update of the correct class denoted as 
\begin{align*}
\sum_{c\not = Y_n} \frac{d \hat{y}(X_n^{Y_n})_c}{dX_n^{Y_n}}=&\sum_{c\not = Y_n} -X_n\hat{y}(X_n)_c\\
=&-X_n(1-\hat{y}(X_n)_{Y_n})\\
=&-\frac{d \hat{y}(X_n^{Y_n})_{Y_n}}{dX_n^{Y_n}}.
\end{align*}
\end{prop}
We can thus see that as the number of classes grow the less the wrong class templates will move away from their initial point.
We now present adversarial example specific analysis to quantify their impact on a given network using the tools and remarks developed in the previous section.
Using the chain rule and the definition of adversarial examples defined in Eq. \ref{adveq} we can measure the sensitivity of a network via analysis of the norm of $\frac{d \hat{y}(X_n)_{c}}{dX_n}$. We thus proceed to derive Lipschitz constant of DNNs in the next section.

\subsubsection{Ensuring Adversarial Noise Robustness via Lipschitz Constant minimization: Contractive DNNs}\label{subsub:lip}

In this section we describe the space contraction properties of DNNs and composition of LSOs in general.
As we will see, deriving the exact formula for any given deep neural network is straightforward and will allow us to better understand what causes ''chaotic'' behaviors as seen with adversarial examples.
Let first remind that for differential mappings $f_\Theta^{(\ell)}:\mathbb{R}^{D^{(\ell-1)}}\rightarrow \mathbb{R}^{D^{(\ell)}}$ the Lipschitz constant $K$ is equal to the infinite norm of the total derivative. Hence for LSOs $S[\textbf{A},\textbf{b},\Omega]$ differentiable almost everywhere we have 
\begin{align}
    K\leq \max_{r=1,\dots,R^S}||A_r||^2,
\end{align}
We now briefly present the Lipschitz constant of the most used layers, namely the ReLU layer, the pooling layer and the affine transforms. Finally, we will conclude with the softmax nonlinearity, which is present in any classification framework.

Firstly we study the general case of the affine transforms.
\begin{theorem}
For the affine mappings and nonlinearity layers we have
\begin{align}
     K_W=||Wx_1+b-(Wx_2+b)||^2=&||Wx_1-Wx_2||^2 \nonumber \\
     \leq&||W||^2||x_1-x_2||^2
\end{align}
\begin{align}
     K_{\bC}=||\bC x_1+b-(\bC x_2+b)||^2=&||\bC x_1-\bC x_2||^2 \nonumber \\
     \leq&||\bC ||^2||x_1-x_2||^2
\end{align}
\begin{align}
     K_{\sigma}=||A_\sigma[x_1]x_1-A_\sigma[x_2]x_2||^2     \leq&\max_{r=1,\dots,R^S}||A_{\sigma,r}||^2||x_1-x_2||^2\nonumber \\
     \leq&D^2||x_1-x_2||^2
\end{align}
\begin{align}
     K_{\rho}=||A_\rho[x_1]x_1-A_\rho[x_2]x_2||^2     \leq&\max_{r=1,\dots,R^S}||A_{\rho,r}||^2||x_1-x_2||^2\nonumber \\
     \leq&D^2||x_1-x_2||^2
\end{align}
with $D$ representing the output dimension.
Those translate into the norm of the weight for the FC and convolutional layers and upper bounded by the output dimension for ReLU,LReLU and max-pooling.
\end{theorem}
All the demonstration of the results are provided in the Appendix.
We also present the softmax nonlinearity which is a strictly contractive operator. In fact, we have the following result.
The softmax layer is strictly contractive with $K\leq \frac{C-1}{C^2}$
In general given a composition of affine spline operators with parameters $\textbf{A}^{(\ell)},\textbf{b}^{(\ell)}$ for the $\ell^{th}$ operator, we have the Lipschitz constant of their composition defined as 
\begin{equation}
    ||f_\Theta(x)-f_\Theta(y)||^2 \leq \left(\prod_{\ell=1}^L\max_{r=1,\dots,R} || A^{(\ell)}_{r}||^2\right) ||x-y||^2,
\end{equation}
The composition of LSOs thus inherits this property regarding its Lipschitz constant. Based on the previously derived upper bounds we can thus analyze in general the regularity property of DNNs depending on their layer composition and special topologies or weight constraints.
In particular, we propose to study the case of adversarial examples, a typical application of perturbation leading to unstable outputs. In fact, as presented in earlier section, adversarial examples represents an optimized perturbation introduced into an input before being fed into the DNN mapping. The large SNR implies that if the DNN output changes drastically, there is a clear regularity drawback for the mapping. As in practice this perturbation is able to completely fool the network making it predict with very high accuracy the incorrect class. Hence, practical evidence demonstrate the lack of contractivity.
Hence, based on the previous analysis, we see that two reasons exist.
\begin{cor}
Adversarial examples are caused by an ''explosion'' of the weight norms coupled with very high-dimensional mappings.
\end{cor}
In order to solve or lessen this effect two solutions appear.
Firstly, regularization applied on the weights can reduce the norms hence the irregularity of the mapping. As seen before, this penalty term is also crucial for optimal template learning. However there is also a second way to prevent unstable outputs and it is via sparsity of the activation. This denotes the number of neurons firing after a ReLU or LReLU nonlinearity and can be easily measured given an input and a given layer as $||A_\sigma^{(\ell)}[\bz^{(\ell-1)}[\bz^{(\ell-1)}||_0$. This activation sparsity in fact can be upper bounded by a quantity smaller than the unconstrained $D^{(\ell)}$. For example, replacing the standard nonlinearity with one letting go through only the $\kappa$ order statistics brings down the nonlinearity Lipschitz upper bound to $\kappa$. Another solution can be to impose an extra layer before the nonlinearity with aim to structure the input such that it can not be all positive, the worst case for the ReLU. This solution is discussed in details in Sec. \ref{subsec:sparse}.

%% file: unbiased.tex

\section{Extension}
In this section we propose to leverage the developed tools to propose some solution to current DNN drawbacks. This will consist of proposing a systematic way to ensure regularization and generalization measures via DNN inversion and input reconstruction. This will also allow the development of a generic semi-supervised and unsupervised strategy for DNNs. Secondly, we will study the impact of inhibitor connections to provide network stability, bias removal. One of the key concept of this part will consist of studying DNNs from a dual point of view : forward (template inference) pass and backward (reconstruction, learning) path. As we will see, adding the right connections can increase the forward sparsity whereas densifying the backward pass. Finally, we will develop a simple methodology to measure the quality and potential of untrained DNN topologies and weight initialization. This will find great application in topology search and automated DNN design as there is no longer need to train the network to obtain a qualitative measure.
In all cases, we also provide experiments on MNIST, CIFAR10.

\subsection{DNN Inversion: Input Reconstruction is Necessary for Generalization}\label{subsub:all_recon}

Deep learning systems have made great strides recently in a wide range of difficult machine perception tasks.
However, most systems are still trained in a fully supervised fashion requiring a large set of labeled data, which can be extremely tedious and costly to acquire.
Hence, there is a great need to study the inversion problem of DNN such that {\em semi-supervised} algorithms can be used, leveraging both labeled and unlabeled data for learning and inference. 
Limited progress has been made on semi-supervised learning algorithms for deep neural networks \cite{rasmus2015semi,salimans2016improved,patel2015probabilistic,nguyen2016semi} and today's methods suffer from a range of drawbacks, including training instability, lack of topology generalization, and computational complexity.
Most importantly, there exists no universal methodology to equip any given deep net with an inversion scheme.
In this section, we develop a universal methodology to invert a network allowing input reconstruction. This will allow for semi-supervised learning which can also be extended to unsupervised tasks with arbitrary DNN mappings.
Our approach simply relies on the derived inverse mapping strategy of a deep network allowing to add an additional term to the loss function. This extra term will guide the weight updates such that information contained in unlabeled data are incorporated to the network.
Our key insight is that the defined and general inverse function can be easily derived and computed; thus for unlabeled data points we can both compute and minimize the error between the input signal and the estimate provided by applying the inverse function to the network output without extra cost or change in the used model.
The simplicity of this approach, coupled with its universal applicability promise to significantly advance the purview of semi-supervised and unsupervised learning.
A series of experiments demonstrate that these modified networks have attain state-of-the-art performance in a range of semi-supervised learning tasks.

A major drawback to supervised learning is the need for a massive set of fully labeled training data. 
{\em Semi-supervised learning} relaxes this requirement by leaning $\Theta$ based on two datasets:  a fully labeled set $\mathcal{D}$ of $N$ training data pairs and a ''complementary'' unlabeled set $\mathcal{D}_u:=\{X_n,n=1,...,N_u\}$ of $N_u$ training inputs.
Unlabelled training data is useful for learning, because the unlabelled inputs provide information on the statistical distribution of the data and will help to guide the learning of $\Theta$ to classify the supervised dataset as well as characterize the unlabeled samples present in $\mathcal{D}_u$.
The lionshare of deep learning research has focused on supervised learning, because it has not been clear how to best incorporate unlabeled data points into the loss function to incorporate those unlabeled examples information in $f_\Theta$.  
However, there has been limited progress in a few directions which we now review.

When considering network inversion, standard approaches \cite{dua2000inversion} The only DNN models will reconstruction ability are based on autoencoders as \cite{ng2011sparse}. While more complex topologies have been used such as stacked convolutional autoencoder \cite{masci2011stacked} there exists two main drawbacks. Firstly, the difficulty to train complex models in a stable manner even when using per layer optimization. Secondly, the difficulty to leverage supervised information into the learning of the representation. If one considers the problem of semi-supervised learning with deep neural networks, different methods have been developed.
The improved {\em generative adversarial network} (GAN) technique \cite{salimans2016improved} couples two deep networks; a generative model that can create new signal samples on the fly and a discriminative network predicting the class of the labeled examples as well as the generated versus original nature of the input. Both neural nets are trained jointly as in typical GAN frameworks but the fact that the discriminator has to perform both tasks force it to incorporate the unlabeled signal information to distinguish them from the fake one generated by the generator. 
This enables the generator and the classifier to better model the data distribution by using the labeled examples to learn the class boundaries and the unlabeled examples to learn the distinction between ``true'' and ``fake'' signals.
The main drawbacks to this approach include training instability of GANS \cite{arjovsky2017wasserstein}, its lack of portability to time series, and high-resolution images (e.g., Imagenet\cite{deng2009imagenet}) which so far are out-of-reach of GANs\cite{salimans2016improved} as well as the extra time and memory cost of training two deep networks.
The {\em semi-supervised with ladder network} approach \cite{rasmus2015semi} employs a per-layer denoising reconstruction loss, which enables the system to be viewed as a stacked denoising autoencoder which is a standard and until now only way to tackle unsupervised tasks.
By forcing the last denoising autoencoder to output an encoding describing the class distribution of the input, this deep unsupervised model is turned into a semi-supervised model. 
The main drawback from this methods relies in the lack of a clear path to generalize it to other network topologies, such as recurrent network or residual networks. Also, the per layer ''greedy'' reconstruction loss might be too restrictive unless correctly weighted pushing the need for a precise and large cross-validation of hyper-parameters.
The probabilistic formulation of deep convolutional nets presented in \cite{patel2015probabilistic,nguyen2016semi} natively supports semi-supervised learning. 
The main drawback of this approach lies in its probabilistic nature requiring activation function to be ReLU and the DNN topology to be a DCN as well as inheriting standard difficulties of probabilistic graphical models in the context of large scale high dimensional dataset.

We propose a simple way to invert any piecewise differentiable mapping including DNNs. We provide the inverse mapping which requires no change in the current models and is computationally optimal as input reconstruction results in a backward pass as is used today for weight updates via backpropagation. This efficiency coupled with grounded mathematical motivation highlighting the need to reconstruct makes the approach a necessary step to improve DNNs performances. While tremendous applications would leverage this generalized inversion scheme, we will focus here one demonstrating the benefits of input reconstruction as a network regularized and provide semi-supervised experiments where we are able to reach state-of-the-art results with different neural network topologies. 
Highlighting that any supervised model can be kept as they are while simply changing the loss function should bring the proposed method to any known or to be known model in deep learning.

As demonstrated in previous sections, DNNs can be considered as composition of linear splines or be closely approximated as such. As a result, DNNs can be rewritten as a linear spline of the form
\begin{equation}
f_\Theta(x) = A[x]X_n+b[x],
\end{equation}
where we denote by $f_\Theta$ a general DNN, $x$ a generic input, $A[x],b[x]$ the spline parameters conditioned on the input region. Based on this interpretation, DNN can be considered as template matching machines where $A[x]$ plays the role of an input adaptive template. To illustrate this point we provide for two common topologies the exact input-output mappings. For a standard deep convolutional neural network (DCN) with succession of convolutions, nonlinearities and pooling, one had
\begin{empheq}[box=\fbox]{align}
    \bz_{CNN}^{(L)}(x) =& \underbrace{W^{(L)} \prod_{\ell=L-1}^1A_{\rho}^{(\ell)}A^{(\ell)}_{\sigma} \bC^{(\ell)}}_{\text{Template Matching}}x+\underbrace{W^{(L)}\sum_{\ell=1}^{L-1}\left(\prod_{j=L-1}^{\ell+1}A_{\rho}^{(j)}A^{(j)}_{\sigma} \bC^{(j)}\right)\left(A_{\rho}^{(\ell)}A^{(\ell)}_{\sigma} b^{(\ell)}\right)+b^{(L)}}_{\text{Bias}}.
\end{empheq}
where $\bz^{(\ell)}$ represents the latent representation at layer $\ell$. Similarly for a deep residual network one has
\begin{equation}
    \boxed{\bz_{RES}^{(L)}(x) = \underbrace{W^{(L)}\left[ \prod_{\ell=L-1}^1\left(A_{\sigma,in}^{(\ell)}\bC_{in}^{(\ell)}+ \bC^{(\ell)}_{out} \right) \right]}_{\text{Template Matching}}x+\underbrace{\sum_{\ell=L-1}^1\left(\prod_{i=L-1}^{\ell+1}( A_{\sigma,in}^{(\ell)}\bC_{in}^{(\ell)}+ \bC^{(\ell)}_{out})\right)\left( A_{\sigma,in}^{(\ell)}b_{in}^{(\ell)}+b_{out}^{(\ell)} \right)+b^{(L)}}_{\text{Bias}}}.
\end{equation}

Based on those findings, and by imposing a simple L2 norm upper bound on the templates, it has been shown that optimal DNNs should be able to generate templates proportional to their input, positively for the belonging class and negatively for the others. 

\begin{theorem}
In the case where all inputs have identity norm $||X_n||=1,\forall n$ and assuming all templates denoted by $A[X_n]_c,c=1,\dots,C$ have a norm constraints as $\sum_{c=1}^C ||A[X_n]_c||^2\leq K,\forall X_n$ then the unique globally optimal templates are
 \begin{equation}
     A^*[X_n]_c=\left\{ \begin{array}{l}
          \sqrt{\frac{C-1}{C}K}X_n,\iff c=Y_n\\
          -\sqrt{\frac{K}{C(C-1)}}X_n,\text{ else}
     \end{array}\right.
 \end{equation}
 \end{theorem}

As a result, We now leverage the analytical optimal DNN solution to demonstrate that reconstruction is indeed implied by such an optimum.

\subsubsection{Optimal DNN Leads to Input Reconstruction}
The study of reconstruction in the context of DNN differs from standard signal processing. In fact, for any given input, the number of templates is defined as the number of classes $C$ which is in general much smaller than the number of input dimensions $D$. Hence, approaches based on the study of orthogonal or over-complete basis can not be applied. Also, composition of mappings is hardly studied in this context unless dealing with deep NMF (CITE) types of approaches. On the other hand, the templates or atoms are adapted to the input as opposed to NMF where they are optimized explicitly. Firstly, we define the inverse of a DNN as
\begin{align}
f_\Theta^{-1}(X_n)=&A[X_n]^T(A[X_n]X_n+b[X_n])\nonumber \\
=&\sum_{c=1}^C (<A[X_n]_c,X_n>+b[X_n]_c)A[X_n]_c.
\end{align}
\begin{theorem}
In the case of no or negligible  $A[X_n]^Tb[X_n]$, optimal templates defined in the previous Theorem lead to exact reconstruction as we have by definition
\begin{align*}
\sum_{c=1}^C \langle A[X_n]_c,x\rangle A[X_n]_c=&(C-1) K\frac{1}{C(C-1)}||X_n||^2X_n +K\frac{C-1}{C}||X_n||^2X_n=X_n
\end{align*}
Hence optimal templates imply perfect reconstruction.
\end{theorem}
However, exact reconstruction is in practice not optimal as from all the content present in $X_n$ only a subpart is sufficient for the task at hand, there also is presence of noise and so on. In fact, the quantity $A[X_n]^Tb[X_n]$ is negligible but not null representing this fact. Hence if we now consider the problem of reconstructing a noisy input, we can bridge this scheme to standard thresholding and denoising scheme such as wavelet thresholding. In particular we study now the case of Relu or LReLU activations and mean or max pooling. The ReLU with bias is equivalent to asymmetric soft thresholding. Hence we can see the inversion of DNNs are equivalent to a composition of soft-thresholding based denoising operators. With minimization of the reconstruction error, one then has an adaptive filter-bank able to span the dataset. While not being sufficient for generalization it is necessary as absence of templates adapted to an input is synonym of false induced representation.
We now present a particular application of the derived reconstruction loss: semi-supervised.

\subsubsection{Boundary Inversion Method, Implementation and Loss Function}

We briefly describe how to apply the proposed strategy to a given task with arbitrary DNN. As exposed earlier all the needed changes happen in the loss functions by adding extra terms. As a result if automatic differentiation is used as in Theano\cite{bergstra2010theano}, TensorFlow\cite{abadi2016tensorflow} for example then it is sufficient to change the loss function and all the updates will be adapted via the change in the gradients for each of the parameters.
The great efficiency of this inversion schemes is due to the following.
As we have seen in the previous section, any deep network can be rewritten as a linear mapping. This leads to a simple derivation of a network inverse defined as $f^{-1}$ that will be used to derive our unsupervised and semi-supervised loss function via
\begin{align}
    f^{-1}(x,A[x],b[x])=&A[x]^T\Big(A[x]x+b[x]\Big)\nonumber \\
                       =&A[x]^Tf(x;\Theta)\nonumber \\
                       =&\frac{d f(x;\Theta)}{dx}^Tf(x;\Theta).
\end{align}
The main efficiency argument thus comes from the fact that
\begin{equation}
    A[x]=\frac{d f(x;\Theta)}{dx},
\end{equation}
allowing to very efficiently compute this matrix on any deep networks via differentiation as it would be done to back-propagate a gradient for example.

Interestingly for neural networks and many common frameworks s.a. wavelet thresholding, PCA,\dots, $\mathcal{E}$ is considered as the reconstruction error as $(\frac{df(x)}{dx})^Tf(x)$ is the definition of the inverse transform.
In particular and for illustration purposes, we present in the table below some common frameworks for which $\mathcal{E}$ represents exactly the reconstruction loss and thus Eq.\ref{inverse} is considered as the inverse transform.
\begin{table}[h]
\centering
\caption{Examples of frameworks with similar inverse transform definition.}
\scriptsize 
\begin{tabular}{|c|c|c|c|c|}\hline
     & $\alpha_i$ & $f(x)_i$& loss  \\ \hline
Sparse Coding     & Learned    & $\frac{<x,W_i>}{||W_i||^2}$&  $||x-\sum_i\alpha_i \frac{d f(x)_i}{dx}||^2+ \lambda ||\alpha||_1$ \\ \hline
NMF  & Learned & $<x,W_i>$&  $||x-\sum_i\alpha_i \frac{d f(x)_i}{dx}||^2$ s.t. $J_{f}(x) \geq 0$ \\ \hline
PCA  & $f(x)_i$ & $<x,W_i>$&  $||x-\sum_i\alpha_i \frac{d f(x)_i}{dx}||^2$ s.t. $J_{f}(x)$ orthonormal \\ \hline
Soft Wavelet Thresh. & $f(x)_i$ & $\max\Big(|<x,W_i>|-b_i,0\Big)sign(<x,W_i>)$ & $||x-\sum_i \alpha_i\frac{d f(x)_i}{dx}||^2$ \\ \hline     
Hard Wavelet Thresh. & $f(x)_i$ & $1_{|<x,W_i>|-b_i>0}<x,W_i>$ & $||x-\sum_i \alpha_i\frac{d f(x)_i}{dx}||^2$ \\ \hline     
Best Basis (WTA) & $f(x)_i$ & $1_{i=\argmax_i \frac{<x,W_i>}{||W_i||^2}}<x,W_i>$ &$||x-\sum_i \alpha_i\frac{d f(x)_i}{dx}||^2$ \\ \hline
k-NN &  $1$  & $1_{i=\argmax <x,W_i>-||W_i||^2/2}<x,W_i>$ & $||x-\sum_i\alpha_i\frac{df(x)_i}{dx}||^2$\\ \hline
\end{tabular}
\end{table}
\normalsize
This inversion scheme is often seen as an ill-posed problem. In fact, for the ReLU case for example, given an output activation, the negative values that have been filtered can not be reconstructed. However with the proposed method, the implied reconstruction is $0$. This corresponds to reconstruction the input on the boundary of the region defined by the current ReLU activation. As one reconstruct with values further away from $0$ into the negative side, as the reconstruction goes away from the region. Hence our propose scheme can be seen as an optimistic case where the input was assumed to lie on the boundaries of the region.

We now describe how to incorporate this loss for semi-supervised and unsupervised learning.
We first define the $R$ reconstruction loss as
\begin{equation}
    R(X_n) = ||(\frac{df_\Theta(X_n)}{dX_n})^Tf_\Theta(X_n)-X_n||^2.
\end{equation}
While we use the mean squared error, any other reconstruction loss which is differentiable can be used s.a. cosine similarity.
We also introduce an additional ''specialization'' loss defined as the Shannon entropy of the prediction:
\begin{align}
E(\hat{y}(X_n))=-\sum_{c=1}^C\hat{y}(X_n)_c\log(\hat{y}(X_n)_c),
\end{align}
pushing the output distribution to have low entropy when minimized. The need for this loss is to make the unlabeled prediction with low-entropy a.k.a predicting a one-hot representation.
As a result, we define our complete loss function as the combination of the standard cross entropy loss for labeled data denoted by $L_{CE}(Y_n,\hat{y}(X_n))$, the reconstruction loss and entropy loss as
\begin{equation}\label{eqss}
    \mathcal{L}(X_n,Y_n)=\alpha L_{CE}(Y_n,\hat{y}(X_n))1_{\{Y_n \not = \emptyset\}}+(1-\alpha)[\beta R(X_n)+(1-\beta)E(X_n)], \alpha,\beta \in [0,1]^2,
\end{equation}
The parameters $\alpha,\beta$ are introduced to form a convex combination of the losses, $2$ of them being unsupervised, with $\alpha$ controlling the ratio between supervised and unsupervised loss and $\beta$ the ration between the two unsupervised losses.

One can also use the presented loss to perform unsupervised tasks and clustering. In fact, we can see that by setting $\alpha=0$ we are in a fully unsupervised framework, and, depending on the value of $\beta$, pushing the mapping $f_\Theta$ to produce a low-entropy, clustered, representation or rather being unconstrained and simply producing optimal reconstruction. Even in a fully unsupervised and reconstruction case $(\alpha=0,\beta=1)$ the proposed framework is not similar to a deep-autoencoder for two main reasons. The first one lies in the fact that there is no greedy (per layer) reconstruction loss, only the final output is considered in the reconstruction loss. Secondly, while in both case there is parameter sharing, in our case there is also ''activation'' sharing which corresponds to the states (spline) that were used in the forward pass that will also be used for the backward one. In a deep autoencoder, the backward activation states are induced by the backward projection and will most likely not be equal to the forward ones.

\subsubsection{Semi-Sup Experiments for State-of-the-art Performances across Topologies}

We now present results of the approach on a semi-supervised tasks on the MNIST dataset where we are able to obtain state-of-the-art performances with different topologies showing the ability of the method to generalize to any topology as well as being competitive. MNIST is made of $70000$ grayscale images of shape $28 \times 28$ which is split into a training set of $60000$ images and a test set of $10000$ images.
We present results for the case with $N_L=50$ which represents the number of samples from the training set which are labeled. All the others are unlabeled.
In addition, $10$ different topologies are tested to show the portability of the approach. The DNNs architecture details as well as training procedures are detailed below.
Furthermore, we tested $\alpha \in \{0.7,0.5,0.3\}$ and found that better results were obtained on average with $\alpha = 0.7$ and thus present below all results  with $\alpha=0.7,\beta=0.5$.
Running the proposed semi-supervised scheme on MNIST led to the results presented in the table below.
We are able to obtain better results than all the standard benchmarks but one. In particular, with $N_L=50$ we are also able to outperform most method which use $100$ labels.
We used Theano and Lasagne libraries. Details on the learning procedure and the used topologies are provided in the next section, the code for reproducible results is in the attached materials. The column Sup1000 for MNIST corresponds to the accuracy after training of DNN using only supervised loss on $1000$ data, showing the great impact of the proposed solution.
\begin{table}[htbp]
\centering
\begin{tabular}{|c|c|c|c|} \hline
$\mathbf{N_L} $              & $\mathbf{50}$   & $\mathbf{100}$ & $\mathbf{Sup 1000}$\\ \hline
SmallCNNmean      &  \textbf{99.07},$(0.7,0.2)$    &      &$94.9$\\ \hline
SmallCNNmax        & $98.63$,$(0.7,0.2)$        &      &$95.0$\\ \hline
SmallUCNN        & $98.85$,$(0.5,0.2)$      &      &$96.0$\\ \hline
LargeCNNmean      &  $98.63$,$(0.6,0.5)$    &      &$94.7$\\ \hline
LargeCNNmax        & $98.79$,$(0.7,0.5)$       &      &$94.8$\\ \hline
LargeUCNN        & $98.23$,$(0.5,0.5)$      &      &$96.1$\\ \hline
Resnet2-32mean           &  \textbf{99.11},$(0.7,0.2)$       &       &$95.5$\\
\hline
Resnet2-32max           &  \textbf{99.14},$(0.7,0.2)$       &       &$94.9$\\ 
\hline
UResnet2-32           &  $98.84$,$(0.7,0.2)$   &       &$95.6$\\ 
\hline
Resnet3-16mean           &  $98.67$,$(0.7,0.2)$       &       &$95.4$\\
\hline
Resnet3-16max           &  $98.56$,$(0.5,0.5)$       &       &$94.8$\\  \hline
UResnet3-16           &  $98.7$,$(0.7,0.2)$       &       &$95.5$\\ 
\hline \hline
Improved GAN\cite{salimans2016improved}        & $97.79 \pm 1.36$ &$99.07 \pm 0.065$  &\\ \hline
Auxiliary Deep Generative Model\cite{maaloe2016auxiliary} & - &$99.04$ & \\ \hline
LadderNetwork\cite{rasmus2015semi} & - & $98.94 \pm 0.37$& \\ \hline
Skip Deep Generative Model\cite{maaloe2016auxiliary} & - & $98.68$ & \\ \hline
Virtual Adversarial\cite{miyato2015distributional} & - & $97.88$ & \\ \hline
catGAN \cite{springenberg2015unsupervised} & - &$98.61 \pm 0.28$& \\ \hline
DGN \cite{kingma2014semi} & - & $96.67 \pm 0.14$& \\ \hline \hline
DRMM\cite{nguyen2016semi} &$78.27$&$86.59$& \\ \hline
DRMM +NN penalty & $77.9$ &$87.72$& \\ \hline
DRMM+KL penalty & $97.54$ &$98.64$& \\ \hline
DRMM +KL+NN penalties       & \textbf{99.09} & $99.43$ &  \\ \hline
\end{tabular}\caption{Test Error on MNIST for $50$ and $100$ labeled examples for the used networks as well as comparison with other methods. The column Sup1000 demonstrates the raw performance of the same networks trained only with the supervised loss with $1000$ labels.}
\end{table}

In addition we present in Fig. \ref{fig:reconstruction} the reconstruction for the case $N_L=50$ as well as the test set accuracy in Fig. \ref{fig:testerror}. 
\begin{figure}[!htb]
    \centering
    \begin{minipage}{.5\textwidth}
        \centering
        \includegraphics[width=0.9\linewidth]{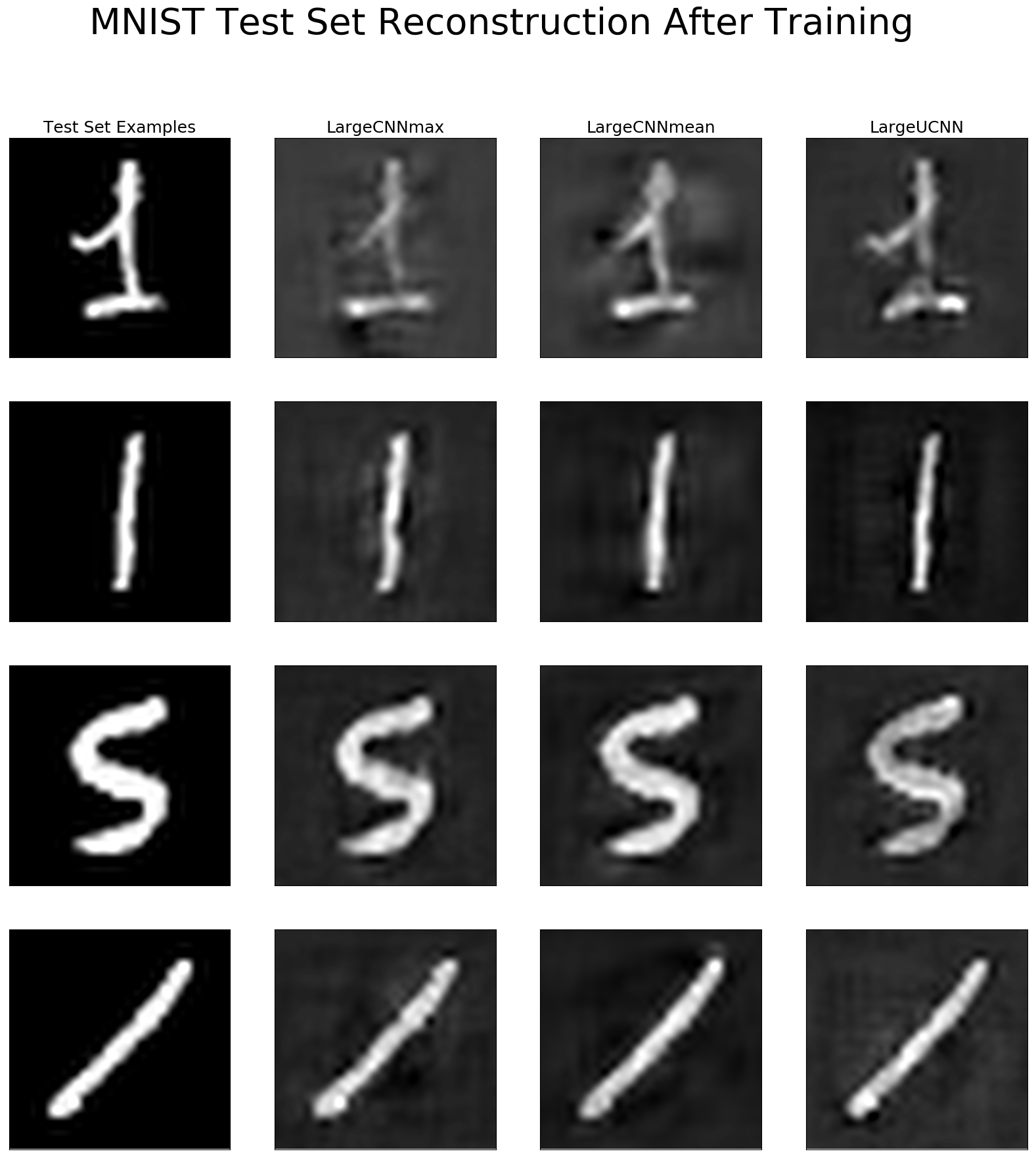}
    \end{minipage}%
    \begin{minipage}{0.5\textwidth}
        \centering
        \includegraphics[width=0.9\linewidth]{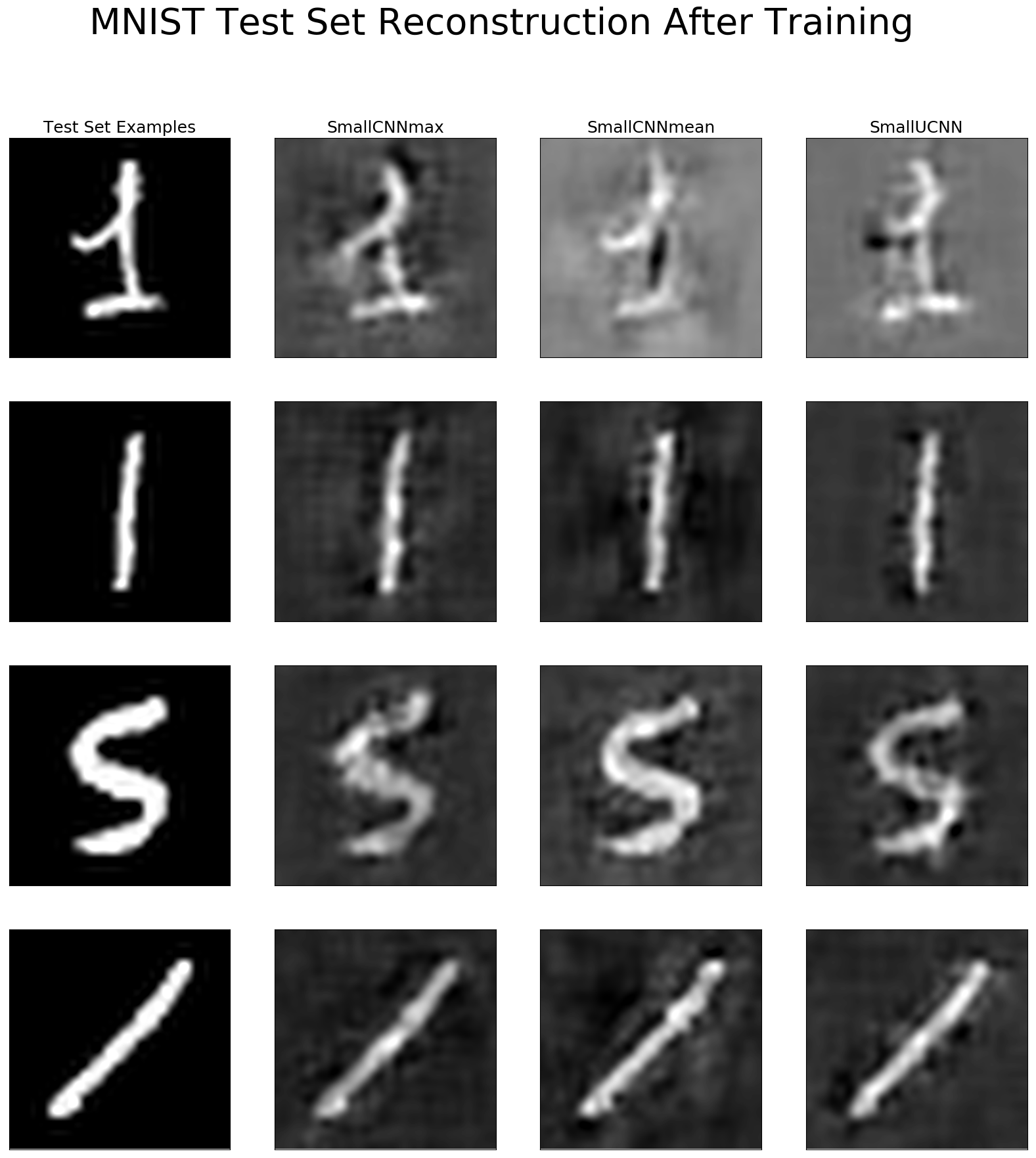}
        \label{fig:prob1_6_1}
    \end{minipage}
    \\
    \begin{minipage}{.5\textwidth}
        \centering
        \includegraphics[width=0.9\linewidth]{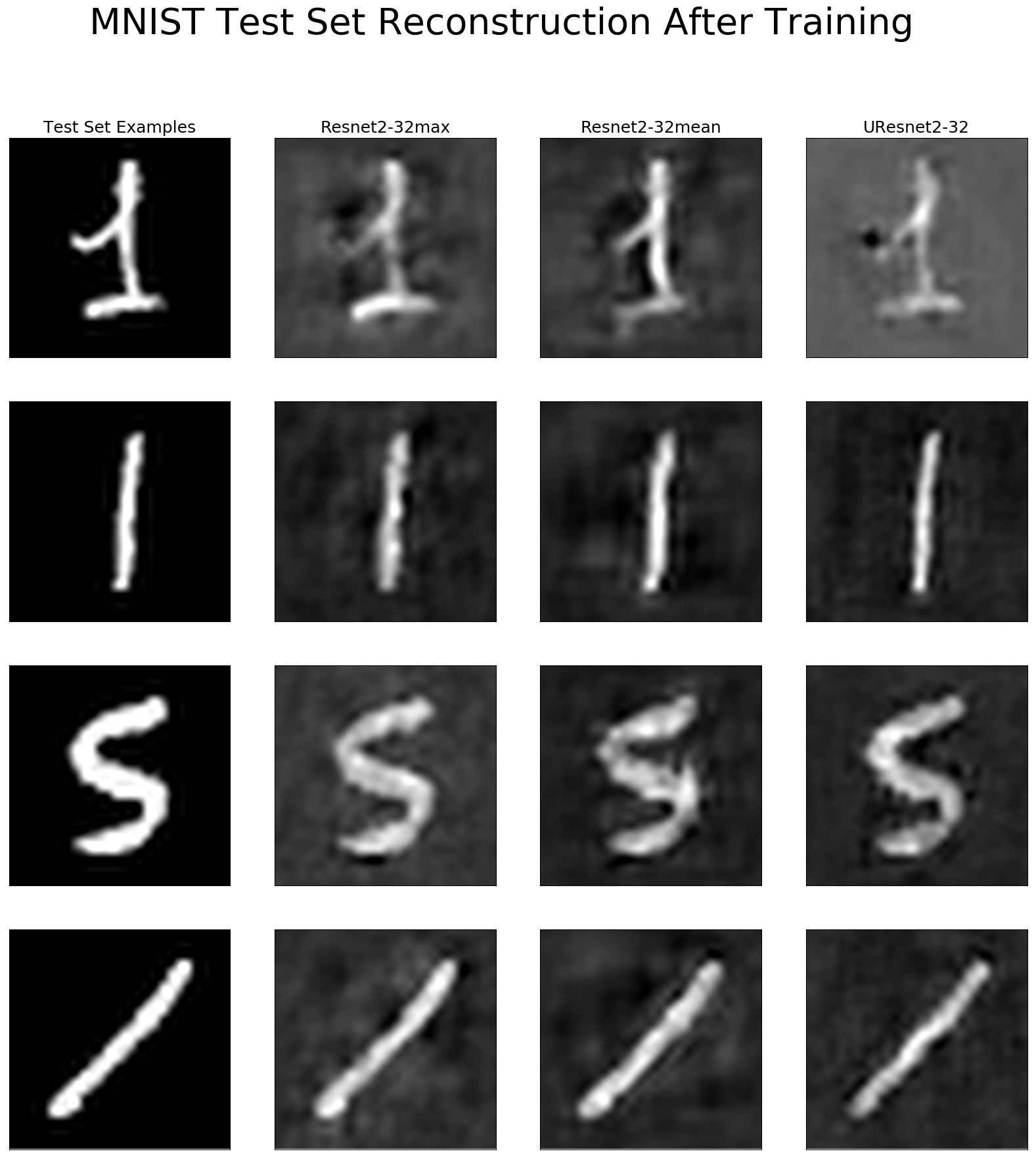}
    \end{minipage}%
    \begin{minipage}{0.5\textwidth}
        \centering
        \includegraphics[width=0.9\linewidth]{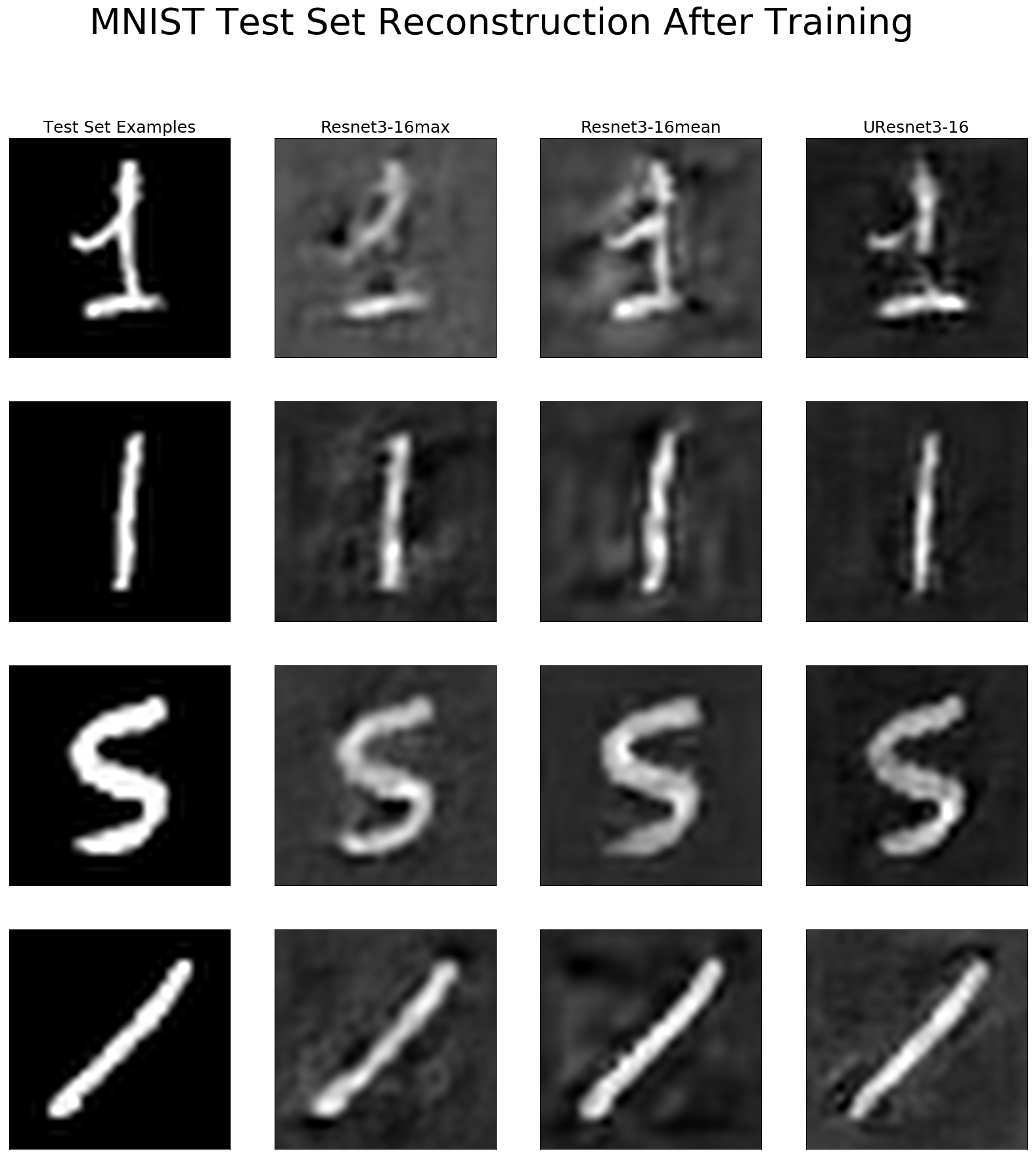}
        \label{fig:prob1_6_1}
    \end{minipage}
    \caption{Reconstruction of the studied DNN models for $4$ test set samples. In each subplot the columns from left to right correspond to: the original image, mean-pooling reconstruction, max-pooling reconstruction, inhibitor connections. Each group of subplot represents a specific topology being in clockwise order: LargeCNN, SmallCNN, Resnet2-32 and Resnet3-16.}\label{fig:reconstruction}
\end{figure}

\begin{figure}[!htb]
    \centering
    \begin{minipage}{.99\textwidth}
        \centering
        \includegraphics[width=0.95\linewidth]{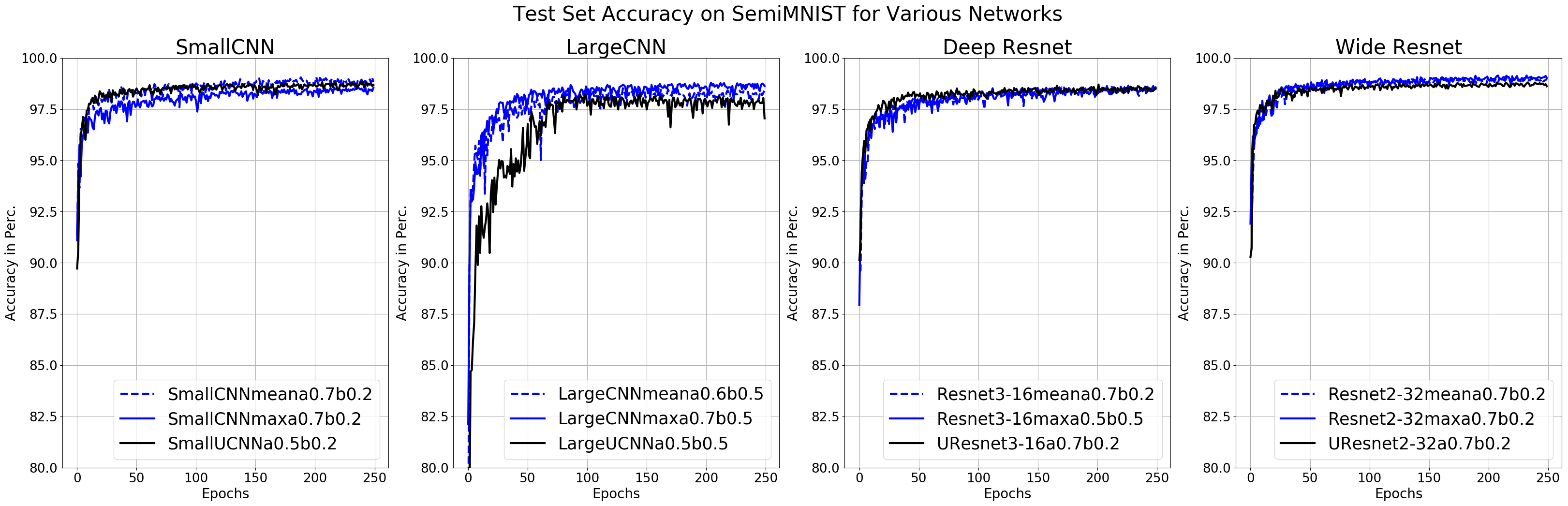}
    \end{minipage}%
    \caption{Er present in in this figures the test set accuracies during learning for all the studied topologies. In black are the DINN, and blue mean and max pooling. From left to right, the topologies are: SmallCNN, LargeCNN, Resnet3-16, Resnet2-32.}\label{fig:testerror}
\end{figure}
\subsubsection{Extensions}

Among the possible extensions, one can develop the reconstruction loss into a per layer reconstruction loss. Doing so, there is the possibility to weight each layer penalty bringing flexibility as well as meaningful reconstruction. Let define the per layer loss as 
\begin{equation}
    \mathcal{L}(X_n,Y_n)=\alpha L_{CE}(Y_n,\hat{y}(X_n))1_{\{Y_n \not = \emptyset\}}+\beta E(X_n) +\sum_{\ell=0}^{L-1} \gamma^{(\ell)}R^{(\ell)}(X_n),
\end{equation}
with 
\begin{equation}
    R^{(\ell)}(X_n) = ||(\frac{df_\Theta(X_n)}{d\bz^{(\ell)}(X_n)})^Tf_\Theta(X_n)-\bz^{(\ell)}(X_n)||^2.
\end{equation}
Doing so, one can adopt a strategy in favor of high reconstruction objective for inner layers, close to the final latent representation $\bz^{(L)}$ and lessen the reconstruction cost for layers closer to the input $X_n$. In fact, inputs of standard dataset are usually noisy, with background, and the object of interest only contains a small energy w.r.t. the total energy of $X_n$. 
Another extension would be to update the weighting while performing learning. Hence, if we denote by $t$ the position in time such as the current epoch or batch, we now have the previous loss becoming
\begin{equation}
    \mathcal{L}(X_n,Y_n;\Theta)=\alpha(t) L_{CE}(Y_n,\hat{y}(X_n))1_{\{Y_n \not = \emptyset\}}+\beta(t) E(X_n) +\sum_{\ell=0}^{L-1} \gamma^{(\ell)}(t)R^{(\ell)}(X_n).
\end{equation}
One approach would be to impose some deterministic policy based on heuristic such as favoring reconstruction at the beginning to then switch to classification and entropy minimization. Finer approaches could rely on an explicit optimization schemes for those coefficients. One way to perform this, would be to optimize the loss weighting coefficients $\alpha,\beta,\gamma^{(\ell)}$ after each batch or epoch by backpropagation on the updates weights. Let define
\begin{align}
    \Theta(t+1)=\Theta(t) - \lambda\frac{d L(X_n,Y_n)}{d\Theta},
\end{align}
representing an generic iterative update based on a given policy such as gradient descent. One can thus adopt the following update strategy for the hyper-parameters as
\begin{align}
    \gamma^{(\ell)}(t+1) = \gamma^{(\ell)}(t)-\frac{d L(X_n,Y_n;\Theta(t+1))}{d\gamma^{(\ell)}(t)},
\end{align}
and so for all hyper-parameters. Another approach would be to use adversarial training to update those hyper-parameters where both update cooperate trying to accelerate learning.

%% file: theory.tex
\section{Extra Material}

\subsection{Spline Operator}
In this section we first review the literature on splines in order to ease the introduction of multivariate spline functions which will be the building block of the spline operators. We then discuss properties of the defined mathematical objects s.t. the next section on the rewriting of deep neural networks via spline operators become intuitive.
\subsubsection{Spline Functions[FINI]}\label{splineoperator}
Throughout this section, the main reference comes from the formidable monograph \cite{schoenberg1964interpolation} revisited ad extended in \cite{schumaker2007spline}. We first present standard univariate splines constructed via piecewise polynomial functions. Let's now present the beautiful development of spline functions.

We first define the space of {\em univariate polynomials of order} $m$ as
\begin{align}
    \mathcal{P}_m=&\{p:p(x) = \sum_{i=1}^mc_ix^{i-1},c_1,\dots,c_m,x\in \mathbb{R}\},\\
    =&{\rm span}\{1,x,\dots,x^{m-1}\},
\end{align}
note that the order $m$ is equal to the number of degrees of freedom, and is the degree of the polynomial minus one. 
A specific element of this set denoted by $p_m$ is thus fully determined by its specific coefficients $\mathbf{c}=(c_1,\dots,c_m) \in \mathbb{R}^m$.
\begin{defn}
we denote the polynomial of order $m$ with parameters $\mathbf{c} \in \mathbb{R}^m$ by
\begin{align}
    p_m[\textbf{c}]:\mathbb{R} &\rightarrow \mathbb{R}\nonumber\\
    x & \rightarrow \sum_{i=1}^mc_ix^{i-1}.
\end{align}
\end{defn}
While a polynomial $p_m[\textbf{c}]$ acts with the same set of parameters $\mathcal{c}$ across all its input space, it is possible to define a partition $\Omega$ of the input space on which ''local'' polynomials can act, hence depending on the region of an input, not necessarily the same polynomial with parameters $\mathcal{c}$ will be used for the mapping, this defines piecewise polynomials.
The set $\Omega$ being a collection of $R$ regions $\omega_r$ of the input space forms a partition of $[a,b] \subset \mathbb{R}$ s.t. 
\begin{align}
    &\Omega = \{\omega_r,i=1,\dots,R\},\nonumber\\
    &\omega_i\cap \omega_j =\emptyset ,\forall i\not = j,\nonumber\\
    &\cup_{r=1}^R \omega_r=[a,b],
\end{align}
where the partition can be extended to the real line by adding the elements $]-\infty,a[$ and $]b,\infty[$.
From this, we denote the space of piecewise polynomials of order $\textbf{m}=(m_1,\dots,m_R) \in \mathbb{Z}^R$ as
\begin{equation}
    \mathcal{P}\mathcal{P}_\textbf{m}(\Omega)=\{f : \forall x \in \omega_r,\exists p_r \in \mathcal{P}_{m_r}|f(x)=p_r(x),r=1,\dots,R\}.
\end{equation}
It is clear that the number of degrees of freedom is thus $m_r$ for region $r$ and thus for the piecewise polynomial it is equal to $\sum_{r=1}^R m_r$. We denote this collection of coefficients as $\mathbf{c}$ and the coefficients specific to region $r$ by $\mathbf{c}_r\in \mathbb{R}^{m_r}$.

\begin{figure}
    \centering
    \includegraphics[width=6in]{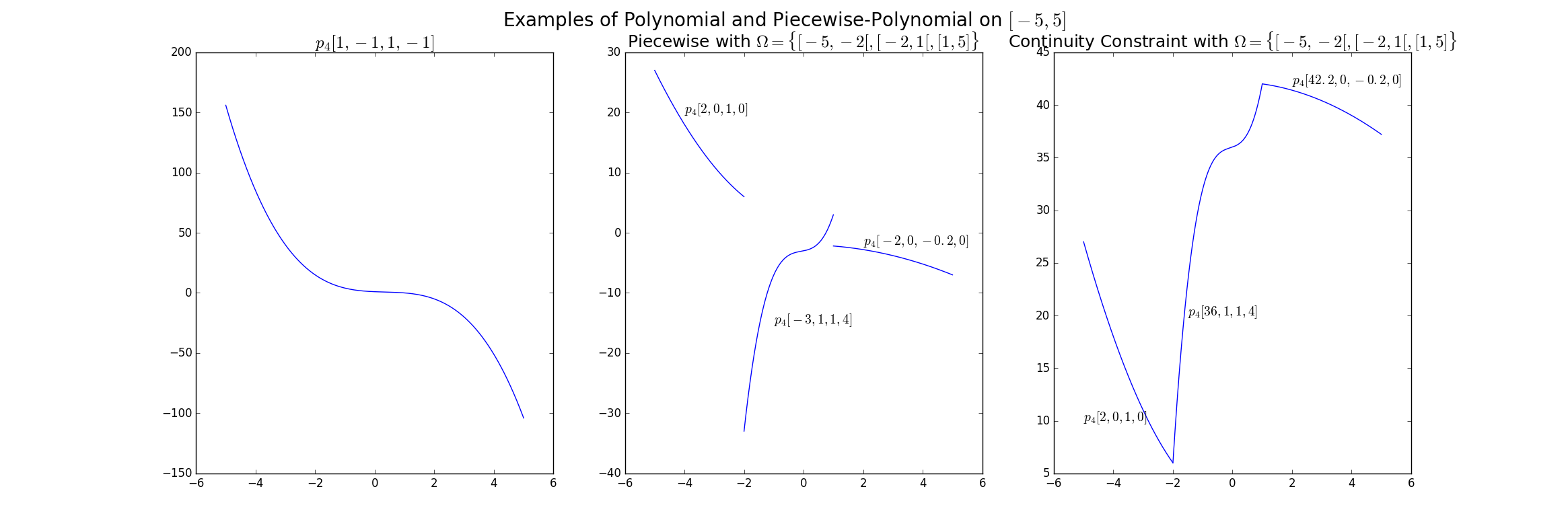}
    \caption{Illustrative examples of polynomial/piecewise-polynomial/continuous piecewise-polynomial.}
    \label{fig_p}
\end{figure}

\begin{defn}
We denote the piecewise polynomial of order $\textbf{m}$ with parameters $\mathbf{c}\in \mathbb{R}^{\sum_{r=1}^R m_r}$, a partition of $O \subset \mathbb{R}$ defined by $\Omega$ with $Card(\Omega)=R$ by
\begin{align}
    pp_\textbf{m}[\mathbf{c},\Omega]:\mathbb{R}&\rightarrow \mathbb{R}\nonumber\\
    x&\rightarrow \sum_{r=1}^Rp_{m_r}[\mathbf{c}_r](x)1_{\{x \in \omega_r\}}=\sum_{r=1}^R\left(\sum_{i=1}^{m_r}\mathbf{c}_{r,i}x^{i-1}\right)1_{\{x \in \omega_r\}}.
\end{align}
\end{defn}
This defines a mapping with which different polynomials can be ''activated'' depending on the input location and the partition, hence the piecewise property.

We can now define the space of splines as a subset of this piecewise polynomial space by adding a ''regularity'' constraint between polynomials of neighbouring/adjacent regions. If we denote by $x_r,r=0,\dots,R$ the ''knots'' of the partitions with $x_0=a$ and $x_R=b$ then the space of polynomial splines is defined as
\begin{align}
    S(\mathcal{P}_m,\mathcal{M},\Omega)=&\{f:\forall x \in \omega_r,\exists p_r \in \mathcal{P}_{m},f(x)=p_r(x),D^jp_{r}(x_r)=D^jp_{r+1}(x_r),\nonumber\\
    &j=0,\dots,m-1-M_i,r=1,\dots,R\}
\end{align}
with $\mathcal{M}=(M_1,\dots,M_K)$ the multiplicity vector which is the ''smoothness'' conditioning of the spline and $D^j$ the differential operator of order $j$. Two simple examples are presented below :
\begin{align*}
    &S(\mathcal{P}_m,(m,\dots,m),\Omega)=\mathcal{P}\mathcal{P}_{(m,\dots,m)}(\Omega),\\
    &S(\mathcal{P}_m,(1,\dots,1),\Omega)=\mathcal{P}\mathcal{P}_{(m,\dots,m)}(\Omega)\cap C^{m-2}([a,b]),
\end{align*}
the first is the least constraint spline while the latter is the most constraint one in term of boundary conditions.

\subsubsection{Multivariate spline functions[FINI]}

We present the multivariate polynomials and splines allowing to process multivariate inputs of dimension $K\geq 1$, with a slight difference form the literature\cite{schumaker2007spline,schoenberg1964interpolation,chui1988multivariate,de1983approximation} where we allow non-rectangular regions for general spaces of dimension $d$, as most of the development of irregular grids focus on $2/3$-dimensional spaces for PDE specific applications.
We thus omit here the introduction to the known tensor multivariate splines as they are constructed with rectangular regions and thus will not be used in the later sections. From this, all the tools will be made clear for us to present the next section which consists of ''adapting'' the spline and piecewise multivariate polynomial terminology and notations for use in deep learning via the development of the spline operator, for mappings going to $\mathbb{R}^K,K>1$.

\begin{defn}
Let first define the multivariate integer set as
\begin{equation}
    \mathbb{N}^d=\{\alpha: \alpha = (\alpha_1,\dots,\alpha_d),\alpha_i\in \mathbb{N},i=1,\dots,d\}.
\end{equation}
\end{defn}
Using this notation, we denote the space of multivariate polynomials , given $\Lambda \subset \mathbb{N}^d$ as
\begin{align}
    \mathcal{P}^d_\Lambda =& {\rm span}\{x^\lambda:\lambda \in \Lambda\},\\
    =&\{p : p(x)=\sum_{\lambda \in \Lambda}c_\lambda x^\lambda,c_\lambda \in \mathbb{R}\},
\end{align}
where we denoted $x^\lambda=\prod_{i=1}^dx_i^{\lambda_i}$ with $x_i$ the $i^{th}$ input dimension of $x$, we also denote by $\mathbf{c}$ the ordered collection of the $c_\lambda$. The collection $\Lambda$ thus holds all the possible configuration of power for each of the input dimension and their configuration, as each element $\lambda \in \Lambda$ defines uniquely a combination of some power of the input dimensions. For example in the $3$-dimensional setting, $\lambda=(2,1,0)$ leads to $x^\lambda=x_1^2x_2$.

\begin{defn}
We rewrite the multivariate polynomial acting on $\mathbb{R}^d$ with parameters $\mathbf{c}\in \mathbb{R}^{Card(\Lambda)}$ and order $\Lambda \subset \mathbb{Z}^d_+$ by
\begin{align}
    p_\Lambda^d[\mathbf{c}]:\mathbb{R}^d&\rightarrow \mathbb{R}\nonumber\\
    x&\rightarrow \sum_{\lambda \in \Lambda}c_\lambda x^\lambda.
\end{align}
\end{defn}
Note that a particular case occurs given a tuple $\mathbf{m}=(m_1,\dots,m_d)$ with the property that $\Lambda_\textbf{m} = \{\alpha : 0\leq \alpha_i\leq m_i,i=1,\dots,d\}=\otimes_{i=1}^{d} \{0,\dots,m_i\}$, we also denote by $\Lambda_m$ the case where $\textbf{m}=(m,\dots,m)$. For example with $\textbf{m}=(2,2)$ we have the basis functions $\mathcal{P}^d_{\Lambda_2}=\{1,x,y,xy\}$ which is a bilinear polynomial.
We now develop the piecewise version of the multivariate polynomial space.
Given an arbitrary partition $\Omega$ of $O\subset \mathbb{R}^d$ and corresponding $\Lambda_r \subset \mathbb{N}^d$ we define a piecewise multivariate polynomial
\begin{align}
    \mathcal{P}\mathcal{P}^d_\Lambda =\{f : \forall \omega \in \Omega,\forall x \in \omega, \exists p \in \mathcal{P}^d_{\Lambda_r}: f(x)=p(x)\},
\end{align}
\begin{defn}
We rewrite the piecewise multivariate polynomial acting on $O\subset \mathbb{R}^d$ with $\Omega$ a partition of $O$, corresponding $\Lambda_r \subset \mathbb{Z}^d_+$ and with parameters $\mathbf{c}\in \mathbb{R}^{\sum_{r=1}^R Card(\Lambda_r)}$ by
\begin{align}
    pp_\Lambda^d[\mathbf{c},\Omega]:\mathbb{R}^d&\rightarrow \mathbb{R}\nonumber \\ 
    x&\rightarrow \sum_{r=1}^R\left(\sum_{\lambda \in \Lambda_r} c_{r,\lambda} x^\lambda\right)1_{\{x \in \omega_r\}}.
\end{align}
\end{defn}
From this set of piecewise-polynomial functional, we can now define the space of splines by adding a smoothness constraints between neighboring regions. As in the univariate case we introduce $\mathcal{M}$ the tuple of regularization coefficients forcing for each neighboring regions to have piecewise polynomials with same first derivatives, up to the order specified by the $\mathcal{M}$ entry corresponding to it. As our work focus on two simple cases of regularization, we present the most and least constraint multivariate splines, with respectively $\mathcal{M}=(m,\dots,m):=\mathcal{M}_m$ and $\mathcal{M}=(1,\dots,1):=\mathcal{M}_1$ as
\begin{align}
    &\mathcal{S}(\mathcal{P}^d_{\Lambda_m};\Omega;\mathcal{M}_1)=\mathcal{P}\mathcal{P}^d_{\Lambda_m}(\Omega)\cap \mathbb{C}^{m-2}(O),\\
    &\mathcal{S}(\mathcal{P}^d_{\Lambda_m};\Omega;\mathcal{M}_m)=\mathcal{P}\mathcal{P}^d_{\Lambda_m}(\Omega).
\end{align}
where we remind that $\Lambda_m=\otimes_{i=1}^d \{0,\dots,m\}$.
Since we have $\mathcal{S}(\mathcal{P}^d_{\Lambda_m};\Omega;\mathcal{M}_1) \subset \mathcal{S}(\mathcal{P}^d_{\Lambda_m};\Omega;\mathcal{M}_m)$ we now present the general formulation and results for the latter as other cases are ''restricted'' cases than can be derived from it, we thus now denote a multivariate spline simply by $\mathcal{S}(\mathcal{P}^d_{\Lambda_m};\Omega)$. Given a spline $s^d \in \mathcal{S}(\mathcal{P}^d_{\Lambda_m};\Omega)$ we denote by $\mathbf{c}$ the parameters of the splines, namely, the coefficients of the $m$-order polynomials for each region. 

\begin{defn}
We denote the multivariate spline for the case $s \in \mathcal{S}(\mathcal{P}^d_{\Lambda_m};\Omega)$ with parameters $\mathbf{c} \in \mathbb{R}^{m \times Card(\Omega)}$ given a partition $\Omega$ of $O\subset \mathbb{R}^d$ by
\begin{align}
    s[\mathbf{c},\mathcal{P}^d_{\Lambda_m},\Omega]:\mathbb{R}^d&\rightarrow \mathbb{R}\\
    x &\rightarrow s[\mathbf{c},\mathcal{P}_{\Lambda_m}^d,\Omega](x)=pp_{\Lambda_m}^d[\mathbf{c},\Omega](x)
\end{align}
where this definition is based on the use of $\mathcal{M}_m$ as defined above.
\end{defn}
If we denote by $\mathcal{I}[\Omega_k][(x)$ the index of the region in which $x$ belongs according to the $\Omega_k$ partitioning as
 \begin{align}
     \mathcal{I}[\Omega_k]:\mathbb{R}^d &\rightarrow \{1,\dots,Card(\Omega_k)\}\\
     x&\rightarrow \sum_{i=1}^{Card(\Omega_k)}i*1_{\{x \in \omega_{k,i}\}},
 \end{align}
 where $\omega_{k,i}$ is the $i^{th}$ region of $\Omega_k$. It is clear that given an input $x$. We can thus simplify notations by introducing this region selection operator coupled with the multivariate spline now defined as
\begin{align}
    s[\mathbf{c},\mathcal{P}^d_{\Lambda_m},\Omega]:\mathbb{R}^d&\rightarrow \mathbb{R}\\
    x &\rightarrow s[\mathbf{c},\mathcal{P}^d_{\Lambda_m}]_{\mathcal{I}[\Omega]}(x).
\end{align}

\subsubsection*{Linear Multivariate Splines[FINI]}
We describe briefly a special case in which all the local mappings are linear. It will become of importance in the next section for the introduction of affine spline operators and neural networks.
We first consider a special multivariate polynomial defined with total order $m$ and denoted by $\mathcal{P}^d_{|m|}$ where $|m|$ denotes the total order property as opposed to the standard polynomials of order $m$. It is explicitly defined as the space
\begin{defn}
\begin{equation}
    \mathcal{P}^d_{|m|}={\rm span}\{\prod_{i=1}^dx_i^{\alpha_i},\alpha \in \mathbb{Z}_+^d,|\alpha|<|m|\},
\end{equation}
with $m \in \mathbb{N}$ and $|\alpha|=\sum_{i=1}^d\alpha_i$.
\end{defn} 
As an example, if we consider the $3D$ space with $x=(x_1,x_2,x_3)$ the space of linear polynomials as
\begin{align*}
    \mathcal{P}^3_{|2|}&={\rm span}\{1,x_1,x_2,x_3\},\text{$3D$ polynomial of total order $2$}
    \end{align*}
    as opposed to the nonlinear but multi-linear case of
    \begin{align*}
\mathcal{P}^3_{2}&={\rm span}\{1,x_1,x_2,x_3,x_1x_2,x_1x_3,x_2x_3\},\text{$3D$ polynomial of order $2$}.
\end{align*}
As we will focus on linear cases in the remaining of the study we now simplify notations for the latter.
\begin{defn}
A linear $K$ dimensional polynomial $p\in \mathcal{P}^K_{|2|}$ with coefficients $\textbf{a} \in \mathbb{R}^K,b\in \mathbb{R}$ is denoted as
\begin{equation}
    p^K[\textbf{a},b](x)=\langle \textbf{a},x\rangle+b.
\end{equation}
\end{defn}
From this liner polynomial we define the linear piecewise polynomials according to the last section notations. Given a partition of $O \subset \mathbb{R}^d$ defined as $\Omega=\{\omega_i,i=1,\dots,I\}$ s.t. $\cup_{i=1}^I\omega_i = O$ and $\omega_i \cap \omega_j=\emptyset,\forall i\not = j$, the space of linear piecewise polynomials is defined as
\begin{equation}
    \mathcal{P}\mathcal{P}^d_{|2|}(\Omega)=\{f : \forall x \in \omega_i,\exists p_i \in \mathcal{P}^d_{|2|}|f(x)=p_i(x),i=1,\dots,K\}.
\end{equation}
From this, splines are defined in a similar fashion as in the last section where we have
\begin{align}
    &\mathcal{S}(\mathcal{P}^d_{|2|};\Omega;\mathcal{M}_1)=\mathcal{P}\mathcal{P}^d_{|2|}(\Omega)\cap \mathbb{C}^{m-2}(O),\\
    &\mathcal{S}(\mathcal{P}^d_{|2|};\Omega;\mathcal{M}_2)=\mathcal{P}\mathcal{P}^d_{|2|}(\Omega).
\end{align}
We now focus for the case of total order $2$ and we omit the constraint $\mathcal{M}_m$ as we present the general case, any other one is a restriction of the coefficients to fulfill the smoothness boundary condition.
\begin{defn}
We denote multivariate linear splines the case with total order $|2|$ denoted by $\mathcal{S}(\mathcal{P}^d_{|2|};\Omega)$ given the partition $\Omega$ of $O\subset \mathbb{R}^d$ with $Card(\Omega)=R$ and with parameters $(\textbf{a}_r,b_r)\in \mathbb{R}^d\times \mathbb{R},r=1,\dots,R$ by using the last proposition as
\begin{align}
    s[(\textbf{a}_r,b_r)_{r=1}^R;\mathcal{P}^d_{|2|},\Omega](x)=&\sum_{r=1}^R \left(\langle \textbf{a}_r,x\rangle + b_r\right)1_{\{x \in \omega_r\}}\\
    =&\textbf{a}[x]^Tx+b[x],
\end{align}
where $a_r$ represent the slope and $b_r$ the intercept for each region. The input dependant selection is abbreviated via
\begin{equation}
    \textbf{a}[x]=\textbf{a}_{\mathcal{I}[\Omega](x)} \text{ and }b[x]=b_{\mathcal{I}[\Omega](x)}.
\end{equation}
\end{defn}

From now on, the term multivariate is dropped as the mappings will be explicit.
We define a \textbf{{\em local} linear spline} function as a special case where the support of $\textbf{a}_r$ is constrained, in the sense that some dimensions are constraint to be $0$. This forces $\textbf{a}_r$ to only act on a sub-part of the input $x$, thus the local property where locality is again in the dimension domain as opposed to the input domain. This ''$0$-constraint'' is denoted by the collection of indices on which we enforce it om each of the $\textbf{a}_r$:
\begin{align}
    \Gamma(\textbf{a}_r):&=\{i \in \{1,\dots,d\} : [\textbf{a}_r]_i\equiv 0\},
\end{align}
where $[\textbf{a}_r]_i$ denotes the $i^{th}$ dimension of $\textbf{a}_r$ and $[\textbf{a}_r]_i\equiv 0$ represents the presence of the $0-$constraint for the given dimension. We also denote by $\Gamma^C(\textbf{a}_r)$ the unconstrained part where the $C$ upperscript stands for contrapose. In fact, those two collections are complementary w.r.t to the list of indices and thus given one the other is uniquely defined as $\Gamma^C(\textbf{a}_r)=\{1,\dots,D\} \backslash \Lambda(\textbf{a}_r)$ and vice-versa.
Finally, for further precision, we call a \textbf{{\em uniform local} linear spline}
the case where $\Lambda(\textbf{a}_r)=\Lambda(\textbf{a}_p),\forall r\not =p$, thus, the dimensions of $x$ on which the linear transform acts does not depend on the partition, they are ''shared''. 
This constraint allows one to control the way the mapping ''sensitivity'' to the input space dimensions, in fact by setting the constraints one can ensure that some changes in the input for those constraint dimensions will not affect the output.
In fact, given a uniform local linear spline with constraint $\Gamma(\textbf{a}_1)$ we have
\begin{align}
    \forall \epsilon \in \mathbb{R}^d, [\epsilon]_i=0,i \in \Gamma(\textbf{a}_1),
    s[(\textbf{a}_r,b_r)_{r=1}^R;\Omega](x)=s[(\textbf{a}_r,b_r)_{r=1}^R;\Omega](x+\epsilon)
\end{align}

\begin{align}
\mathbfcal{S}\left[\left(s[\textbf{c}_k,\mathcal{P}^d_{\Lambda_{m_k}},\Omega_k]\right)_{k=1}^K\right](x)=&\sum_{\alpha \in \bm{\alpha}} \left[ 
    \begin{matrix}
    s[\textbf{c}_1,\mathcal{P}^d_{\Lambda_{m_1}},\Omega_1](x)\\
    \vdots \\
    s[\textbf{c}_K,\mathcal{P}^d_{\Lambda_{m_K}},\Omega_K](x)
    \end{matrix}
    \right]1_{\{x\in \omega_\alpha\}}\\
    =&\sum_{\alpha \in \bm{\alpha}} \left[ 
    \begin{matrix}
    \sum_{\lambda \in \Lambda_{m_1}} c_{\alpha_1,\lambda} x^\lambda\\
    \vdots \\
    \sum_{\lambda \in \Lambda_{m_K}} c_{\alpha_K,\lambda} x^\lambda
    \end{matrix}
    \right]1_{\{x\in \omega_\alpha\}}
\end{align}

\subsection{Dataset Memorization Proof}

 \begin{theorem}
 Assuming all templates denoted by $A[X_n]_c,c=1,\dots,C$ have a norm constraints as $\sum_{c=1}^C ||A[X_n]_c||^2\leq K,\forall X_n$ and that all the inputs have identity norm $||X_n||=1,\forall x$ then the unique optimal templates are
 \begin{equation}
     A^*[X_n]_c=\left\{ \begin{array}{l}
          \sqrt{\frac{C-1}{C}K}X_n,\iff c=Y_n\\
          -\sqrt{\frac{K}{C(C-1)}}X_n,\text{ else}
     \end{array}\right.
 \end{equation}
 \end{theorem}
 \begin{proof}
 We aim at minimizing the cross-entropy loss function for a given input $X_n$ belonging to class $Y_n$, we also have the constraint $\sum_{c=1}^C||A[X_n]_c||^2\leq K$. The loss function is thus convex on a convex set, it is thus sufficient to find a extremum point. We denote the augmented loss function with the Lagrange multiplier as
 \begin{align*}
     l(A[X_n]_1,\dots,A[X_n]_C,\lambda)=-\langle A[X_n]_{Y_n},X_n\rangle+\log \left(\sum_{c=1}^Ce^{\langle A[X_n]_c,X_n\rangle}\right)-\lambda \left(\sum_{c=1}^C ||A[X_n]_c ||^2-K\right).
 \end{align*}
 The sufficient KKT conditions are thus
 \begin{align*}
     &\frac{d l}{d A[X_n]_1}=-1_{\{Y_n=1\}}x+\frac{e^{\langle A[X_n]_1, x\rangle }}{\sum_{c=1}^C e^{\langle A[X_n]_c, x\rangle }}x-2\lambda A[X_n]_1= \underline{0}\\
     & \vdots \\
     &\frac{d l}{d A[X_n]_C}=-1_{\{Y_n=1\}}x+\frac{e^{\langle A[X_n]_C, x\rangle }}{\sum_{c=1}^C e^{\langle A[X_n]_c, x\rangle }}x-2\lambda A[X_n]_C=\underline{0}\\
     &\frac{\partial l}{\partial \lambda}=K-\sum_{c=1}^C ||A[X_n]_c ||^2=0\\
 \end{align*}
 We first proceed by identifying $\lambda$ as follows
 \begin{align*}
 \left.
 \begin{array}{l}
     \frac{d l}{d A[X_n]_1}= \underline{0}\\
      \vdots \\
     \frac{d l}{d A[X_n]_C}=\underline{0}\\
 \end{array}\right\} &\implies \sum_{c=1}^CA[X_n]_c^T\frac{d l}{d A[X_n]_c}=0\\
 &\implies -\langle A[X_n]_{Y_n},x\rangle +\sum_{c=1}^C\frac{e^{\langle A[X_n]_c, x\rangle }}{\sum_{c=1}^C e^{\langle A[X_n]_c, x\rangle }}\langle A[X_n]_c,x\rangle-2\lambda \sum_{c=1}^C||A[X_n]_c||^2=0\\
 &\implies \lambda = \frac{1}{2K}\left(\sum_{c=1}^C\frac{e^{\langle A[X_n]_c, x\rangle }}{\sum_{c=1}^C e^{\langle A[X_n]_c, x\rangle }}\langle A[X_n]_c,x\rangle-\langle A[X_n]_{Y_n},x\rangle \right)
 \end{align*}
 Now we plug $\lambda$ in $\frac{d l}{d A[X_n]_k},\forall k=1,\dots,C$ 
 \begin{align*}
     \frac{d l}{d A[X_n]_k}=&-1_{\{Y_n=k\}}x+\frac{e^{\langle A[X_n]_k, x\rangle }}{\sum_{c=1}^C e^{\langle A[X_n]_c, x\rangle }}x-2\lambda A[X_n]_k\\
     =&-1_{\{Y_n=k\}}x+\frac{e^{\langle A[X_n]_k, x\rangle }}{\sum_{c=1}^C e^{\langle A[X_n]_c, x\rangle }}x-\frac{1}{K}\sum_{c=1}^C\frac{e^{\langle A[X_n]_c, x\rangle }}{\sum_{c=1}^C e^{\langle A[X_n]_c, x\rangle }}\langle A[X_n]_c,x\rangle A[X_n]_k\\
     &+\frac{1}{K}\langle A[X_n]_{Y_n},x\rangle A[X_n]_k
 \end{align*}
 we now leverage the fact that $A[X_n]_i=A[X_n]_j,\forall i,j \not = Y_n$ to simplify notations
 \begin{align*}
    \frac{d l}{d A[X_n]_k}=&\left(\frac{e^{\langle A[X_n]_k, x\rangle }}{\sum_{c=1}^C e^{\langle A[X_n]_c, x\rangle }}-1_{\{k=Y_n\}}\right)x-\frac{C-1}{K}\frac{e^{\langle A[X_n]_i, x\rangle }}{\sum_{c=1}^C e^{\langle A[X_n]_c, x\rangle }}\langle A[X_n]_i,x\rangle A[X_n]_k\\
     &+\frac{1}{K}\left(1-\frac{e^{\langle A[X_n]_{Y_n}, x\rangle }}{\sum_{c=1}^C e^{\langle A[X_n]_c, x\rangle }}\right)\langle A[X_n]_{Y_n},x\rangle A[X_n]_k\\
     =&\left(\frac{e^{\langle A[X_n]_k, x\rangle }}{\sum_{c=1}^C e^{\langle A[X_n]_c, x\rangle }}-1_{\{k=Y_n\}}\right)x-\frac{C-1}{K}\frac{e^{\langle A[X_n]_i, x\rangle }}{\sum_{c=1}^C e^{\langle A[X_n]_c, x\rangle }}\langle A[X_n]_i,x\rangle A[X_n]_k\\
     &+\frac{C-1}{K}\frac{e^{\langle A[X_n]_i, x\rangle }}{\sum_{c=1}^C e^{\langle A[X_n]_c, x\rangle }}\langle A[X_n]_{Y_n},x\rangle A[X_n]_k\\
     =&\left(\frac{e^{\langle A[X_n]_k, x\rangle }}{\sum_{c=1}^C e^{\langle A[X_n]_c, x\rangle }}-1_{\{k=Y_n\}}\right)x+\frac{C-1}{K}\frac{e^{\langle A[X_n]_i, x\rangle }}{\sum_{c=1}^C e^{\langle A[X_n]_c, x\rangle }}\langle A[X_n]_{Y_n}-A[X_n]_i,x\rangle A[X_n]_k
\end{align*}
 
 We now proceed by using the proposed optimal solutions $A^*[X_n]_c,c=1,\dots,C$ and demonstrate that it leads to an extremum point which by nature of the problem is the global optimum. We denote by $i$ any index different from $Y_n$, first case $k=Y_n$:
 \begin{align*}
    \frac{d l}{d A[X_n]_{Y_n}}&=-\left(1-\frac{e^{\langle A[X_n]_{Y_n}, x\rangle }}{\sum_{c=1}^C e^{\langle A[X_n]_c, x\rangle }}\right)x+\frac{C-1}{K}\frac{e^{\langle A[X_n]_i, x\rangle }}{\sum_{c=1}^C e^{\langle A[X_n]_c, x\rangle }}\langle A[X_n]_{Y_n}-A[X_n]_i,x\rangle A[X_n]_{Y_n}\\
    &=-\frac{C-1}{K}\frac{e^{\langle A[X_n]_i, x\rangle }}{\sum_{c=1}^C e^{\langle A[X_n]_c, x\rangle }}x+\frac{C-1}{K}\frac{e^{\langle A[X_n]_i, x\rangle }}{\sum_{c=1}^C e^{\langle A[X_n]_c, x\rangle }}\langle A[X_n]_{Y_n}-A[X_n]_i,x\rangle A[X_n]_{Y_n}\\
    &=-\frac{C-1}{K}\frac{e^{\langle A[X_n]_i, x\rangle }}{\sum_{c=1}^C e^{\langle A[X_n]_c, x\rangle }}\left(-x+\langle A[X_n]_{Y_n}-A[X_n]_i,x\rangle A[X_n]_{Y_n}\right)\\
    &=-\frac{C-1}{K}\frac{e^{\langle A[X_n]_i, x\rangle }}{\sum_{c=1}^C e^{\langle A[X_n]_c, x\rangle }}\left(-x+\langle \sqrt{\frac{C-1}{C}K}X_n+\sqrt{\frac{K}{C(C-1)}}X_n,x\rangle \sqrt{\frac{C-1}{C}K}X_n\right)\\
     &=\frac{(C-1)}{K}\frac{e^{\langle A[X_n]_i, x\rangle }}{\sum_{c=1}^C e^{\langle A[X_n]_c, x\rangle }}\left(-X_n+||X_n||^2X_n\right)\\
     &=0
 \end{align*}
 Other cases $k\not =Y_n$
 \begin{align*}
    \frac{d l}{d A[X_n]_{i}}=&\frac{e^{\langle A[X_n]_i, x\rangle }}{\sum_{c=1}^C e^{\langle A[X_n]_c, x\rangle }}x+\frac{C-1}{K}\frac{e^{\langle A[X_n]_i, x\rangle }}{\sum_{c=1}^C e^{\langle A[X_n]_c, x\rangle }}\langle A[X_n]_{Y_n}-A[X_n]_i,x\rangle A[X_n]_i\\
    =&\frac{e^{\langle A[X_n]_i, x\rangle }}{\sum_{c=1}^C e^{\langle A[X_n]_c, x\rangle }}\left(x+\frac{C-1}{K}\langle A[X_n]_{Y_n}-A[X_n]_i,x\rangle A[X_n]_i\right)\\
        =&\frac{e^{\langle A[X_n]_i, x\rangle }}{\sum_{c=1}^C e^{\langle A[X_n]_c, x\rangle }}\left(x-\frac{C-1}{K}\langle \sqrt{\frac{C-1}{C}K}X_n+\sqrt{\frac{K}{C(C-1)}}X_n,x\rangle \sqrt{\frac{K}{C(C-1)}}X_n\right)\\
        =&\frac{e^{\langle A[X_n]_i, x\rangle }}{\sum_{c=1}^C e^{\langle A[X_n]_c, x\rangle }}\left(X_n-||X_n||^2X_n\right)\\
     =&0
 \end{align*}
 \end{proof}
 
\subsection{Conditions for local to global inference}\label{global}

As shown in the classification tasks, having the global inference property is not always synonym of better accuracy. In fact, it is not because one network is able to produce optimal templates according to its constraint topology that the resulting mapping is ''better'' than a local inference done on an ''unconstraint'' mapping. We believe that better network conditioning would allow to have global inference and be ''optimal'' for classification tasks across topologies. We remind that the used constraints are only sufficient conditions and thus can be replaced with many others.
Overall, we also believe that the induced convex property that is paired with the global template inference property could also be leveraged during optimization to obtained faster and smarter training as for example is the case with sparse coding.

As we saw, a spline operator made of convex independent multivariate splines can have a input region selection or inference easily done making it input adaptive agnostic of the final space partition. Since a deep neural network is a composition of spline operators it is interesting to study the conditions for this inference to see if its locally optimal inference can become a global optimal inference.

We now study the conditions in order for one to pull this maximization process outside of the inner layer which would transform this greedy per layer maximization a global maximization step as for example in the case of two layers 
\[
\argmax_{\Phi^{(1)}}\langle \mathbfcal{S}^{(2)}\left(\mathbfcal{S}^{(1)}(x)\right),1\rangle =\argmax_{\Phi^{(1)}}\langle \mathbfcal{S}^{(1)}(x),1\rangle.
\]
In order to analyze this possibility we fist remind that in the case of a composition of affine spline operators, we can always rewrite the inner layers mappings as an affine transform, thus we now present the following result.
\begin{theorem}
Given a spline operator $\mathbfcal{S}[\Phi,\Omega]$ made of independent multivariate convex affine splines, we have
\begin{align}
    \argmax_{\phi \in \Phi} \langle B[W\phi(y)+b](W\phi(y)+b),1\rangle=&\argmax_{\phi \in \Phi} \langle\phi(y),1\rangle \nonumber \\
    =&\Phi[y],
\end{align}
if and only if $\sum_d (W_{d,k}+b_d)\kappa^B_d[W\phi(y)+b]>0,\forall k$ and increases w.r.t. $\phi(y)$.
\end{theorem}

\begin{proof}
We seek to prove the equality of the argmax for every input of the layer $y$,
\begin{align*}
    \argmax_{\phi \in \Phi} \langle & B[W\phi(y)+b](W\phi(y)+b),1\rangle=\argmax_{[\phi_1,\dots,\phi_K] \in \mathcal{C}[\phi_1[.],\dots,\phi_K[.]]}\langle B[W\phi(y)+b](W\phi(y)+b),1\rangle\\
    =&\argmax_{[\phi_1,\dots,\phi_K] \in \mathcal{C}[\phi_1[.],\dots,\phi_K[.]]}\langle 
    \left[\begin{matrix}
    \sum_d B[W\phi(y)+b]_{1,d}\sum_i W_{d,i}\phi_i(y)+\sum_dB[W\phi(y)+b]_{1,d}b_d\\
    \vdots \\
    \sum_d B[W\phi(y)+b]_{1,d}\sum_i W_{d,i}\phi_i(y)+\sum_dB[W\phi(y)+b]_{1,d}b_d\\
    \end{matrix}
    \right]
    ,1\rangle\\
    =&\argmax_{[\phi_1,\dots,\phi_K] \in \mathcal{C}[\phi_1[.],\dots,\phi_K[.]]}
    \sum_k\sum_d B[W\phi(y)+b]_{k,d}\sum_i W_{d,i}\phi_i(y)+\sum_k\sum_dB[W\phi(y)+b]_{k,d}b_d\\
    =&\argmax_{[\phi_1,\dots,\phi_K] \in \mathcal{C}[\phi_1[.],\dots,\phi_K[.]]}
    \sum_i\phi_i(y)\sum_d W_{d,i}\sum_kB[W\phi(y)+b]_{k,d}+\sum_db_d\sum_kB[W\phi(y)+b]_{k,d}\\
    =&\argmax_{[\phi_1,\dots,\phi_K] \in \mathcal{C}[\phi_1[.],\dots,\phi_K[.]]}
    \sum_i\phi_i(y)\alpha_i[W\phi(y)+b]+\beta[W\phi(y)+b]\\
    =&\left[
    \begin{matrix}
    \argmax_{\phi_1 \in \phi_1[.]}\sum_i\phi_i(y)\alpha_i[W\phi(y)+b]+\beta[W\phi(y)+b]\\
    \vdots\\
    \argmax_{\phi_K \in \phi_K[.]}\sum_i\phi_i(y)\alpha_i[W\phi(y)+b]+\beta[W\phi(y)+b]\\
    \end{matrix}
    \right]\\
    =&\left[
    \begin{matrix}
    \argmax_{\phi_1 \in \phi_1[.]}\phi_1(y)\\
    \vdots\\
    \argmax_{\phi_K \in \phi_K[.]}\phi_K(y)\\
    \end{matrix}
    \right]\\
    =&\argmax_{\Phi \in \Phi[.]}\langle \Phi(y),1\rangle\\
    =&\Phi[y]
\end{align*}

\end{proof}
\begin{mdframed}
\begin{cor}
In order to a deep neural network to have globally optimal inference, we have the following sufficient conditions
\begin{itemize}
    \item Unconstrained first layer filters and bias
    \item Positive filters and nonnegative bias for inner-layers, strictly increasing nonlinearities, last layer should be a fc-layer.
\end{itemize}
\end{cor}
\end{mdframed}
\begin{prop}
We this mentioned properties, one also has the following property
\begin{align}
    \argmax_{\phi \in \Phi} \left(B[W\phi(y)+b](W\phi(y)+b)\right)_k=&\argmax_{\phi \in \Phi} \langle\phi(y),1\rangle \nonumber \\
    =&\Phi[y],\forall k
\end{align}
Hence the local inference leads to the same spline as the one maximizing each output neuron of the network.
\end{prop}

\subsection{Space Contraction and Adversarial Examples}

We present the softmax nonlinearity case which is as opposed to the intuition a strictly contractive operator. In fact, we have the following result.
\begin{theorem}
The softmax layer is strictly contractive with $K=\frac{D-1}{D^2}$
\end{theorem}
\begin{proof}
We now from that 
\begin{equation}
    ||f(x)-f(y)||^2_2\leq \max_x ||Df(x)||_F^2||x-y||^2_2,
\end{equation}
thus we now analyze $\max_x ||Df(x)||_F^2$.
\begin{align*}
    \max_{p \in \bigtriangleup_D} ||Df(p)||_F^2=&\max_{p \in \bigtriangleup_D}\sum_{i=1}^D\sum_{j=1,j\not = i}^Dp_i^2p_j^2+\sum_{i=1}^Dp_i^2(1-p_i)^2
\end{align*}

where we used $\bigtriangleup_D$ the simplex of dimension $D$ defined as 
\begin{equation}
    \bigtriangleup_D=\{x\in \mathbb{R}^{D+}|\sum_i x_i=1\}.
\end{equation}
The Lagrangian is the augmented loss function with the constrain as
\begin{align*}
    \mathcal{L}(p)=\sum_{i=1}^D\sum_{j=1,j\not = i}^Dp_i^2p_j^2+\sum_{i=1}^Dp_i^2(1-p_i)^2+\lambda(\sum_{i=1}^Dp_i-1),
\end{align*}
we now seek the stationary points
\begin{align*}
    \frac{\partial \mathcal{L}}{\partial p_k}=&4p_k\sum_{j=1,j\not = k}^Dp_j^2+2p_k(1-p_k)^2-2(1-p_k)p_k^2+\lambda\\
    =&4p_k\sum_{j=1,j\not = k}^Dp_j^2+2p_k-4p_k^2+2p_k^3-2p_k^2+2p_k^3+\lambda\\
    =&4p_k\sum_{j=1,j\not = k}^Dp_j^2+2p_k-6p_k^2+4p_k^3+\lambda\\
    =&4p_k\sum_{j=1}^Dp_j^2+2p_k-6p_k^2+\lambda,\;\;\forall k\\
    \frac{\partial \mathcal{L}}{\partial \lambda}=&\sum_{i=1}^Dp_i-1
\end{align*}
we now seek to solve the system $\nabla \mathcal{L}=0$. Since we have $\frac{\partial \mathcal{L}}{\partial p_k}=0\forall k$, it is clear that $\sum_{k=1}^D\frac{\partial \mathcal{L}}{\partial p_k}=0$ leading to
\begin{align*}
\sum_{k=1}^D\frac{\partial \mathcal{L}}{\partial p_k}&=0\\
\implies \sum_{k=1}^D\left( 4p_k\sum_{j=1}^Dp_j^2+2p_k-6p_k^2+\lambda\right) &=0\\
\implies 4\sum_{j=1}^Dp_j^2+2-6\sum_{j=1}^Dp_j^2+D\lambda &=0\\
\implies D\lambda&=2\sum_{j=1}^Dp_j^2-2\\
\implies \lambda = \frac{2}{D}(\sum_{j=1}^Dp_j^2-1). 
\end{align*}
Now plugin this into the gradient of $\mathcal{L}$, we have the updated system of equations
\begin{align*}
    4p_k\sum_{j=1}^Dp_j^2+2p_k-6p_k^2+\\frac{2}{D}(\sum_{j=1}^Dp_j^2-1)=&0,\;\;\forall k\\
    \sum_{i=1}^Dp_i=&1
\end{align*}
and thus in vector form leads to 
\begin{align*}
    \textbf{p}(2||\textbf{p}||^2_2+1)-3 \textbf{p} \bullet \textbf{p}=(1-||\textbf{p}||^2_2)vec(1/D),
\end{align*}
where $\bullet$ denotes the element-wise product.
It is clear that $\textbf{p}$ must be constant across its dimensions, with the constraint this leads to $p=vec(1/D)$.
It is indeed a maximum as on the boundary of the domain $f=0$ and at the point we have 
\begin{align*}
    f(vec(1/D))=&\sum_{i=1}^D\sum_{j=1,j\not = i}^D \frac{1}{D^4}+\sum_{i=1}^D\frac{1}{D^2}(1-\frac{1}{D})^2\\
    =&\frac{D-1}{D^3}+\frac{1}{D}(1-\frac{1}{D})^2\\
    =&\frac{D-1}{D^3}+\frac{1}{D}-\frac{2}{D^2}+\frac{1}{D^3}\\
    =&\frac{D-1+D^2-2D+1}{D^3}\\
    =&\frac{D(D-1)}{D^3}\\
    =&\frac{D-1}{D^2}
\end{align*}
\end{proof}

\subsection{Function Approximation, Orbits, Class Separation, Generalization and Activation Graph}\label{orbits}

In this section, we develop some simple formulations of invariant learning and orbits in order to provide on way to define generalization for deep learning.

\subsubsection{Function Approximation and Orbits for Invariant Learning}

Firstly, as a deep network is a composition of affine spline operators, hence a spline operators itself, we analyze what is the function that is approximated. For the classification case, the objective is to learn the mapping that predicts a density distribution representation the probability of class belonging for the input. This means that if the target class is $k$ we aim at learning
\begin{equation}\label{orbit}
f(x)=\textbf{e}_k,\forall x \in \mathcal{X}_k,
\end{equation}
where $\mathcal{X}_k$ represents the collection of data belonging to class $k$. It is clear from this formulation that $f$ aims at learning orbits, namely the manifold of class $k$.
Given that in general the input space is high dimensional, and the training set finite, the main challenge one has to face for generalization is to be able to perform efficient interpolation given new test points. Using splines for manifold learning and interpolation has been used for example in \cite{hofer2004energy,savel1995splines,atteia1989spline,bezhaev1988splines,gu2006manifold} for low dimensional cases. In fact, the interpolation of splines is flexible enough to learn locally sharp functions via small partition of the input space, and yet allow for robust interpolation.

During the learning phase, by changing the weights regions of the input space are learned along with the per-region mappings. As a result during learning, back-propagation will guide the partition s.a. Eq. \ref{orbit} is fulfilled. 
At test time, the regions are fixed and in order to generalize well, it is sufficient that the region in which the input belongs has the affine mapping corresponding to the right class.

We base the following analysis on \cite{mallat2016understanding} in which orbits and invariance learning is brought in the context of deep learning, especially convolutional neural networks.

Firstly, for notation and simplification we call $y_x$ the label associated to a given input $x$ which belongs to some space $\mathcal{X} \subset \mathbb{R}^d$.
First, all given inputs $x,x' \in \mathcal{X}^2$ are separated independently from the class belonging as
\begin{equation}
    ||x-x'||>0,\forall y_x,y_{x'},
\end{equation}
but as this input is fed into a deep neural to be transformed into a succession of $z^{(1)},z^{(2)},\dots,z^{(L)}$, we aim at those representation to somehow regroup together when considering inputs from the same class, and conversely, for difference classes, these representation should not ''colide''. We represent this separation as
\begin{align}
y_x\not =y_{x'} \implies ||z^{(\ell)}[x]-z^{(\ell)}[x']||>0,\forall \ell =1,\dots,L
\end{align}
on the other hand, for same class inputs, we aim at having
\begin{align}
y_x =y_{x'} \implies ||z^{(\ell)}[x]-z^{(\ell)}[x']||>||z^{(\ell+1)}[x]-z^{(\ell+1)}[x']||,\forall \ell =1,\dots,L-1
\end{align}
which is a soft condition requiring that in the limit of depth, same class input collide together.
We denote by $D^{(\ell)}[x,x']$ the operator representing this separation as
\begin{equation}
    D^{(\ell)}[x,x']=||z^{(\ell)}[x]-z^{(\ell)}[x'] ||.
\end{equation}
It is clear that the separation measure $D$ is an indicator of the modeling power of deep neural networks. From this, the analysis proposed in \cite{mallat2016understanding} relies on the aggregation or invariance to group actions while moving through the inner layers.
Given a set $\mathcal{X}$, such as the space of images of the world, and a group $G$ of operators acting on $\mathcal{X}$, an orbit of $x\in \mathcal{X}$ is defined as the set of all points
\begin{equation}
    G.x=\{g.x:g\in G\}.
\end{equation}
For example, $G$ can be the group of rotation operators acting on images, and thus given an image, its orbit is the set of all the rotated version of this image. In general, this invariance learning is restricted to local symmetries. In fact, if we take the example of MNIST dataset, being invariant to small rotation is beneficial, yet, the ''global'' rotation group will bring a $1$ to become a $7$ or a $6$ to become a $9$. As a result, being invariant to the action group of rotations is globally detrimental but locally beneficial, as it is more generally for various transformations.

Via the presented work, we can now postulate on the way these orbits are approximated. In fact, we remind that a deep neural networks, given an input, produces an affine transformation to produce its output via the adapted template matching.
Given one layer we simplify as an affine transformation followed by a nonlinearity:
\begin{equation}
    z^{(\ell)}=\textbf{A}^{(\ell)}(Wz^{(\ell-1)}+b^{(\ell)}),
\end{equation}
we have that given two inputs $z^{(\ell-1)}_1$ and $z^{(\ell-1)}_2$, their respective output ''energy'' is for standard deep networks
\begin{align}
\langle z^{(\ell)}_1,1 \rangle &=\max_{A}\langle A^{(\ell)}(Wz^{(\ell-1)}_1+b^{(\ell)}),1 \rangle,\\
\langle z^{(\ell)}_2,1\rangle &=\max_{A}\langle A^{(\ell)}(Wz^{(\ell-1)}_2+b^{(\ell)}),1 \rangle.
\end{align}
Since the affine transforms are usually convolutions with filters of small sizes, it is likely that those two quantities will be close.
However, analyzing their separateness leads to insightful results.
We now consider the case of using the ReLU nonlinearity where we thus remind that the possible matrices $A$ are diagonal with diagonal values belonging to $\{0,1\}$.
\begin{align*}
    D^{(\ell)}[x_1,x_2]=&||A^{(\ell)}_1(W^{(\ell)}z^{(\ell-1)}_1+b)-A^{(\ell)}_2(W^{(\ell)}z^{(\ell-1)}_2+b)||^2\\
    =&||A^{(\ell)}_1(W^{(\ell)}z^{(\ell-1)}_1+b)||^2+||A^{(\ell)}_2(W^{(\ell)}z^{(\ell-1)}_2+b)||^2\\
    &-2\langle A^{(\ell)}_1(W^{(\ell)}z^{(\ell-1)}_1+b),A^{(\ell)}_2(W^{(\ell)}z^{(\ell-1)}_2+b) \rangle \\
    =&||A^{(\ell)}_1(W^{(\ell)}z^{(\ell-1)}_1+b)||^2+||A^{(\ell)}_2(W^{(\ell)}z^{(\ell-1)}_2+b)||^2\\
    &-2(W^{(\ell)}z^{(\ell-1)}_1+b)^TA^{(\ell)T}_1A^{(\ell)}_2(W^{(\ell)}z^{(\ell-1)}_2+b).
\end{align*}
While the first two terms simply represent the amount of energy going through the layer, the last term is the ''pivot'' as it is directly link to the ''similarity'' of the two inputs. In fact, for the ReLU and LReLU it is easy to see that 
$A^{(\ell)T}_1 A^{(\ell)}_2$ is a diagonal matrix that represents the region associated with the intersection of the two input regions. In fact, to illustrate this, with the ReLU, the $d^{th}$ element of the diagonal is $1$ is the corresponding dimension of the affine transform is greater than $0$. Now given two inputs, their affine transforms produce two features maps with associated ReLU as just described. An element-wise multiplication of those induced $A$ matrices will ''leave'' a $1$ in the $d^{th}$ dimension if and only if both features maps have their $d^{th}$ dimension greater than $0$. As a result, it is clear that
\begin{equation}
    -2(W^{(\ell)}z^{(\ell-1)}_1+b)^TA^{(\ell)T}_1A^{(\ell)}_2(W^{(\ell)}z^{(\ell-1)}_2+b) \propto m(\omega_1\cap \omega_2),
\end{equation}
where $\omega_1$ and $\omega_2$ correspond to the regions in which $W^{(\ell)}z^{(\ell-1)}_1+b$ and respectively $W^{(\ell)}z^{(\ell-1)}_2+b$ belong to. As a result, we can see that as this measure of space intersection $m(\omega_1\cap \omega_2)$ grows, as $D^{(\ell)}[x_1,x_2]$ diminishes.
From this, it is interesting to note that the succession of $D^{(\ell)}[x_1,x_2],\ell=1,\dots,L$ can be summarized by the succession of $m(\omega^{(\ell)}_1\cap \omega^{(\ell)}_2)$, a.k.a the measure of the space intersection induced after all the affine transformations.
If we now go back to the definition of orbit and invariance learning per layer, it is clear that it translates to learning a layer s.a. $m(\omega^{(\ell)}_1\cap \omega^{(\ell)}_2)$ is proportional to the action group ''locality'' applied between $x_1$ and $x_2$.

Also, on the first layers, the number of points from the training set belonging to each region might be very small as we remind that the number of possible region for a $d$-dimensional ReLU/LReLU is $2^d$, and in practice $d\gg 784$ which is much obviously much smaller than any possible dataset. However, layers after layers, the dimension eventually goes down till ultimately being equal to $K$ the number of classes to predict.
Thus, it is likely that 
\[
m(\omega^{(1)}_1,\omega^{(1)}_2)\gg 0,\forall y_1,y_2,
\]
thus all the points are ''scattered'' across the possible regions. From that, only remains the task to collide together points of the same class via this succession of spline operators.

For the network to generalize, it is now enough that for a new observation $x$ we have $D^{(\ell)}[x,x_1]<D^{(\ell)}[x,x_2]$ for some $\ell$ and for all $x_2$ of the wrong classes. As this occur with greater $\ell$ as the amount of ''generalization'' required increases since it relates directly to the amount of composition needed to disentangle the hierarchy of group actions. If this occur at $\ell=1$ for all possible points, this translates to have a task linearly separable in the first place.
In practice we note that this measure $m(\omega^{(\ell)}_1\cap \omega^{(\ell)}_2)$ can be easily defined via
\begin{equation}\label{distanceeq}
    m(\omega^{(\ell)}_1\cap \omega^{(\ell)}_2)=||A^{(\ell)T}_1A^{(\ell)}_2||.
\end{equation}

\subsubsection{Activation Graph, Paths}
We now present some interesting results concerning the core structure of the activated regions of the affine spline operators that form deep neural networks. In particular we will see some properties when considering the following graph given an ordered collection of affine spline operators $\left(\mathbfcal{S}^{(\ell)}[\textbf{A}^{(\ell)},\textbf{b}^{(\ell)},\bm{\Omega}^{(\ell)}]\right)_{\ell=1}^L$ as
\begin{equation}
    \mathcal{G}[\left(\mathbfcal{S}^{(\ell)}[\textbf{A}^{(\ell)},\textbf{b}^{(\ell)},\bm{\Omega}^{(\ell)}]\right)_{\ell=1}^L]=\left(V,E(x)\right)
\end{equation}
with
\begin{align}
    V&=\cup_{\ell=1}^L\bm{\Omega}^{(\ell)},\\
    E(x)&=\{(\omega^{(\ell)}_{\alpha^*},\omega^{(\ell+1)}_{\alpha^*})_{\ell=1}^{L-1}\},
\end{align}
where it is clear that $V$ the set of vertex corresponds to all the possible regions of all the possible spline operators while the set of edges which is input dependant corresponds to the active region linked from one spline operator to the next one.
This input dependency is natural as we recall that for each spline function, depending on the region in which its input lie, an specific affine function is used to produce the output which in turn is fed into the next affine spline operator.
This per spline region selection also characterize the input, and can be used to compare different input signals. Intuitively, if given two observations $x_1$ and $x_2$ the sequence of produced regions is the same, the two input $x_1$ and $x_2$ line in the same parametric line of $z^{(\ell)}$.

\begin{prop}
The possible regions that can take an input when considering each region of each spline as a node gives rise to a bipartite graph. Thus, $\left(V,E(x)\right)$ is  always a bipartite graph $\forall x$.
\end{prop}
It is clear that since the connectivity only go from one spline possible region to the other, we can create the bipartite graph by considering the two following sets of nodes
\begin{align*}
    V_1 =& \cup_{\ell=1,\ell \text{ odd}}^L\bm{\Omega}^{(\ell)},\\
    V_2 =& \cup_{\ell=1,\ell \text{ even}}^L\bm{\Omega}^{(\ell)},
\end{align*}
and thus is clear that for any input $x$, we have the two partitions $V_1$ and $V_2$ as no edges can link two nodes of $V_1$ nor two nodes of $V_2$.

\begin{defn}
The {\bf activation graph} (AG) corresponding to a signal $x$ is the collection of active sub-regions among all the splines defining the deep network, and thus it is $E(x)$. Its root is an output and its leaves are at the input level.
\end{defn}

Note that the introduced bipartite graph $(V,E(x))$ via the selection regions of the affine spline operators is different from the derived neural network paths from \cite{choromanska2015loss,nguyen2016semi}. In those works, the created graph has nodes all the possible neurons of any layers, and the edges are between neurons that fired, as those work only apply in the case of max-pooling  and ReLU. As a result, the result AG has no interesting structure and is in general not usable for as the number of edges is extremely large. On the other hand, with the proposed region associated graph, we believe that much better results can be derived either in term of input characterization, outlier detection or generally structure behaving of the spline operators. By the natural sparsity we also reduce the computational overhead and are able to better visualize, interpret the AG.

\subsubsection{Experiments}

In order to highlight the aspect of ''separation'' between classes and invariance learning we provide some simple experiments on the MNIST dataset. We use three fully trained deep neural networks, the SmallCNN, LargeCNN and Resnet3 topologies. We then take a sample from the test set made of $25$ examples of $4$ classes. We present the distance matrix at each level of those networks before and after training. We used to compute the distance the definition introduced in \ref{distanceeq}. We present the used images in Fig. \ref{mnistimages} and the corresponding distance matrices and results in Fig. \ref{mnistdistancessmall},\ref{mnistdistanceslarge},\ref{mnistdistancesresnet} respectively for the SmallCNN, LargeCNN and resnet.
\begin{figure}[t]
    \centering
    \includegraphics[width=6in]{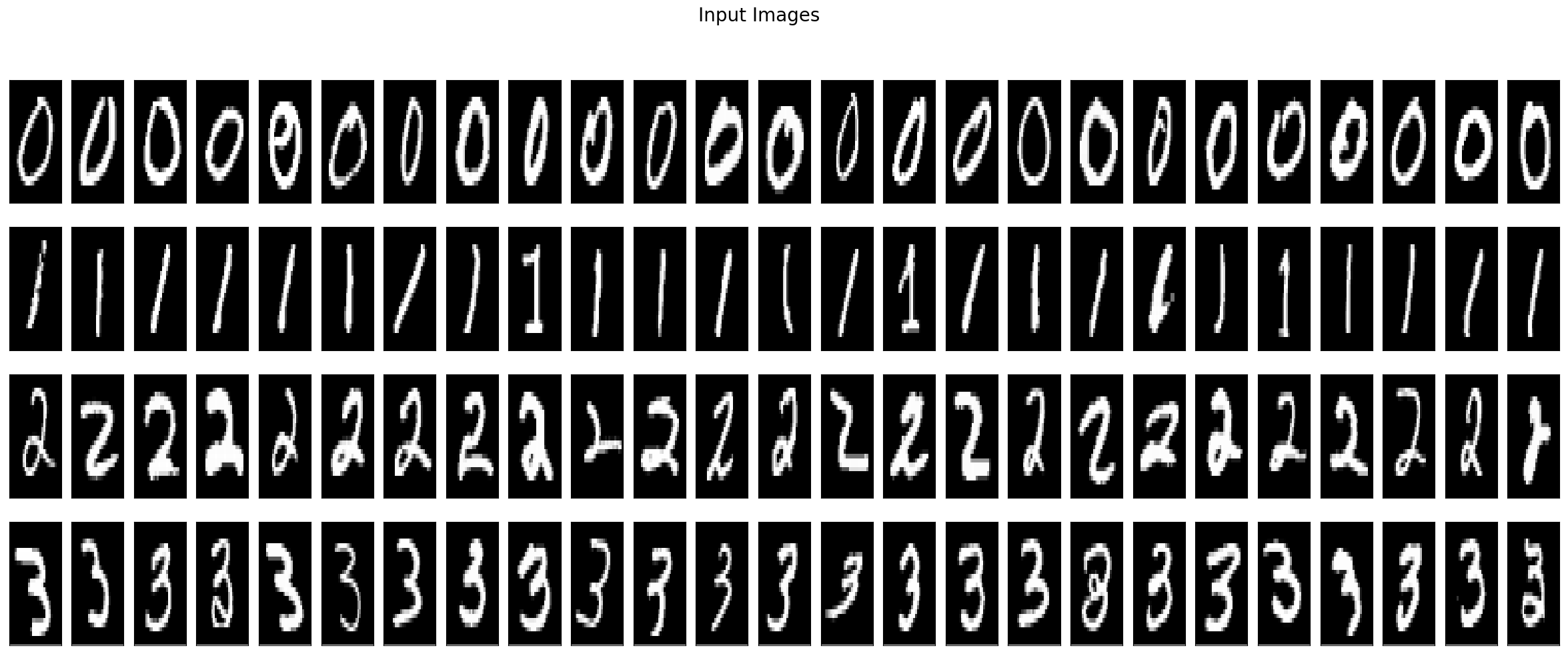}
    \caption{Examples of images tested for MNIST on $4$ different classes Evolution of the sub-regions distances over the layers.  }
    \label{mnistimages}
\end{figure}

\begin{figure}[t]
    \centering
    \includegraphics[width=6in]{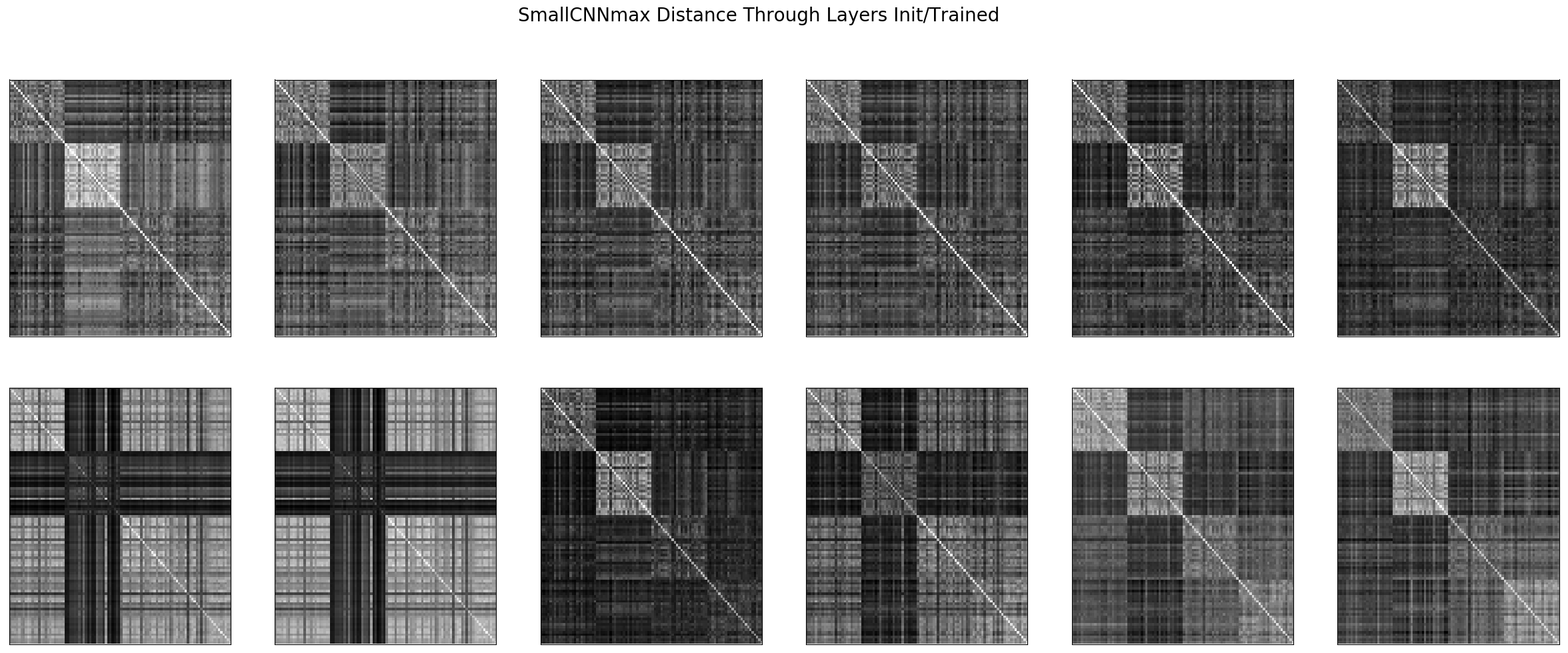}
    \caption{SmallCNNmean Evolution of the sub-regions distances over the layers. }
    \label{mnistdistancessmall}
\end{figure}

The class corresponding to the digit $1$ is very interesting to analyze as all the given observations can be mapped to the same orbit just via a global rotation matrix as opposed to all the other classes provided here. As a result we can see the need for ''depth'' being null for this class, translated as a distance matrix almost optimal at the first layer output in Fig. \ref{mnistdistanceslarge}.

\begin{figure}[t]
    \centering
    \includegraphics[width=6in]{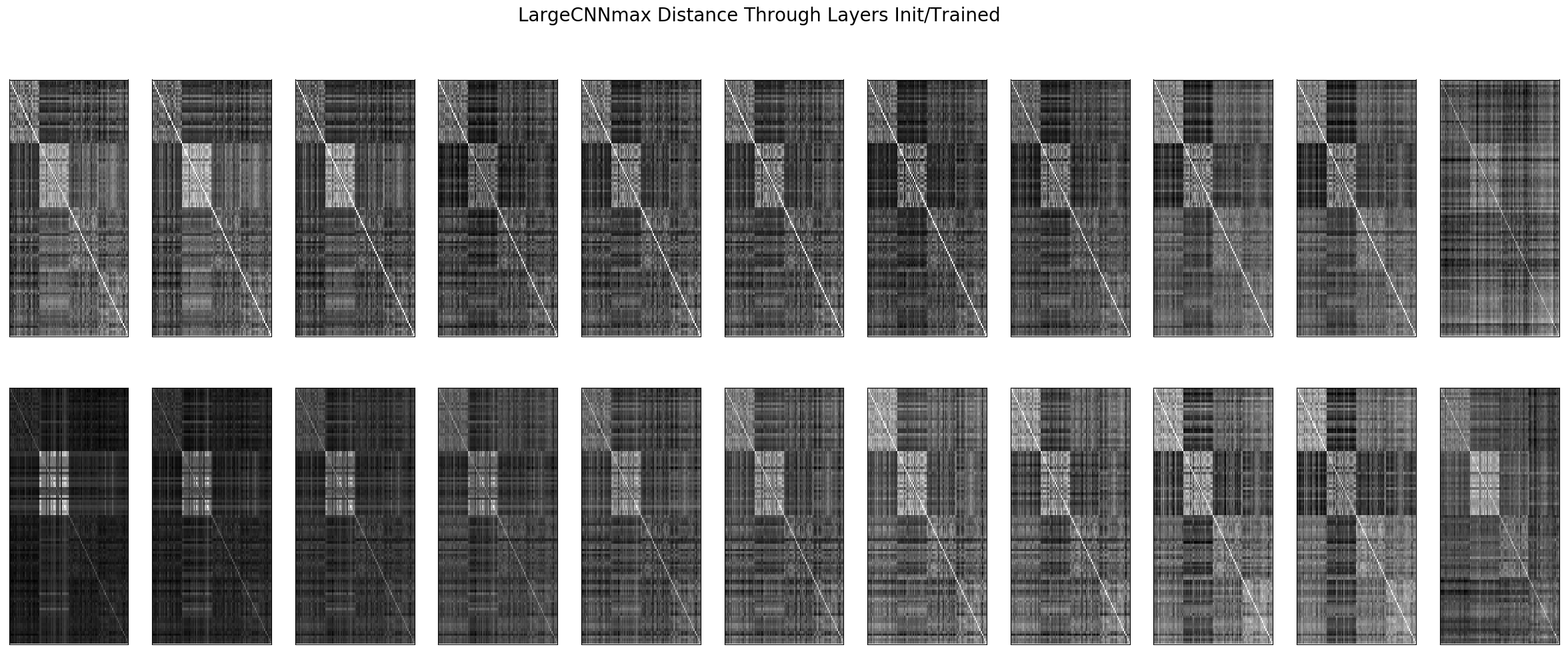}
    \caption{LargeCNNmax Evolution of the sub-regions distances over the layers.  ditto figures 15 16.}
    \label{mnistdistanceslarge}
\end{figure}

\begin{figure}[t]
    \centering
    \includegraphics[width=6in]{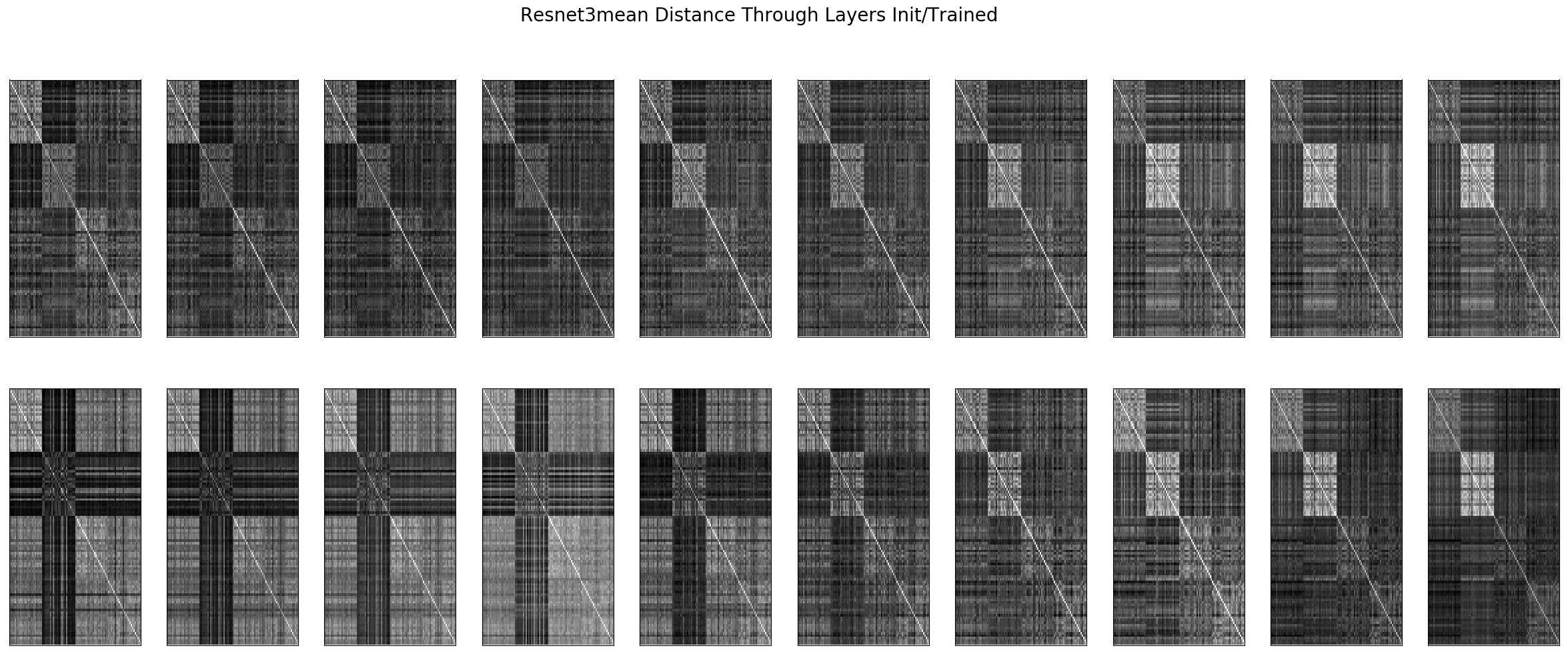}
    \caption{Resnet3mean Evolution of the sub-regions distances over the layers.}
    \label{mnistdistancesresnet}
\end{figure}

%% file: main.bbl
\begin{thebibliography}{}

\bibitem[Abadi et~al., 2016]{abadi2016tensorflow}
Abadi, M., Agarwal, A., Barham, P., Brevdo, E., Chen, Z., Citro, C., Corrado,
  G.~S., Davis, A., Dean, J., Devin, M., et~al. (2016).
\newblock Tensorflow: Large-scale machine learning on heterogeneous distributed
  systems.
\newblock {\em arXiv preprint arXiv:1603.04467}.

\bibitem[Arjovsky et~al., 2017]{arjovsky2017wasserstein}
Arjovsky, M., Chintala, S., and Bottou, L. (2017).
\newblock Wasserstein gan.
\newblock {\em arXiv preprint arXiv:1701.07875}.

\bibitem[Atteia and Benbourhim, 1989]{atteia1989spline}
Atteia, M. and Benbourhim, M. (1989).
\newblock Spline elastic manifolds.
\newblock {\em Mathematical methods in computer aided geometric design}, pages
  45--50.

\bibitem[Bajcsy and Kova{\v{c}}i{\v{c}}, 1989]{bajcsy1989multiresolution}
Bajcsy, R. and Kova{\v{c}}i{\v{c}}, S. (1989).
\newblock Multiresolution elastic matching.
\newblock {\em Computer vision, graphics, and image processing}, 46(1):1--21.

\bibitem[Bastien et~al., 2012]{bastien2012theano}
Bastien, F., Lamblin, P., Pascanu, R., Bergstra, J., Goodfellow, I., Bergeron,
  A., Bouchard, N., Warde-Farley, D., and Bengio, Y. (2012).
\newblock Theano: new features and speed improvements.
\newblock {\em arXiv preprint arXiv:1211.5590}.

\bibitem[Bengio et~al., 2013]{bengio2013advances}
Bengio, Y., Boulanger-Lewandowski, N., and Pascanu, R. (2013).
\newblock Advances in optimizing recurrent networks.
\newblock In {\em Acoustics, Speech and Signal Processing (ICASSP), 2013 IEEE
  International Conference on}, pages 8624--8628. IEEE.

\bibitem[Berger et~al., 1994]{berger1994removing}
Berger, J., Coifman, R.~R., and Goldberg, M.~J. (1994).
\newblock Removing noise from music using local trigonometric bases and wavelet
  packets.
\newblock {\em Journal of the Audio Engineering Society}, 42(10):808--818.

\bibitem[Bergstra et~al., 2010]{bergstra2010theano}
Bergstra, J., Breuleux, O., Bastien, F., Lamblin, P., Pascanu, R., Desjardins,
  G., Turian, J., Warde-Farley, D., and Bengio, Y. (2010).
\newblock Theano: A cpu and gpu math compiler in python.
\newblock In {\em Proc. 9th Python in Science Conf}, pages 1--7.

\bibitem[Bezhaev, 1988]{bezhaev1988splines}
Bezhaev, A.~Y. (1988).
\newblock Splines on manifolds.
\newblock {\em Russian Journal of Numerical Analysis and Mathematical
  Modelling}, 3(4):287--300.

\bibitem[Bishop, 2008]{bishop2008training}
Bishop, C.~M. (2008).
\newblock Training with noise is equivalent to tikhonov regularization.
\newblock {\em Training}, 7(1).

\bibitem[Bloor and Wilson, 1990]{bloor1990representing}
Bloor, M.~I. and Wilson, M.~J. (1990).
\newblock Representing pde surfaces in terms of b-splines.
\newblock {\em Computer-Aided Design}, 22(6):324--331.

\bibitem[Blumer et~al., 1987]{blumer1987occam}
Blumer, A., Ehrenfeucht, A., Haussler, D., and Warmuth, M.~K. (1987).
\newblock Occam's razor.
\newblock {\em Information processing letters}, 24(6):377--380.

\bibitem[Boyd and Xu, 2009]{boyd2009divergence}
Boyd, J.~P. and Xu, F. (2009).
\newblock Divergence (runge phenomenon) for least-squares polynomial
  approximation on an equispaced grid and mock--chebyshev subset interpolation.
\newblock {\em Applied Mathematics and Computation}, 210(1):158--168.

\bibitem[Breiman, 1993]{breiman1993hinging}
Breiman, L. (1993).
\newblock Hinging hyperplanes for regression, classification, and function
  approximation.
\newblock {\em IEEE Transactions on Information Theory}, 39(3):999--1013.

\bibitem[Burr, 1981]{burr1981elastic}
Burr, D.~J. (1981).
\newblock Elastic matching of line drawings.
\newblock {\em IEEE Transactions on Pattern Analysis and Machine Intelligence},
  3(6):708.

\bibitem[Carlini and Wagner, 2016]{carlini2016defensive}
Carlini, N. and Wagner, D. (2016).
\newblock Defensive distillation is not robust to adversarial examples.
\newblock {\em arXiv preprint}.

\bibitem[Cheney, 1980]{cheney1980approximation}
Cheney, E.~W. (1980).
\newblock {\em Approximation theory III}, volume~12.
\newblock Academic Press New York.

\bibitem[Choromanska et~al., 2015]{choromanska2015loss}
Choromanska, A., Henaff, M., Mathieu, M., Arous, G.~B., and LeCun, Y. (2015).
\newblock The loss surfaces of multilayer networks.
\newblock In {\em AISTATS}.

\bibitem[Chui, 1988]{chui1988multivariate}
Chui, C.~K. (1988).
\newblock {\em Multivariate splines}.
\newblock SIAM.

\bibitem[Chung et~al., 2014]{chung2014empirical}
Chung, J., Gulcehre, C., Cho, K., and Bengio, Y. (2014).
\newblock Empirical evaluation of gated recurrent neural networks on sequence
  modeling.
\newblock {\em arXiv preprint arXiv:1412.3555}.

\bibitem[Coifman and Wickerhauser, 1992]{coifman1992entropy}
Coifman, R.~R. and Wickerhauser, M.~V. (1992).
\newblock Entropy-based algorithms for best basis selection.
\newblock {\em IEEE Transactions on information theory}, 38(2):713--718.

\bibitem[Cybenko, 1989]{cybenko1989approximation}
Cybenko, G. (1989).
\newblock Approximation by superpositions of a sigmoidal function.
\newblock {\em Mathematics of Control, Signals, and Systems (MCSS)},
  2(4):303--314.

\bibitem[de~Boor and DeVore, 1983]{de1983approximation}
de~Boor, C. and DeVore, R. (1983).
\newblock Approximation by smooth multivariate splines.
\newblock {\em Transactions of the American Mathematical Society},
  276(2):775--788.

\bibitem[de~Br{\'e}bisson and Vincent, 2015]{de2015exploration}
de~Br{\'e}bisson, A. and Vincent, P. (2015).
\newblock An exploration of softmax alternatives belonging to the spherical
  loss family.
\newblock {\em arXiv preprint arXiv:1511.05042}.

\bibitem[Deng et~al., 2009]{deng2009imagenet}
Deng, J., Dong, W., Socher, R., Li, L.-J., Li, K., and Fei-Fei, L. (2009).
\newblock Imagenet: A large-scale hierarchical image database.
\newblock In {\em Computer Vision and Pattern Recognition, 2009. CVPR 2009.
  IEEE Conference on}, pages 248--255. IEEE.

\bibitem[Duchi et~al., 2011]{duchi2011adaptive}
Duchi, J., Hazan, E., and Singer, Y. (2011).
\newblock Adaptive subgradient methods for online learning and stochastic
  optimization.
\newblock {\em Journal of Machine Learning Research}, 12(Jul):2121--2159.

\bibitem[Fawzi et~al., 2015]{fawzi2015analysis}
Fawzi, A., Fawzi, O., and Frossard, P. (2015).
\newblock Analysis of classifiers' robustness to adversarial perturbations.
\newblock {\em arXiv preprint arXiv:1502.02590}.

\bibitem[Gal and Ghahramani, 2016]{gal2016dropout}
Gal, Y. and Ghahramani, Z. (2016).
\newblock Dropout as a bayesian approximation: Representing model uncertainty
  in deep learning.
\newblock In {\em international conference on machine learning}, pages
  1050--1059.

\bibitem[Glorot et~al., 2011]{glorot2011deep}
Glorot, X., Bordes, A., and Bengio, Y. (2011).
\newblock Deep sparse rectifier neural networks.
\newblock In {\em Aistats}, volume~15, page 275.

\bibitem[Graves, 2013]{graves2013generating}
Graves, A. (2013).
\newblock Generating sequences with recurrent neural networks.
\newblock {\em arXiv preprint arXiv:1308.0850}.

\bibitem[Graves and Schmidhuber, 2005]{graves2005framewise}
Graves, A. and Schmidhuber, J. (2005).
\newblock Framewise phoneme classification with bidirectional lstm networks.
\newblock In {\em Neural Networks, 2005. IJCNN'05. Proceedings. 2005 IEEE
  International Joint Conference on}, volume~4, pages 2047--2052. IEEE.

\bibitem[Gu and Rigazio, 2014]{gu2014towards}
Gu, S. and Rigazio, L. (2014).
\newblock Towards deep neural network architectures robust to adversarial
  examples.
\newblock {\em arXiv preprint arXiv:1412.5068}.

\bibitem[Gu et~al., 2006]{gu2006manifold}
Gu, X., He, Y., and Qin, H. (2006).
\newblock Manifold splines.
\newblock {\em Graphical Models}, 68(3):237--254.

\bibitem[Guyon et~al., 1992]{guyon1992structural}
Guyon, I., Vapnik, V., Boser, B., Bottou, L., and Solla, S.~A. (1992).
\newblock Structural risk minimization for character recognition.
\newblock In {\em Advances in neural information processing systems}, pages
  471--479.

\bibitem[Hannah and Dunson, 2013]{hannah2013multivariate}
Hannah, L.~A. and Dunson, D.~B. (2013).
\newblock Multivariate convex regression with adaptive partitioning.
\newblock {\em The Journal of Machine Learning Research}, 14(1):3261--3294.

\bibitem[He et~al., 2016]{he2016deep}
He, K., Zhang, X., Ren, S., and Sun, J. (2016).
\newblock Deep residual learning for image recognition.
\newblock In {\em Proceedings of the IEEE Conference on Computer Vision and
  Pattern Recognition}, pages 770--778.

\bibitem[Hecht-Nielsen et~al., 1988]{hecht1988theory}
Hecht-Nielsen, R. et~al. (1988).
\newblock Theory of the backpropagation neural network.
\newblock {\em Neural Networks}, 1(Supplement-1):445--448.

\bibitem[Hinton, 1986]{hinton1986learning}
Hinton, G.~E. (1986).
\newblock Learning distributed representations of concepts.
\newblock In {\em Proceedings of the eighth annual conference of the cognitive
  science society}, volume~1, page~12. Amherst, MA.

\bibitem[Hinton, 1987]{hinton1987learning}
Hinton, G.~E. (1987).
\newblock Learning translation invariant recognition in a massively parallel
  networks.
\newblock In {\em International Conference on Parallel Architectures and
  Languages Europe}, pages 1--13. Springer.

\bibitem[Hochreiter and Schmidhuber, 1995]{hochreiter1995simplifying}
Hochreiter, S. and Schmidhuber, J. (1995).
\newblock Simplifying neural nets by discovering flat minima.
\newblock In {\em Advances in neural information processing systems}, pages
  529--536.

\bibitem[Hofer and Pottmann, 2004]{hofer2004energy}
Hofer, M. and Pottmann, H. (2004).
\newblock Energy-minimizing splines in manifolds.
\newblock {\em ACM Transactions on Graphics (TOG)}, 23(3):284--293.

\bibitem[Hornik et~al., 1989]{hornik1989multilayer}
Hornik, K., Stinchcombe, M., and White, H. (1989).
\newblock Multilayer feedforward networks are universal approximators.
\newblock {\em Neural networks}, 2(5):359--366.

\bibitem[Jayaraman et~al., 2009]{jayaraman2009digital}
Jayaraman, S., Esakkirajan, S., and Veerakumar, T. (2009).
\newblock Digital image processing tmh publication.
\newblock {\em Year of Publication}.

\bibitem[Kim and De~Ara{\'u}jo, 2007]{kim2007grayscale}
Kim, H.~Y. and De~Ara{\'u}jo, S.~A. (2007).
\newblock Grayscale template-matching invariant to rotation, scale,
  translation, brightness and contrast.
\newblock In {\em Pacific-Rim Symposium on Image and Video Technology}, pages
  100--113. Springer.

\bibitem[Kingma and Ba, 2014]{kingma2014adam}
Kingma, D. and Ba, J. (2014).
\newblock Adam: A method for stochastic optimization.
\newblock {\em arXiv preprint arXiv:1412.6980}.

\bibitem[Kingma et~al., 2014]{kingma2014semi}
Kingma, D.~P., Mohamed, S., Rezende, D.~J., and Welling, M. (2014).
\newblock Semi-supervised learning with deep generative models.
\newblock In {\em Advances in Neural Information Processing Systems}, pages
  3581--3589.

\bibitem[Korman et~al., 2013]{korman2013fast}
Korman, S., Reichman, D., Tsur, G., and Avidan, S. (2013).
\newblock Fast-match: Fast affine template matching.
\newblock In {\em Proceedings of the IEEE Conference on Computer Vision and
  Pattern Recognition}, pages 2331--2338.

\bibitem[Lang et~al., 1990]{lang1990time}
Lang, K.~J., Waibel, A.~H., and Hinton, G.~E. (1990).
\newblock A time-delay neural network architecture for isolated word
  recognition.
\newblock {\em Neural networks}, 3(1):23--43.

\bibitem[LeCun et~al., 1989]{lecun1989generalization}
LeCun, Y. et~al. (1989).
\newblock Generalization and network design strategies.
\newblock {\em Connectionism in perspective}, pages 143--155.

\bibitem[LeCun et~al., 1995]{lecun1995learning}
LeCun, Y., Jackel, L., Bottou, L., Cortes, C., Denker, J.~S., Drucker, H.,
  Guyon, I., Muller, U., Sackinger, E., Simard, P., et~al. (1995).
\newblock Learning algorithms for classification: A comparison on handwritten
  digit recognition.
\newblock {\em Neural networks: the statistical mechanics perspective},
  261:276.

\bibitem[Li et~al., 2016]{li2016whiteout}
Li, Y., Xu, R., and Liu, F. (2016).
\newblock Whiteout: Gaussian adaptive regularization noise in deep neural
  networks.
\newblock {\em arXiv preprint arXiv:1612.01490}.

\bibitem[Lyu et~al., 2015]{lyu2015unified}
Lyu, C., Huang, K., and Liang, H.-N. (2015).
\newblock A unified gradient regularization family for adversarial examples.
\newblock In {\em Data Mining (ICDM), 2015 IEEE International Conference on},
  pages 301--309. IEEE.

\bibitem[Maal{\o}e et~al., 2016]{maaloe2016auxiliary}
Maal{\o}e, L., S{\o}nderby, C.~K., S{\o}nderby, S.~K., and Winther, O. (2016).
\newblock Auxiliary deep generative models.
\newblock {\em arXiv preprint arXiv:1602.05473}.

\bibitem[MacKay, 1996]{mackay1996bayesian}
MacKay, D.~J. (1996).
\newblock Bayesian methods for backpropagation networks.
\newblock In {\em Models of neural networks III}, pages 211--254. Springer.

\bibitem[Magnani and Boyd, 2009]{magnani2009convex}
Magnani, A. and Boyd, S.~P. (2009).
\newblock Convex piecewise-linear fitting.
\newblock {\em Optimization and Engineering}, 10(1):1--17.

\bibitem[Mallat, 1999]{mallat1999wavelet}
Mallat, S. (1999).
\newblock {\em A wavelet tour of signal processing}.
\newblock Academic press.

\bibitem[Mallat, 2008]{mallat2008wavelet}
Mallat, S. (2008).
\newblock {\em A wavelet tour of signal processing: the sparse way}.
\newblock Academic press.

\bibitem[Mallat, 2016]{mallat2016understanding}
Mallat, S. (2016).
\newblock Understanding deep convolutional networks.
\newblock {\em Phil. Trans. R. Soc. A}, 374(2065):20150203.

\bibitem[Matsuoka, 1992]{matsuoka1992noise}
Matsuoka, K. (1992).
\newblock Noise injection into inputs in back-propagation learning.
\newblock {\em IEEE Transactions on Systems, Man, and Cybernetics},
  22(3):436--440.

\bibitem[Meyer, 1993]{meyer1993algorithms}
Meyer, Y. (1993).
\newblock Algorithms and applications.
\newblock {\em SIAM, philadelphia}.

\bibitem[Miyato et~al., 2015]{miyato2015distributional}
Miyato, T., Maeda, S.-i., Koyama, M., Nakae, K., and Ishii, S. (2015).
\newblock Distributional smoothing by virtual adversarial examples.
\newblock {\em stat}, 1050:2.

\bibitem[Moody and Utans, 1994]{moody1994architecture}
Moody, J. and Utans, J. (1994).
\newblock Architecture selection strategies for neural networks: Application to
  corporate bond rating prediction.
\newblock In {\em Neural networks in the capital markets}, pages 277--300. John
  Wiley \& Sons.

\bibitem[Morgan and Bourlard, 1990]{morgan1990generalization}
Morgan, N. and Bourlard, H. (1990).
\newblock Generalization and parameter estimation in feedforward nets: Some
  experiments.
\newblock In {\em Advances in neural information processing systems}, pages
  630--637.

\bibitem[Murray and Edwards, 1993]{murray1993synaptic}
Murray, A.~F. and Edwards, P.~J. (1993).
\newblock Synaptic weight noise during mlp learning enhances fault-tolerance,
  generalization and learning trajectory.
\newblock In {\em Advances in neural information processing systems}, pages
  491--498.

\bibitem[Nalisnick et~al., 2015]{nalisnick2015scale}
Nalisnick, E., Anandkumar, A., and Smyth, P. (2015).
\newblock A scale mixture perspective of multiplicative noise in neural
  networks.
\newblock {\em arXiv preprint arXiv:1506.03208}.

\bibitem[Nguyen et~al., 2016]{nguyen2016semi}
Nguyen, T., Liu, W., Perez, E., Baraniuk, R.~G., and Patel, A.~B. (2016).
\newblock Semi-supervised learning with the deep rendering mixture model.
\newblock {\em arXiv preprint arXiv:1612.01942}.

\bibitem[Nishikawa, 1998]{nishikawa1998accurate}
Nishikawa, H. (1998).
\newblock Accurate piecewise linear continuous approximations to
  one-dimensional curves: Error estimates and algorithms.

\bibitem[Nowlan and Hinton, 1992]{nowlan1992simplifying}
Nowlan, S.~J. and Hinton, G.~E. (1992).
\newblock Simplifying neural networks by soft weight-sharing.
\newblock {\em Neural computation}, 4(4):473--493.

\bibitem[Olshausen et~al., 1996]{olshausen1996emergence}
Olshausen, B.~A. et~al. (1996).
\newblock Emergence of simple-cell receptive field properties by learning a
  sparse code for natural images.
\newblock {\em Nature}, 381(6583):607--609.

\bibitem[Pal and Mitra, 1992]{pal1992multilayer}
Pal, S.~K. and Mitra, S. (1992).
\newblock Multilayer perceptron, fuzzy sets, and classification.
\newblock {\em IEEE Transactions on neural networks}, 3(5):683--697.

\bibitem[Papernot et~al., 2016]{papernot2016distillation}
Papernot, N., McDaniel, P., Wu, X., Jha, S., and Swami, A. (2016).
\newblock Distillation as a defense to adversarial perturbations against deep
  neural networks.
\newblock In {\em Security and Privacy (SP), 2016 IEEE Symposium on}, pages
  582--597. IEEE.

\bibitem[Patel et~al., 2015]{patel2015probabilistic}
Patel, A.~B., Nguyen, T., and Baraniuk, R.~G. (2015).
\newblock A probabilistic theory of deep learning.
\newblock {\em arXiv preprint arXiv:1504.00641}.

\bibitem[Pe{\~n}a, 2000]{pena2000multivariate}
Pe{\~n}a, J.~M. (2000).
\newblock On the multivariate horner scheme.
\newblock {\em SIAM journal on numerical analysis}, 37(4):1186--1197.

\bibitem[Plaut et~al., 1986]{plaut1986experiments}
Plaut, D.~C. et~al. (1986).
\newblock Experiments on learning by back propagation.

\bibitem[Rasmus et~al., 2015]{rasmus2015semi}
Rasmus, A., Berglund, M., Honkala, M., Valpola, H., and Raiko, T. (2015).
\newblock Semi-supervised learning with ladder networks.
\newblock In {\em Advances in Neural Information Processing Systems}, pages
  3546--3554.

\bibitem[Reinsch, 1967]{reinsch1967smoothing}
Reinsch, C.~H. (1967).
\newblock Smoothing by spline functions.
\newblock {\em Numerische mathematik}, 10(3):177--183.

\bibitem[Rister and Rubin, 2017]{rister2017piecewise}
Rister, B. and Rubin, D.~L. (2017).
\newblock Piecewise convexity of artificial neural networks.
\newblock {\em Neural Networks}, 94:34--45.

\bibitem[Rumelhart et~al., 1988]{rumelhart1988learning}
Rumelhart, D.~E., Hinton, G.~E., Williams, R.~J., et~al. (1988).
\newblock Learning representations by back-propagating errors.
\newblock {\em Cognitive modeling}, 5(3):1.

\bibitem[Rumelhart and Mcclelland, 1986]{rumelhart1986parallel}
Rumelhart, D.~E. and Mcclelland, J.~L. (1986).
\newblock Parallel distributed processing: Explorations in the microstructure
  of cognition: Foundations (parallel distributed processing).

\bibitem[Salimans et~al., 2016]{salimans2016improved}
Salimans, T., Goodfellow, I., Zaremba, W., Cheung, V., Radford, A., and Chen,
  X. (2016).
\newblock Improved techniques for training gans.
\newblock In {\em Advances in Neural Information Processing Systems}, pages
  2226--2234.

\bibitem[Savel'ev, 1995]{savel1995splines}
Savel'ev, I.~V. (1995).
\newblock Splines and manifolds.
\newblock {\em Russian Mathematical Surveys}, 50(6):1306--1307.

\bibitem[Schmidhuber, 1994]{schmidhuber1994discovering}
Schmidhuber, J. (1994).
\newblock Discovering problem solutions with low kolmogorov complexity and high
  generalization capability.
\newblock In {\em Machine Learning: Proceedings of the Twelfth International
  Conference}. Citeseer.

\bibitem[Schoenberg, 1964]{schoenberg1964interpolation}
Schoenberg, I.~J. (1964).
\newblock On interpolation by spline functions and its minimal properties.
\newblock In {\em On Approximation Theory/{\"U}ber Approximationstheorie},
  pages 109--129. Springer.

\bibitem[Schumaker, 2007]{schumaker2007spline}
Schumaker, L. (2007).
\newblock {\em Spline functions: basic theory}.
\newblock Cambridge University Press.

\bibitem[Shaham et~al., 2015]{shaham2015understanding}
Shaham, U., Yamada, Y., and Negahban, S. (2015).
\newblock Understanding adversarial training: Increasing local stability of
  neural nets through robust optimization.
\newblock {\em arXiv preprint arXiv:1511.05432}.

\bibitem[Smith, 1985]{smith1985numerical}
Smith, G.~D. (1985).
\newblock {\em Numerical solution of partial differential equations: finite
  difference methods}.
\newblock Oxford university press.

\bibitem[Springenberg, 2015]{springenberg2015unsupervised}
Springenberg, J.~T. (2015).
\newblock Unsupervised and semi-supervised learning with categorical generative
  adversarial networks.
\newblock {\em arXiv preprint arXiv:1511.06390}.

\bibitem[Srivastava et~al., 2014]{srivastava2014dropout}
Srivastava, N., Hinton, G.~E., Krizhevsky, A., Sutskever, I., and
  Salakhutdinov, R. (2014).
\newblock Dropout: a simple way to prevent neural networks from overfitting.
\newblock {\em Journal of Machine Learning Research}, 15(1):1929--1958.

\bibitem[Szegedy et~al., 2013]{szegedy2013intriguing}
Szegedy, C., Zaremba, W., Sutskever, I., Bruna, J., Erhan, D., Goodfellow, I.,
  and Fergus, R. (2013).
\newblock Intriguing properties of neural networks.
\newblock {\em arXiv preprint arXiv:1312.6199}.

\bibitem[Tieleman and Hinton, 2012]{tieleman2012lecture}
Tieleman, T. and Hinton, G. (2012).
\newblock Lecture 6.5-rmsprop: Divide the gradient by a running average of its
  recent magnitude.
\newblock {\em COURSERA: Neural networks for machine learning}, 4(2):26--31.

\bibitem[Tropp, 2004]{tropp2004greed}
Tropp, J.~A. (2004).
\newblock Greed is good: Algorithmic results for sparse approximation.
\newblock {\em IEEE Transactions on Information theory}, 50(10):2231--2242.

\bibitem[Uchida and Sakoe, 2005]{uchida2005survey}
Uchida, S. and Sakoe, H. (2005).
\newblock A survey of elastic matching techniques for handwritten character
  recognition.
\newblock {\em IEICE transactions on information and systems},
  88(8):1781--1790.

\bibitem[Vapnik, 1992]{vapnik1992principles}
Vapnik, V. (1992).
\newblock Principles of risk minimization for learning theory.
\newblock In {\em Advances in neural information processing systems}, pages
  831--838.

\bibitem[Veit et~al., 2016]{veit2016residual}
Veit, A., Wilber, M.~J., and Belongie, S. (2016).
\newblock Residual networks behave like ensembles of relatively shallow
  networks.
\newblock In {\em Advances in Neural Information Processing Systems}, pages
  550--558.

\bibitem[Wager et~al., 2013]{wager2013dropout}
Wager, S., Wang, S., and Liang, P.~S. (2013).
\newblock Dropout training as adaptive regularization.
\newblock In {\em Advances in neural information processing systems}, pages
  351--359.

\bibitem[Weigend et~al., 1990]{weigend1990predicting}
Weigend, A.~S., Huberman, B.~A., and Rumelhart, D.~E. (1990).
\newblock Predicting the future: A connectionist approach.
\newblock {\em International journal of neural systems}, 1(03):193--209.

\bibitem[Williams, 1995]{williams1995bayesian}
Williams, P.~M. (1995).
\newblock Bayesian regularization and pruning using a laplace prior.
\newblock {\em Neural computation}, 7(1):117--143.

\bibitem[Wolpert, 1994]{wolpert1994bayesian}
Wolpert, D.~H. (1994).
\newblock Bayesian backpropagation over io functions rather than weights.
\newblock In {\em Advances in neural information processing systems}, pages
  200--207.

\bibitem[Xu et~al., 2015]{xu2015empirical}
Xu, B., Wang, N., Chen, T., and Li, M. (2015).
\newblock Empirical evaluation of rectified activations in convolutional
  network.
\newblock {\em arXiv preprint arXiv:1505.00853}.

\bibitem[Yann, 1987]{yann1987modeles}
Yann, L. (1987).
\newblock {\em Mod{\`e}les connexionnistes de l’apprentissage}.
\newblock PhD thesis, These de Doctorat, Universite Paris 6.

\bibitem[Zeiler, 2012]{zeiler2012adadelta}
Zeiler, M.~D. (2012).
\newblock Adadelta: an adaptive learning rate method.
\newblock {\em arXiv preprint arXiv:1212.5701}.

\bibitem[Zhang et~al., 2016]{zhang2016understanding}
Zhang, C., Bengio, S., Hardt, M., Recht, B., and Vinyals, O. (2016).
\newblock Understanding deep learning requires rethinking generalization.
\newblock {\em arXiv preprint arXiv:1611.03530}.

\bibitem[Zhang et~al., 1997]{zhang1997face}
Zhang, J., Yan, Y., and Lades, M. (1997).
\newblock Face recognition: eigenface, elastic matching, and neural nets.
\newblock {\em Proceedings of the IEEE}, 85(9):1423--1435.

\end{thebibliography}
